\documentclass[letterpaper]{article} 
\usepackage{aaai2026}  
\usepackage{times}  
\usepackage{helvet}  
\usepackage{courier}  
\usepackage[hyphens]{url}  
\usepackage{graphicx} 
\usepackage{natbib}  
\usepackage{caption} 
\usepackage[linesnumbered,ruled,vlined,noend]{algorithm2e}

\usepackage[noend]{algpseudocode}
\usepackage{mathtools}
\usepackage{float}
\usepackage{dsfont}
\usepackage{amssymb}
\newcommand{\dblsetminus}{\mathbin{{\setminus}\mspace{-9mu}{\setminus}}}

\usepackage{amssymb, amsmath, amsthm}

\usepackage[polish,english]{babel}
\usepackage[T1,nomathsymbols]{polski}
\usepackage[utf8]{inputenc}

\usepackage{bm}
\usepackage{MnSymbol}
\usepackage[inline]{enumitem}

\usepackage{booktabs} 
\usepackage{multirow} 

\usepackage{tikz}
\usetikzlibrary{decorations.pathreplacing, arrows.meta}

\usepackage{newfloat}
\usepackage{listings}
\DeclareCaptionStyle{ruled}{labelfont=normalfont,labelsep=colon,strut=off} 
\floatstyle{ruled}
\newfloat{listing}{tb}{lst}{}
\floatname{listing}{Listing}
\usepackage{subcaption}

\usepackage[title]{appendix}

\usepackage{thm-restate}

\newtheorem{theorem}{Theorem}
\newtheorem{proposition}[theorem]{Proposition}
\newtheorem{corollary}{Corollary}[theorem]
\newtheorem{lemma}[theorem]{Lemma}
\newtheorem{claim}[theorem]{Claim}

\theoremstyle{definition}
\newtheorem{definition}{Definition}
\newtheorem{example}{Example}
\usepackage{todonotes}
\newcommand{\myvdots}{\ \vdots\ }
\usepackage{array}
\newenvironment{psketch}{%
  \proof}{\endproof}

\definecolor{mossgreen}{RGB}{34,139,34}

\title{Incremental Maintenance of DatalogMTL Materialisations}
\author {
    Kaiyue Zhao\textsuperscript{\rm 1}\equalcontrib,
    Dingqi Chen\textsuperscript{\rm 1}\equalcontrib,
    Shaoyu Wang\textsuperscript{\rm 1, \rm 2}\equalcontrib,
    Pan Hu\textsuperscript{\rm 1}\thanks{Corresponding author}
}
\affiliations {
    \textsuperscript{\rm 1}School of Computer Science, Shanghai Jiao Tong University, China\\
    \textsuperscript{\rm 2}Department of Computer Science, University of Oxford, UK
}

\begin{document}

\maketitle

\begin{abstract}
DatalogMTL extends the classical Datalog language with metric temporal logic (MTL), enabling expressive reasoning over temporal data. While existing reasoning approaches, such as materialisation-based and automata-based methods, offer soundness and completeness, they lack support for handling efficient dynamic updates—a crucial requirement for real-world applications that involve frequent data updates. In this work, we propose DRed$_\text{MTL}$, an incremental reasoning algorithm for DatalogMTL with bounded intervals. Our algorithm builds upon the classical Delete/Rederive (DRed) algorithm, which incrementally updates the materialisation of a Datalog program. Unlike a Datalog materialisation which is in essence a finite set of facts, a DatalogMTL materialisation has to be represented as a finite set of facts plus periodic intervals indicating how the full materialisation can be constructed through unfolding. To cope with this, our algorithm is equipped with specifically designed operators to efficiently handle such periodic representations of DatalogMTL materialisations. We have implemented this approach and tested it on several publicly available datasets. Experimental results show that DRed$_\text{MTL}$ often significantly outperforms rematerialisation, sometimes by orders of magnitude.
\end{abstract}

\begin{links}
\link{Code, datasets and instructions}{github.com/Horizon12275/DREDmtl-for-DatalogMTL}
\link{Extended version with full proof}{arxiv.org/abs/2511.12169}
\end{links}

\section{Introduction}
DatalogMTL extends the well-known rule language Datalog \cite{DBLP:journals/tkde/CeriGT89} with metric temporal logic (MTL) \cite{DBLP:journals/rts/Koymans90}. It has found applications in various domains, including ontology-based query answering \cite{DBLP:journals/jair/BrandtKRXZ18,guzel2018ontop}, stream reasoning \cite{walkega2023stream,DBLP:conf/aaai/WalegaKG19}, and temporal reasoning in the financial sector \cite{DBLP:conf/edbt/ColomboBCL23,DBLP:conf/edbt/NisslS22,mori2022neural}, among others. 

To showcase the capabilities of DatalogMTL, consider an industrial application scenario in which DatalogMTL has been applied by our research group to automatically detect anomaly of power transformers. More concretely, gas concentration data is first collected in real time using existing gas-in-oil sensors and then fed into a DatalogMTL rule engine, which fires alarms whenever appropriate. As an example, the following rule shows how oil thermal faults can be detected:
\begin{equation}\label{rule:ex}
\begin{aligned}
OTF(x) \gets & HasEthylene(x,y) 
\\ & \land HasEthane(x,z) 
\\ & \land \diamondminus_{[0,10]} AboveThirty(y) 
\\ & \land \diamondminus_{[0,10]} AboveSeventy(z)
\end{aligned}
\end{equation}
This rule states that at any time point $t$, if ethylene and ethane have both been detected in the transformer oil, and their gas concentration values surpassed 30 ppm and 70 ppm at any time point in the past ten minutes, respectively, then an oil thermal fault (OTF) is detected.


Reasoning in DatalogMTL can be implemented using top-down or bottom-up approaches, or a combination of the two. One typical top-down approach is based on B\"uchi automata: it ensures correctness but incurs high reasoning costs. Therefore, efforts have been made to devise efficient bottom-up (or materialisation-based) approach for DatalogMTL reasoning. The vadalog system \cite{DBLP:conf/icde/BellomariniNS22} implements such an approach, but it may not terminate due to recursion. The MeTeoR system \cite{DBLP:conf/aaai/WangHWG22} combines bottom-up and top-down approaches, but only resorts to the top-down approach when necessary. More recently, magic set rewriting, which simulates top-down evaluation via bottom-up reasoning, has been extended to support DatalogMTL reasoning \cite{wang2025goal}. While materialisation-based approaches are popular for DatalogMTL reasoning, to the best of our knowledge, no incremental maintenance algorithm for DatalogMTL reasoning has been developed: when the set of explict facts change, the above systems have no choice but to recompute the materialisation from scratch. 

For plain Datalog, many incremental materialisation maintenance algorithms have been developed, including Delete/Rederive (DRed) and its variants \cite{10.1145/170035.170066,10.5555/645922.673479, 10.1145/2063576.2063696, 10.1007/978-3-642-41335-3_41, HuMH18}, Backward/Forward (B/F) \cite{10.5555/2886521.2886537}, FBF~\cite{motik2019maintenance}, among others.
However, adapting an incremental materialisation maintenance algorithm for Datalog to support DatalogMTL reasoning is nontrivial. Unlike Datalog, DatalogMTL supports recursion over time, which easily leads to unbounded time intervals, making termination problematic. As a result, changes to the dataset may not only affect immediate derivations but also propagate across time through periodic patterns. Although it is known that under certain restrictions, materialisations in DatalogMTL exhibit repeating structures that can be finitely represented, how to correctly update such repeating structures is still highly challenging, especially due to the complex interplay between consequence propagation and termination condition checks.

In this paper, building upon the well-known DRed algorithm, we propose DRed$_\text{MTL}$. It replaces standard Datalog materialisation maintenance with a novel set of operations over finite representations of DatalogMTL materialisations called \textit{periodic materialisations}. Periodic materialisations compactly encode infinite sets of temporal facts by capturing their recurring periodic patterns, which allows us to reason over infinite time domains using finite representations. 


To facilitate reasoning over periodic materialisations, we devised a novel semina\"ive evaluation operator specifically designed to efficiently handle the propagation of facts in the new DatalogMTL setting. In addition, we developed a new period identification algorithm that effectively and correctly guarantees termination. Our experimental evaluation demonstrates that, compared to rematerialisation from scratch, our approach achieves significant performance improvements, especially for small updates.

\section{Preliminaries}
In this section, we first briefly recapitulate the syntax and semantics of DatalogMTL. We focus on DatalogMTL with bounded intervals, as this restriction enables materialisation-based reasoning. Then, we briefly discuss how the materialisation of a DatalogMTL program can be finitely represented: our incremental maintenance algorithm will have to incrementally update such finite representations in response to changes in the explicitly given data.
\subsubsection{Syntax of DatalogMTL}
Throughout the paper, we assume that the
timeline consists of rational numbers. 
A time interval is a set of continuous time points $\varrho$ of the form $\langle t_1, t_2\rangle$, where $t_1\in \mathbb{Q}  \cup  \{-\infty\}$, $t_2\in\mathbb{Q}\cup\{ \infty\}$, $\langle$ is $[$ or $($ and likewise $\rangle$ is $]$ or $)$.
An interval is bounded if both of its endpoints are rational numbers, i.e., neither $\infty$ nor $-\infty$. 
When it is clear from the context, we may abuse the distinction between intervals (i.e. sets of time points) and their representation $\langle t_1, t_2\rangle$.
If $\varrho=\langle t_1,t_2\rangle$, let $\varrho^-=t_1$ and $\varrho^+ = t_2$. 

A term is either a variable or a constant. 
A relational atom is an expression $R(\bm{t})$, where $R$ is a predicate and $\bm{t}$ is a tuple of terms of arity matching that of $R$. Metric atoms extend relational atoms by allowing operators from metric temporal logic (MTL), namely $\boxplus_\varrho$, $\boxminus_\varrho$,
$\diamondplus_\varrho$, 
$\diamondminus_\varrho$, 
$\mathcal{U}_\varrho$, and $\mathcal{S}_\varrho$, where $\varrho$ is an interval. Formally, metric atoms, $M$, are   generated by the  grammar
\begin{align*}
    M ::= &\bot \mid \top \mid R(\bm{t}) \mid \boxplus_{\varrho}M \mid \boxminus_{\varrho}M \mid \diamondplus_{\varrho}M \mid \diamondminus_{\varrho}M \mid {}\\ 
    & M\mathcal{U}_{\varrho}M \mid M \mathcal{S}_{\varrho} M,
\end{align*}
where $\top$ and $\bot$ are the constants representing truth and falsehood, respectively, and $\varrho$ is any arbitrary interval containing only nonnegative rationals. 
A DatalogMTL rule, $r$,  is of the form
\begin{align*}
    M' \leftarrow M_1 \land M_2\land \dots \land M_n, \quad \text{for $n \geq 1$},
\end{align*}
where each $M_i$ is a metric atom and $M'$ is a metric atom not mentioning $\bot$, $\diamondplus$, $\diamondminus$, $\mathcal{U}$, and $\mathcal{S}$. 
Metric atom $M'$ and the set ${\{M_i \mid i \in \{1, \dots , n\} \}}$ are the head and body of $r$, denoted as $head(r)$ and $body(r)$, respectively. 

A rule $r$ is safe if each variable in its head appears also in its body. A program $\Pi$ is a finite set of safe rules. A substitution ${\sigma}$ is a mapping of finitely many variables to constants. For $\alpha$ an expression (e.g., an atom, a rule, or a program thereof), ${\alpha\sigma}$ is the result of replacing each occurrence
of a variable $x$ in $\alpha$ with $\sigma(x)$, if the latter is defined.
An expression is ground if it mentions no variable. A fact is of the form $R(\bm{t})@\varrho$, where $R(\bm{t})$ is a ground relational atom, $\varrho$ is an interval, and $@$ indicates that the preceding atom holds over the time interval that follows.
A dataset $\mathcal{D}$ is a finite set of facts. 
If all the intervals a dataset (resp., a program) mentions are bounded, then the dataset (resp., the program) is bounded. Our work focuses on bounded datasets and programs.
\subsubsection{Semantics of DatalogMTL}
A DatalogMTL interpretation $\mathfrak{I}$ is a function that maps each time point $t \in \mathbb{Q}$ to a set of ground relational atoms (essentially to atoms that hold at $t$).
For a time point $t\in\mathbb{Q}$, if $R(\bm{t})$ belongs to this set, we write ${\mathfrak{I}, t \models R(\bm{t})}$. 
This extends to the ground metric atoms as presented in Table~\ref{tab:semantable}.
\begin{table}[t]
    \setlength{\tabcolsep}{3pt}
    \renewcommand{\arraystretch}{1.4}
    \centering
\begin{tabular}{|lll|}
    \hline
    $\mathfrak{I}, t  \vDash \top$& & \text{for every $t \in \mathbb{Q}$}  \\
    $\mathfrak{I}, t  \vDash \bot$ && \text{for no $t \in \mathbb{Q}$}\\
     $\mathfrak{I}, t \vDash \boxplus_\varrho M$ & \text{iff} &
    \text{$\mathfrak{I}, t_1 \vDash M$ for all
    $t_1 $ s.t. $t_1-t \in \varrho$}\\
     $\mathfrak{I}, t \vDash \boxminus_\varrho M$ & \text{iff}&  \text{$\mathfrak{I}, t_1 \vDash M$ for all $t_1$ s.t. $t-t_1 \in \varrho$}\\
     $\mathfrak{I}, t \vDash \diamondplus_\varrho M$ & \text{iff}&  \text{$\mathfrak{I}, t_1 \vDash M$ for some $t_1 $ s.t. $t_1-t \in \varrho$}\\
     $\mathfrak{I}, t \vDash \diamondminus_\varrho M$ & \text{iff}&  \text{$\mathfrak{I}, t_1 \vDash M$ for some $t_1 $ s.t. $t-t_1 \in \varrho$}\\
     $\mathfrak{I}, t \vDash M_2 \mathcal{U}_\varrho M_1$ & \text{iff} &  \text{$\mathfrak{I}, t_1 \vDash M_1$ for some  $t_1$ s.t. $t_1 - t \in \varrho$,} \\
    &&     \text{and $\mathfrak{I}, t_2 \vDash M_2$ for all  $t_2 \in (t, t_1)$}\\    
     $\mathfrak{I}, t \vDash M_2 \mathcal{S}_\varrho M_1$ & \text{iff} &  \text{$\mathfrak{I}, t_1 \vDash M_1$ for some $t_1$ s.t. $t - t_1 \in \varrho$,} \\ 
    &&  \text{and $\mathfrak{I}, t_2 \vDash M_2$ for all $t_2 \in (t_1, t)$}
    \\\hline
\end{tabular}
    \caption{Semantics for ground metric atoms}
    \label{tab:semantable}
\end{table}

An interpretation $\mathfrak{I}$ satisfies a fact $R(\bm{t})@\varrho$, denoted as  ${\mathfrak{I}\models R(\bm{t})@\varrho}$, if ${\mathfrak{I}, t \vDash R(\bm{t})}$ for each ${t \in \varrho}$. An interpretation $\mathfrak{I}$ is a model of a dataset $\mathcal{D}$, written $\mathfrak{I}\models\mathcal{D}$, if it satisfies every fact in $\mathcal{D}$.
An interpretation $\mathfrak{I}$ satisfies a ground rule $r_0$ if ${\mathfrak{I}, t \vDash body(r_0)}$ implies ${\mathfrak{I}, t \vDash head(r_0)}$ for each ${t \in \mathbb{Q}}$. 
A ground rule $r_0$ is an instance of a rule $r$ if there is a substitution $\sigma$ such that $r_0 = r\sigma$. An interpretation $\mathfrak{I}$ satisfies a rule $r$, if it satisfies all ground instances of $r$,
and it is a model of a program $\Pi$, if it satisfies all rules in $\Pi$.
$\mathfrak{I}$ is a model of a pair $(\Pi, \mathcal{D})$ if $\mathfrak{I}$ is both a model of $\mathcal{D}$ and a model of $\Pi$ . 
A program-dataset pair $(\Pi, \mathcal{D})$ entails a fact $R(\bm{t})@\varrho$, written $(\Pi, \mathcal{D}) \models R(\bm{t})@\varrho$, if all models of $(\Pi, \mathcal{D})$ satisfy $R(\bm{t})@\varrho$. 
An interpretation $\mathfrak{I}$ contains another interpretation $\mathfrak{I}'$, written $\mathfrak{I}' \subseteq \mathfrak{I}$,  if $\mathfrak{I},t \models R(\bm{t})$ implies $\mathfrak{I}',t\models R(\bm{t})$ for each ground relational atom $R(\bm{t})$ and each $t\in\mathbb{Q}$; moreover, $\mathfrak{I}=\mathfrak{I'}$ if they contain each other.
An interpretation $\mathfrak{I}$ is the least in a set of interpretations $\bm{\mathfrak{I}}$, if $\mathfrak{I}\in\bm{\mathfrak{I}},\text{ and } \forall\mathfrak{I}'\in\bm{\mathfrak{I}}$, $\mathfrak{I}\subseteq\mathfrak{I}'$. Similarly, $\mathfrak{I}$ is the greatest in $\bm{\mathfrak{I}}$, if $\mathfrak{I}\in\bm{\mathfrak{I}},\text{ and } \forall\mathfrak{I}'\in\bm{\mathfrak{I}}$, $\mathfrak{I}'\subseteq\mathfrak{I}$. 
Each dataset $\mathcal{D}$ has a unique least interpretation
$\mathfrak{I}_{\mathcal{D}}$ such that $\mathfrak{I}_{\mathcal{D}}$ is the least in the set of all models of $\mathcal{D}$.
For interpretations $\mathfrak{I}_1$, $\mathfrak{I}_2$, and $\mathfrak{I}_3$, we write $\mathfrak{I}_3 = \mathfrak{I}_1 \cup \mathfrak{I}_2$ if $\mathfrak{I}_3$ is the least in $\{\mathfrak{I}\mid \mathfrak{I}_1 \subseteq \mathfrak{I} \text{ and } \mathfrak{I}_2 \subseteq \mathfrak{I}\}$, 
and we write $\mathfrak{I}_3 = \mathfrak{I}_1 \cap \mathfrak{I}_2$, if $\mathfrak{I}_3$ is the greatest in $\{\mathfrak{I}\mid \mathfrak{I} \subseteq \mathfrak{I}_1\text{ and } \mathfrak{I} \subseteq \mathfrak{I}_2\}$. 
The empty interpretation $\mathfrak{I}_{\emptyset}$ maps $t$ to $\emptyset$ for each $t\in\mathbb{Q}$, and it is contained by any interpretation. 
Finally, we write $\mathfrak{I}_3 = \mathfrak{I}_1-\mathfrak{I}_2$, if $\mathfrak{I}_3$ is the greatest in $\{\mathfrak{I}\mid \mathfrak{I}\subseteq\mathfrak{I}_1, \mathfrak{I}\cap\mathfrak{I}_2=\mathfrak{I}_{\emptyset}\}$.


\subsubsection{Materialisation for DatalogMTL}


We now briefly discuss how materialisation-based reasoning works for DatalogMTL with bounded intervals. To this end, we first introduce a few additional notations that will facilitate our discussion. For $\mathcal{D}$ a dataset, $t_{\mathcal{D}}^-$ and $t_{\mathcal{D}}^+$ denote the minimal and maximal interval endpoints appearing in $\mathcal{D}$, respectively, and $t_{\mathcal{D}}^- = t_{\mathcal{D}}^+ = 0$ if $\mathcal{D}$ mentions no numbers. Moreover, the depth of a rule $r$, written $\mathsf{depth}(r)$, is defined as the sum of right endpoints of all intervals appearing in the operators of $r$, and $\mathsf{depth}(r) = 0$ if $r$ mentions no intervals; for $\Pi$ a program, $\mathsf{depth}(\Pi)$ is the maximum depth of its rules.

The immediate consequence operator $T_\Pi$ for a program $\Pi$ is the function that maps an interpretation $\mathfrak{I}$ to the least interpretation $T_\Pi(\mathfrak{I})$ containing $\mathfrak{I}$ and satisfying the following: for each $r_0$ a ground rule instance of a rule in ${\Pi}$ and each $t$ a time point, ${\mathfrak{I}, t \vDash body(r_0)}$ implies ${\mathfrak{I}, t \vDash head(r_0)}$. For a program $\Pi$ and a dataset $\mathcal{D}$, a transfinite sequence of interpretations $T^\alpha_\Pi(\mathfrak{I}_\mathcal{D})$ can be defined for ordinals $\alpha$ by successively applying $T_\Pi$, the immediate consequence operator for $\Pi$, to $\mathfrak{I}_\mathcal{D}$, the unique least model for $\mathcal{D}$, as follows:
\begin{enumerate*}[label=(\roman*)]
\item $T^0_\Pi(\mathfrak{I}_\mathcal{D}) = \mathfrak{I}_\mathcal{D}$,
\item $T^{\alpha+1}_\Pi(\mathfrak{I}_\mathcal{D}) = T_\Pi(T^\alpha_\Pi(\mathfrak{I}_\mathcal{D}))$ for $\alpha$ an ordinal, and
\item $T^\alpha_\Pi(\mathfrak{I}_\mathcal{D}) = \bigcup_{\beta<\alpha} T^\beta_\Pi(\mathfrak{I}_\mathcal{D})$ for $\alpha$ a limit ordinal
\end{enumerate*};
the canonical model $\mathfrak{C}_{\Pi,\mathcal{D}}$ of $\Pi$ and $\mathcal{D}$ is the interpretation $T^{\omega_1}_\Pi(\mathfrak{I}_\mathcal{D})$ with $\omega_1$ the first uncountable ordinal. In fact, $\mathfrak{C}_{\Pi,\mathcal{D}}$ is the least model of $\Pi$ and $\mathcal{D}$~\cite{Brandt2017}.

For an interpretation $\mathfrak{I}$ and an interval $\varrho$, the projection $\mathfrak{I}\mid_{\varrho}$ of $\mathfrak{I}$ over $\varrho$ is the interpretation that coincides with $\mathfrak{I}$ on $\varrho$ and maps all relational atoms to false outside $\varrho$. 
Moreover, an interpretation $\mathfrak{I}$ is a shift of another interpretation $\mathfrak{I}'$, if there is a rational number $s$ such that for each $R(\bm{t})$ a ground relational atom and each $t$ a time point, $\mathfrak{I}', t\models R(\bm{t})$ implies $\mathfrak{I}, t+s\models R(\bm{t})$, and vice versa. Finally, for a given program $\Pi$ and a dataset $\mathcal{D}$, we are not interested in dealing with arbitrary rational number time points on the time line; instead, it is sufficient to handle time points that are somewhat related to those appearing in $\Pi$ and $\mathcal{D}$. This is formalised by the notion of $(\Pi, \mathcal{D})$-ruler. Concretely, $(\Pi, \mathcal{D})$-ruler is the set of time points of the value $t + i \times div(\Pi)$, where $t$ is an endpoint mentioned in $\mathcal{D}$ and $i$ is an integer, and $div(\Pi) = 1/k$ , with $k$ being the product of all denominators
in the rational endpoints mentioned in $\Pi$; for generality, $k = 1$ and $div(\Pi) = 1$ if $\Pi$ has no mention of rational endpoints.

We are now ready to define the notions of saturated interpretation and unfolding, which are key components for finitely representing the materialisation of a bounded program-dataset pair. 

\begin{definition} \label{def:saturated}
    For a program $\Pi$ and a dataset ${\mathcal{D}}$, interpretation $\textit{T}_{\Pi}^{k}(\mathfrak{I}_{\mathcal{D}})$ is saturated if there exist closed intervals $\varrho_1,\varrho_2,\varrho_3$ and $\varrho_4$ of length $2\mathsf{depth}(\Pi)$, whose endpoints are located on the $(\Pi,\mathcal{D})$-ruler and satisfy $\varrho_1^+<\varrho_2^+<t_{\mathcal{D}}^-$ and $t_{\mathcal{D}}^+<\varrho_{3}^-<\varrho_{4}^-$, and such that the following properties hold:
    \begin{itemize}
        \item $\textit{T}_{\Pi}^{k}(\mathfrak{I}_{\mathcal{D}})$ satisfies $\Pi$ in $[\varrho_1^-,\varrho_4^+]$;
        \item $\textit{T}_{\Pi}^{k}(\mathfrak{I}_{\mathcal{D}}) \mid_{\varrho_1}$ and $\textit{T}_{\Pi}^{k}(\mathfrak{I}_{\mathcal{D}}) \mid_{\varrho_3}$ are shifts of $\textit{T}_{\Pi}^{k}(\mathfrak{I}_{\mathcal{D}}) \mid_{\varrho_2}$  and $\textit{T}_{\Pi}^{k}(\mathfrak{I}_{\mathcal{D}}) \mid_{\varrho_4}$,respectively.
    \end{itemize}
    $[\varrho_1^-,\varrho_2^-)$ and $(\varrho_3^+,\varrho_4^+]$ are often referred to as the left period, written ${\varrho_\mathsf{left}}$, and the right period, written ${\varrho_\mathsf{right}}$, of the interpretation $\textit{T}_{\Pi}^{k}(\mathfrak{I}_{\mathcal{D}})$, respectively.
\end{definition}
\begin{definition}
    The ($\varrho_\mathsf{left},\varrho_\mathsf{right}$)-unfolding of a saturated interpretation $\textit{T}_{\Pi}^{k}(\mathfrak{I}_{\mathcal{D}})$ with periods ($\varrho_\mathsf{left},\varrho_\mathsf{right}$) is the interpretation $\mathfrak{C}$ such that:
    \begin{itemize}
        \item $\mathfrak{C}\mid_{[\varrho_\mathsf{left}^-,\varrho_\mathsf{right}^+]}=\textit{T}_{\Pi}^{k}(\mathfrak{I}_{\mathcal{D}})\mid_{[\varrho_\mathsf{left}^-,\varrho_\mathsf{right}^+]}$,
        \item $\mathfrak{C}\mid_{\varrho_\mathsf{left}-n\cdot\lvert \varrho_\mathsf{left} \rvert}$ is a shift of $\textit{T}_{\Pi}^{k}(\mathfrak{I}_{\mathcal{D}})\mid_{\varrho_\mathsf{left}}$, for any $n \in \mathbb{N}$,
        \item $\mathfrak{C}\mid_{\varrho_\mathsf{right}+n\cdot\lvert \varrho_\mathsf{right} \rvert}$ is a shift of $\textit{T}_{\Pi}^{k}(\mathfrak{I}_{\mathcal{D}})\mid_{\varrho_\mathsf{right}}$, for any $n \in \mathbb{N}$.
    \end{itemize}
\end{definition}

It has been shown that for an arbitrary pair of bounded program and dataset $(\Pi,\mathcal{D})$, there exists $k \in \mathbb{N}$ and intervals ${\varrho_\mathsf{left}}$ and ${\varrho_\mathsf{right}}$ such that $\textit{T}_{\Pi}^{k}(\mathfrak{I}_{\mathcal{D}})$ is a saturated interpretation with periods $(\varrho_\mathsf{left},\varrho_\mathsf{right})$, and the $(\varrho_\mathsf{left},\varrho_\mathsf{right})$-unfolding of $\textit{T}_{\Pi}^{k}(\mathfrak{I}_{\mathcal{D}})$ coincides with $\mathfrak{C}_{\Pi,\mathcal{D}}$, the canonical model of ${\Pi}$ and ${\mathcal{D}}$~\cite{DBLP:conf/aaai/WalegaZWG23}. Intuitively speaking, a saturated interpretation finitely represents the canonical model of a program and a dataset. The goal of this work is to incrementally update a saturated interpretation in response to changes in the explicitly given dataset $\mathcal{D}$.

\subsubsection{The Delete/Rederive Algorithm}

To incrementally update saturated interpretations for DatalogMTL, we draw inspirations from the Delete/Rederive (DRed) algorithm \cite{10.1145/170035.170066,10.5555/645922.673479}, a well-known technique for maintaining Datalog materialisations. Given a Datalog program, a set of explicitly given facts, the original materialisation, and sets of facts to remove from and to add to the given facts, the DRed algorithm updates the materialisation to reflect changes in the explicitly given facts, without recomputing it from scratch. More specifically, the algorithm operates in three stages: in the overdeletion stage, the algorithm eagerly identifies all facts that depend on the set of deleted facts; it then enters the rederivation stage, in which it recognises which of the overdeleted facts can be rederived in one step; finally, during  insertion, the algorithm computes the consequences of both the rederived facts and the newly inserted ones.

\section{Motivation}
Before we present the technical details of our incrmental update approach for DatalogMTL, in this section, we discuss at a high level challenges that need to be addressed in devising such an incremental update approach. We also provide key insights behind our incremental update algorithm. Finally, we provide a concrete example that highlights the benefits of applying our incremental maintenance algorithm for DatalogMTL reasoning compared with recomputing the materialisation from scratch.








In extending DRed for Datalog to support DatalogMTL reasoning, two major challenges arise. First, DRed relies on the semina\"ive evaluation strategy applied in both the overdeletion and the insertion stages to be efficient in dealing with updates. More specifically, in each round of rule application, at least one body atom is required to be evaluated in the set of facts deleted/inserted in the previous round; this strategy ensures that only consequences dependent on the deleted/inserted facts are considered. Although semina\"ive evaluation has been considered for DatalogMTL materialisation, its interplay with the construction of saturated interpretations in the context of DatalogMTL reasoning has not been studied, let alone the application of semina\"ive evaluation strategy in DatalogMTL materialisation maintenance. Second, for the Datalog setting, the bulk of the work is devoted to iterative rule application, so it is sufficient to make rule application process `incremental' in order to obtain an algorithm that is efficient for updates; however, in the DatalogMTL setting, a significant amount of time needs to be dedicated to periods identification (i.e., termination checks). Intuitively, after each round of rule application, the period identification procedure enumerates all pairs of intervals of a certain length, and then it compares pairwisely the facts in these intervals to detect the presence of a repeating pattern. Once two intervals coincide in their content, a period is identified. As such, the cost of this procedure naturally depends on the number of facts inside these intervals. Therefore, it is essential that the periods identification process is also designed to be somewhat `incremental', so that the overall processing time aligns with the size of the updates rather than the size of the entire materialisation; this would require thorough understanding and careful treatment of the periodic structure of DatalogMTL materialisations.

Our incremental maintenance technique addresses both of these two challenges. For the first challenge, we devise a novel semina\"ive evaluation operator that is tailored for incremental updates in the DatalogMTL setting: compared with the existing semina\"ive operator, our new operator identifies more accurately the consequences affected by the update and also facilitates period identification. To tackle the second challenge, our algorithm is designed to identify periods in the updated facts and their consequences, rather than in the entire materialisation. Typically, deletions and insertions are small compared with the entire materialisation. As such, computing periods over the updated facts tend to be less costly than doing so for the entire materialisation; thus, our approach has the potential to significantly reduce redundant computations, especially in scenarios where updates only affect a limited portion of the materialisation. We use the following example to highlight the benefits of our approach; throughout the technical section, this example will be expanded to illustrate key operations in our algorithm.

\begin{example}
\label{ex1}

Let program $\Pi$ consist of a single rule~\eqref{ex1-1}, and let $E$, $E^-$, and $E^+$ be the original dataset, the set of facts to remove from $E$, and set of facts to add to $E$, respectively. Now consider the materialisation of $\Pi$ over $E$. According to the definition of $\mathsf{depth}(\Pi)$ and Definition~\ref{def:saturated}, it is easy to verify that $\mathsf{depth}(\Pi) = 11$, and that when $k = 4$, interpretation $\textit{T}_{\Pi}^{k}(\mathfrak{I}_E)$ is saturated, with ${\varrho_\mathsf{left} = [-24,-23)}$ and ${\varrho_\mathsf{right} = [24,34)}$ being possible left and right periods, respectively. To arrive at the conclusion that ${\varrho_\mathsf{right} = [24,34)}$, the materialisation procedure needs to verify that $\textit{T}_{\Pi}^{k}(\mathfrak{I}_{E}) \mid_{[2,24]}$ is a shift of $\textit{T}_{\Pi}^{k}(\mathfrak{I}_{E})\mid_{[12,34]}$, and this clearly requires comparing $O(n)$ facts. To summarise, the materialisation procedure requires $O(n)$ time.

\begin{equation}\label{ex1-1}
    \boxplus_{[0,1]} R(x)\gets \boxminus_{[9,10]}R(x)
\end{equation}
\begin{equation*}
    E=\{R(a_i)@[0,1]\mid 0<i<n \}
\end{equation*}
\begin{equation*}
    E^{-}=\{R(a_1)@[0,1] \}
\end{equation*}
\begin{equation*}
    E^{+}=\{R(a_n)@[0,1] \}
\end{equation*}

Now consider the removal of $E^-$ and the insertion of $E^+$. If we recompute the materialisation from scratch over the entire set $(E \setminus E^-) \cup E^{+}$, the procedure again requires $O(n)$ time. In contrast, our incremental update approach restricts the attention to only the updated facts and their consequences. For deletion, our approach tries to identify periodic structure only among facts of the form ${R(a_1) @ \varrho}$. In our example, after a constant number of rule application, our algorithm will be able to identify the relevant left and right periods by only comparing $O(1)$ facts instead of $O(n)$ facts. In other words, (over)deletion requires only $O(1)$ time. Insertion follows the same principle, and so we omit the details.

\end{example}


\section{DRed for DatalogMTL}
Before we present the details of our algorithm, we first introduce the notion of periodic materialisations, which are finite representations of possibly unbounded models of DatalogMTL programs. Our update algorithm will have to incrementally maintain periodic materialisations in response to changes in the explicitly given data.



\begin{definition}~\label{def:periodicmat}
For $(\Pi, E)$ a bounded DatalogMTL program-dataset pair, a periodic materialisation of $(\Pi, E)$ is a triple $\mathds{I}$ of the form $\langle I, \varrho_{\mathsf{L}}, \varrho_{\mathsf{R}} \rangle$, where $I$ is a saturated interpretation of $\Pi$ and $E$, and $\varrho_{\mathsf{L}}$ and $\varrho_{\mathsf{R}}$ are periods of this saturated interpretation; moreover, the unfolding of ${\mathds{I}}$, written $\mathsf{unfold}(\mathds{I})$, is defined as the $(\varrho_\mathsf{L},\varrho_\mathsf{R})$-unfolding of $I$.
\end{definition}

By Definition~\ref{def:periodicmat} and properties already discussed in the Preliminary Section, if $\mathds{I}$ is a periodic materialisation of a program-dataset pair $(\Pi, E)$, then $\mathsf{unfold}(\mathds{I})$ coincides with $\mathfrak{C}_{\Pi,E}$, the canonical model of $\Pi$ and $E$.
Note that $\mathds{I}$ can capture the interpretation of a bounded dataset $I$, in which case $\mathds{I}=\langle I, \varrho_{\mathsf{L}}, \varrho_{\mathsf{R}} \rangle$ such that $I\mid_{\varrho_{\mathsf{L}}}=I\mid_{\varrho_{\mathsf{R}}} = \mathfrak{I}_{\emptyset}$, and  $\mathsf{unfold}(\mathds{I})=I$.

\begin{algorithm}[!htb]
\DontPrintSemicolon
\caption{\textbf{\textsc{DRed}$_{MTL}$$(\Pi,E, \mathds{I}, E^{-},E^{+})$}}
\label{algo:DREDMTL}
\SetAlgoNoLine
\SetAlgoNoEnd
\SetKwProg{Procedure}{procedure}{}{}
\SetKwProg{Loop}{loop}{}{}
\SetKw{Break}{break}
\KwIn{program $\Pi$, original dataset $E$, periodic materialisation $\mathds{I}$ of $(\Pi, E)$, dataset to remove $E^{-}$, dataset to insert $E^{+}$}

\vspace{0.5cm}  

$D\coloneqq R\coloneqq A\coloneqq \emptyset, \; E^{-} \coloneqq (E^{-} \cap E) \setminus E^{+}, \; E^{+} \coloneqq E^{+}  \setminus E$

\textsc{Overdelete}

\textsc{Rederive}

\textsc{Insert}

\vspace{0.2cm}  

\Procedure{\textsc{Overdelete}}{
    $N_{D} \coloneqq E^{-}$
    
    \Loop{}{
        $\Delta_D\coloneqq N_{D} \setminus D$
        
        $(\varrho^D_{\mathsf{L}}, \; \varrho^D_{\mathsf{R}}) \coloneqq \textsc{Pds}(\Pi,E,\mathds{I}.\varrho_{\mathsf{L}}, \mathds{I}.\varrho_{\mathsf{R}},D,\Delta_{D})$ 

        \lIf{ ${(\varrho^D_{\mathsf{L}}, \; \varrho^D_{\mathsf{R}})} \neq (\emptyset, \emptyset)$ }{\Break}
        
        $N_{D} \coloneqq \Pi \langle \mathsf{unfold}(\mathds{I}) \setminus D \myvdots \Delta_{D}\rangle $ \label{delete-seminaive-operator}
        
        $D \coloneqq D \cup \Delta_{D}$
    }
    $\mathds{D}\coloneqq \langle D,\varrho^D_{\mathsf{L}},\varrho^D_{\mathsf{R}} \rangle$, \quad $\mathds{I} \coloneqq \mathds{I} \dblsetminus \mathds{D}$
}

\vspace{0.2cm}  

\Procedure{\textsc{Rederive}}{


    ${N_R \coloneqq \emptyset}$, \quad $k = 1$
    
    \Loop{}{

        $t_{\mathsf{L}} = (\mathds{I}.{\varrho_\mathsf{L}})^+ - k \times \mathsf{max}(2\mathsf{depth}(\Pi), |\mathds{I}.\varrho_{\mathsf{L}}|)$

        $t_{\mathsf{R}} = (\mathds{I}.{\varrho_\mathsf{R}})^- + k \times \mathsf{max}(2\mathsf{depth}(\Pi), |\mathds{I}.\varrho_{\mathsf{R}}|)$
        
        $N_R \coloneqq N_R \cup (\mathsf{unfd}(\mathds{D})|_{[t_{\mathsf{L}}, t_{\mathsf{R}}]} \cap T_{\Pi}(\mathsf{unfd}(\mathds{I})))$
        
            $\Delta_{R} \coloneqq N_{R}  \setminus R $
        
        $(\varrho^R_{\mathsf
L},\varrho^R_{\mathsf
R}) \coloneqq \textsc{Pds}(\Pi,E, \mathds{I}.\varrho_{\mathsf{L}}, \mathds{I}.\varrho_{\mathsf{R}},R,\Delta_{R})$

        \lIf{ ${(\varrho^R_{\mathsf{L}}, \; \varrho^R_{\mathsf{R}})} \neq (\emptyset, \emptyset)$ }{\Break}
        
        $R \coloneqq R \cup \Delta_{R}$, \quad $k \coloneqq k + 1$

        $N_{R} \coloneqq \Pi \langle \mathsf{unfold}(\mathds{I}) \cup R \myvdots \Delta_{R}\rangle$ \label{rederive-seminaive-operator}
    }
    $\mathds{R}\coloneqq \langle R,\varrho^{R}_{\mathsf{L}},\varrho^{R}_{\mathsf{R}}\rangle$, \quad $\mathds{I} \coloneqq \mathds{I} \Cup \mathds{R}$
}

\vspace{0.2cm}

\Procedure{\textsc{Insert}}{
    
        

        
        $N_{A} \coloneqq E^{+}$, \quad $E = (E \setminus E^-) \cup E^+$
    
    \Loop{}{
        $\Delta_{A} \coloneqq N_{A}  \setminus (\mathsf{unfold}(\mathds{I}) \cup A)$
        
        $(\varrho^A_{\mathsf
L},\varrho^A_{\mathsf
R}) \coloneqq \textsc{Pds}(\Pi,E, \mathds{I}.\varrho_{\mathsf{L}}, \mathds{I}.\varrho_{\mathsf{R}},A,\Delta_{A})$

        \lIf{ ${(\varrho^A_{\mathsf{L}}, \; \varrho^A_{\mathsf{R}})} \neq (\emptyset, \emptyset)$ }{\Break}
        
        $A \coloneqq A \cup \Delta_{A}$
        
        $N_{A} \coloneqq \Pi \langle \mathsf{unfold}(\mathds{I}) \cup A \myvdots \Delta_{A}\rangle$ \label{insert-seminaive-operator}
    }
    $\mathds{A}\coloneqq \langle A,\varrho^{A}_{\mathsf{L}},\varrho^{A}_{\mathsf{R}}\rangle$, \quad $\mathds{I} \coloneqq \mathds{I} \Cup \mathds{A}$
}
\end{algorithm}

\subsubsection{Algorithm Overview} Our incremental materialisation maintenance algorithm for DatalogMTL is formalised in 
Algorithm~\ref{algo:DREDMTL}. The algorithm takes as input a program $\Pi$, a dataset $E$, a periodic materialisation ${\mathds{I}}$ of ${(\Pi,E)}$, a dataset $E^-$ to be removed from $E$, and a dataset $E^+$ to be added to $E$, and it updates $\mathds{I}$ such that after the execution of the algorithm, $\mathds{I}$ becomes a  periodic materialisation of ${(\Pi, (E\setminus E^-) \cup E^+)}$, and this is achieved without recomputing the periodic materialisation from scratch. 
Similar to the DRed algorithm for plain Datalog, Algorithm~\ref{algo:DREDMTL} consists of three stages, which we outline below.

In the overdeletion stage, the dataset $D$ is extended with all facts that depend on the deleted facts. The algorithm then applies the rules of $\Pi$ iteratively (line~11)  until a pair of periods are found (line~9--10). Line~11 makes use of our newly devised semina\"ive evaluation operator $\Pi\langle I \myvdots \Delta\rangle$, which computes the union of all facts derived by $(r\sigma, t) \langle I \myvdots \Delta \rangle$, where $r \sigma$ is a ground rule instance with $r$ a rule in $\Pi$ and $\sigma$ a substitution mapping each variable appearing in $r$ to a constant appearing in $I$, and $t$ is a rational time point; operator $(r', t) \langle I \myvdots \Delta \rangle$ where $r'$ is a ground rule instance and $t$ is a rational time point is in turn defined as the minimal set of punctual facts $N$ such that $\mathfrak{I}_{I}, t \models body(r')$, $\mathfrak{I}_{I \setminus \Delta}, t \not \models body(r')$, and $\mathfrak{I}_{N}, t \models head(r')$. When $I$ and $\Delta$ are finite, $\Pi\langle I \myvdots \Delta\rangle$ can be efficiently evaluated by instantiating the query from facts inside $\Delta$, which is typically much smaller than $I$. However, in line~21, by slight abuse of notation, we use ${\mathsf{unfold}(\mathds{I}) \setminus D}$ as an operand of the operator: this does not mean that we have to compute the entire unfolding of ${\mathds{I}}$ prior to the execution of the operator. In contrast, the evaluation is still instantiated from from facts in ${\Delta}$, while interpretation ${\mathds{I}}$ can be unfolded lazily as required.


In the rederivation stage, the algorithm recovers the facts that were overdeleted but should actually still hold after the update. It should be noted that overdeleted facts may span the entire timeline, and so it may not be sufficient to recover facts within a fixed interval. Lines~17--19 address this issue, in each round of the loop of lines~16--24, we extend the interval of interest $[t_{\mathsf{L}}, t_{\mathsf{R}}]$ so that overdeleted facts inside this interval could be successfully recovered (line~19); this is done in parallel to the propagation of already recovered facts (line~24). The loop terminates when a pair of periods are successfully identified (lines~21--22), and the periodic materialisation $\mathds{I}$ is updated in line~25.


Insertion is analogous to deletion: the dataset $A$ is populated with new facts that are derivable from $E^+$. In parallel to consequence propagation, the algorithm tries to detect the periodic structure within the dataset $A$. Ultimately, the algorithm updates the periodic materialisation in line~34 to incorporate the inserted facts and their consequences.

\begin{algorithm}[t]
\DontPrintSemicolon
\caption{\textsc{Pds}$(\Pi, E,\varrho_\mathsf{L},\varrho_\mathsf{R}, D, \Delta_D)$}
\label{periods}
\SetAlgoNoLine

\KwIn{program $\Pi$, dataset $E$, intervals $\varrho_\mathsf{L}$ and $\varrho_\mathsf{R}$, dataset $D$ of facts derived so far, and dataset $\Delta_D$ of facts derived in the last round}
\vspace{0.5cm}

$\varrho'_\mathsf{L} \coloneqq \varrho'_\mathsf{R} \coloneqq \emptyset$, \quad $P\coloneqq (\Pi,E)$-ruler

\ForEach{$M\sigma@\varrho \in \Delta_D$}
    {\If{$\varrho \cap [t_E^-, t_E^+] \neq \emptyset$}
        {
        \Return $(\emptyset, \emptyset)$\;
}}

$u = \max(\{x < t_E^- \mid \exists y, M@\langle y,x\rangle \in \Delta_D\}\cup\{-\infty\})$

$v = \min(\{x > t_E^+ \mid \exists y,  M@\langle x,y\rangle \in \Delta_D\}\cup \{\infty\})$

$\gamma_\mathsf{L} \coloneqq (u, \quad \varrho_\mathsf{L}^++2\mathsf{depth}(\Pi)]$

$\gamma_\mathsf{R} \coloneqq [\varrho_\mathsf{R}^--2\mathsf{depth}(\Pi), \quad v)$

\ForEach{$\varrho_1,\varrho_2 \subseteq \gamma_{\mathsf{L}}$ with $|\varrho_1| = |\varrho_2| = 2\mathsf{depth}(\Pi)$}{
    \If{$\varrho_1^-<\varrho_2^-,\varrho_1^-\in P$ and $\varrho_2^-\equiv\varrho_1^- \mod |\varrho_\mathsf{L}|$}{
    \If{$\mathfrak{I}_D \mid_{\varrho_1}$ is a shift of $\mathfrak{I}_D \mid_{\varrho_2}$}{
        $\varrho'_\mathsf{L} \coloneqq [\varrho_1^-, \varrho_2^-)$\;
}}}

\ForEach{$\varrho_3,\varrho_4 \subseteq \gamma_{\mathsf{R}}$ with $|\varrho_3| = |\varrho_4| = 2\mathsf{depth}(\Pi)$}{
    \If{$\varrho_3^+<\varrho_4^+,\varrho_3^+\in P$ and $\varrho_3^+\equiv\varrho_4^+ \mod |\varrho_\mathsf{R}|$}{
    \If{$\mathfrak{I}_D \mid_{\varrho_3}$ is a shift of $\mathfrak{I}_D \mid_{\varrho_4}$}{
        $\varrho'_\mathsf{R} \coloneqq (\varrho_3^+, \varrho_4^+]$\;
}}}

\lIf*{$\varrho'_\mathsf{L} = \emptyset$ or $\varrho'_\mathsf{R} = \emptyset$}{
    \Return $(\emptyset, \emptyset)$\;
}
\lElse*{
    \Return $(\varrho'_\mathsf{L}, \varrho'_\mathsf{R})$\;
}

\end{algorithm}


\subsubsection{Period Identification} The \textsc{Pds} procedure formalised in Algorithm~\ref{periods} identifies the repeating patterns within a dataset being populated during iterative application of rules. More specifically, it takes as input a program $\Pi$, a dataset $E$, intervals $\varrho_\mathsf{L}$ and $\varrho_\mathsf{R}$, the dataset being popuplated $D$, and a set of new facts $\Delta_D$ that are to be integrated into $D$ after period detection. The procedure first examines facts in $\Delta_D$: if there is a fact in $\Delta_D$ that overlaps with the interval $[t_E^-, t_E^+]$, it means facts inside this interval have not stabilised, and so the procedure terminates. Otherwise, $u$ and $v$ will be computed separately (lines~5--6) such that $[u,v]$ is largest interval containing $[t^-_E,t_E^+]$ and does not contain a fact satisfied by $\mathfrak{I}_{\Delta_D}$. Now if $[u,v]$ is nonempty, we search for the left and right periods with endpoints on the $(\Pi,E)-$ ruler in $\gamma_\mathsf{L} = (u, \varrho_\mathsf{L}^+ + 2\mathsf{depth}(\Pi)]$ and $\gamma_\mathsf{R} = [\varrho_\mathsf{right}^- - 2\mathsf{depth}(\Pi), v)$, respectively (lines~9--16). 
The procedure returns the pair of periods $(\varrho'_{\mathsf{L}},\varrho'_{\mathsf{R}})$, or $(\emptyset, \emptyset)$ if no valid period is detected at either end (lines~17--18). Intuitively, the algorithm requires $\varrho_1,\varrho_2$ (resp., $\varrho_3,\varrho_4$) to be a multiple of $|\varrho_{\mathsf{L}}|$ (resp., $|\varrho_{\mathsf{R}}|$) apart so that it would be convenient align the new periods with the given ones.

\subsubsection{Implementation of Periodic Operators}
Algorithm~\ref{algo:DREDMTL} has made frequent use of operators $\dblsetminus$ and $\Cup$, which are responsible for taking the difference and union of two periodic materialisations, respectively. In practice, these two operators can be implemented arbitrarily, so long as $\mathsf{unfold}(\mathds{D}_1)-\mathsf{unfold}(\mathds{D}_2) = \mathsf{unfold}(\mathds{D}_1 \dblsetminus \mathds{D}_2)$ and $\mathsf{unfold}(\mathds{D}_1)\cup\mathsf{unfold}(\mathds{D}_2) = \mathsf{unfold}(\mathds{D}_1 \Cup \mathds{D}_2)$. Next we describe our implementation of these operators, which utilises two auxiliary functions, $\texttt{Ext}$ and $\texttt{Aln}$, responsible for extending and aligning periodic materialisations, respectively.

\begin{algorithm}[t]
\DontPrintSemicolon
\caption{Implementation of Periodic Minus}
\label{algo:periodic_minus}
\KwIn{Two periodic materialisations $\mathds{D}_1$ and $\mathds{D}_2$}


\vspace{0.2cm}

\If{$\mathds{D}_2.\varrho_{\mathsf{L}}$ }{
    $(\mathds{D}_1, \mathds{D}_2) \coloneqq \texttt{Ext}(\mathds{D}_1, \mathds{D}_2, \mathsf{L})$ \;
    $(\mathds{D}_1, \mathds{D}_2, \varrho_\mathsf{L}) \coloneqq \texttt{Aln}(\mathds{D}_1, \mathds{D}_2, \mathsf{L})$ \;
}
\lElse*{
    $\varrho_\mathsf{L} \coloneqq \mathds{D}_1.\varrho_\mathsf{L}$ \;
}

\If{$\mathds{D}_2.\varrho_\mathsf{R}$}{
    $(\mathds{D}_1, \mathds{D}_2) \coloneqq \texttt{Ext}(\mathds{D}_1, \mathds{D}_2, \mathsf{R})$ \;
    $(\mathds{D}_1, \mathds{D}_2, \varrho_\mathsf{R}) \coloneqq \texttt{Aln}(\mathds{D}_1, \mathds{D}_2, \mathsf{R})$ \;
}
\lElse*{
    $\varrho_\mathsf{R} \coloneqq \mathds{D}_1.\varrho_\mathsf{R}$ \;
}

$I \coloneqq \mathds{D}_1.I \setminus \mathds{D}_2.I$ \;

\Return $\langle I, \varrho_\mathsf{L}, \varrho_\mathsf{R} \rangle$
\end{algorithm}


Consider two periodic materialisations $\mathds{D}_1$ and 
$\mathds{D}_2$, and let $\mathsf{end} \in \{ \mathsf{L}, \mathit{R} \}$ denote which end of the periodic materialisation to operate on (either left or right). Given periodic materialisations $\mathds{D}_1$ and  $\mathds{D}_2$ which both have valid (but potentially different) periods $\varrho_\mathsf{end}$, function $\texttt{Ext}$ computes a pair of periodic materialisations $(\mathds{D}'_1, \mathds{D}'_2) = \texttt{Ext}(\mathds{D}_1, \mathds{D}_2, \mathsf{end})$ such that the periodic regions are extended to have the same length, which is the least common multiple (LCM) of the two; this can be easily achieved by copying the relevant periodic segments of the timeline. In contrast, function $\texttt{Aln}$ tries to align 
the start and end points of periodic intervals of the two periodic materialisations at the target end, again through facts copying. If only one of the two periodic operators has a valid $\varrho_{\mathsf{end}}$, an empty period is introduced for the other periodic materialisation so that alignment can still be performed. It should be noted that both $\texttt{Ext}$ and $\texttt{Aln}$ operations preserve the semantics of the interpretations: they do not change which facts hold at any time point but only alters the finite representations of the input periodic materialisations.

\begin{algorithm}[t]
\DontPrintSemicolon
\caption{Implementation of Periodic Union}
\label{algo:periodic_union}
\KwIn{Two periodic materialisations $\mathds{D}_1$ and $\mathds{D}_2$}

\vspace{0.2cm}

$\varrho_{\mathsf{L}} \coloneqq \varrho_{\mathsf{R}} \coloneqq \emptyset$

\If{$\mathds{D}_1.\varrho_{\mathsf{L}}$ and $\mathds{D}_2.\varrho_{\mathsf{L}}$}{
         $(\mathds{D}_1, \mathds{D}_2) \coloneqq \texttt{Ext}(\mathds{D}_1, \mathds{D}_2, \mathsf{L})$
    
}

\If{$\mathds{D}_1.\varrho_{\mathsf{L}}$ or $\mathds{D}_2.\varrho_{\mathsf{L}}$}{
    
    $(\mathds{D}_1, \mathds{D}_2, \varrho_{\mathsf{L}}) \coloneqq \texttt{Aln}(\mathds{D}_1, \mathds{D}_2, \mathsf{L})$ \;
}

\If{$\mathds{D}_1.\varrho_\mathsf{R}$ and $\mathds{D}_2.\varrho_\mathsf{R}$}{
        $(\mathds{D}_1, \mathds{D}_2) \coloneqq \texttt{Ext}(\mathds{D}_1, \mathds{D}_2, \mathsf{R})$
    
}

\If{$\mathds{D}_1.\varrho_\mathsf{R}$ or $\mathds{D}_2.\varrho_\mathsf{R}$}{
    
    $(\mathds{D}_1, \mathds{D}_2, \varrho_\mathsf{R}) \coloneqq \texttt{Aln}(\mathds{D}_1, \mathds{D}_2, \mathsf{R})$ \;
}

$I \coloneqq \mathds{D}_1.I \cup \mathds{D}_2.I$ 

\Return $\langle I, \varrho_{\mathsf{L}}, \varrho_{\mathsf{R}}\rangle$
\end{algorithm}

Theorem~\ref{theorem:periodicoperatorscorrect} states that Algorithms~\ref{algo:periodic_minus} and~\ref{algo:periodic_union} correctly implement operators $\dblsetminus$ and $\Cup$, respectively; Theorem~\ref{maintheorem} states that Algorithm~\ref{algo:DREDMTL} is correct. The full proofs for these theorems are lengthy, so we only provide proof sketches; the full proofs are given in the online technical report. 

\begin{restatable}{thm}{theoremoperatorcorrectness}\label{theorem:periodicoperatorscorrect}
    If given input $\mathds{D}_1$,$\mathds{D}_2$, Algorithm~\ref{algo:periodic_minus} and~\ref{algo:periodic_union} output $\mathds{D}_m$ and $\mathds{D}_u$, respectively, then $\mathsf{unfold}(\mathds{D}_1)-\mathsf{unfold}(\mathds{D}_2) = \mathsf{unfold}(\mathds{D}_m)$ and $\mathsf{unfold}(\mathds{D}_1)\cup\mathsf{unfold}(\mathds{D}_2) = \mathsf{unfold}(\mathds{D}_u)$
\end{restatable}
\begin{psketch}
For minus, we show that the equation holds by leveraging that the periods of $\mathsf{unfold}(\mathds{D}_1)$ and $\mathsf{unfold}(\mathds{D}_2)$ can be extended and aligned to form a new pair of periods, and their difference also abides by the new periods. The case for union is similar. 
\end{psketch}
\vspace{-0.1cm}
\begin{restatable}{thm}{theoremalgorithmcorrectness} \label{maintheorem}
    For a bounded DatalogMTL program $\Pi$, a bounded dataset $E$, a bounded deleting dataset $E^-$, a bounded inserting dataset $E^+$, a periodic materialisation $\mathds{I}$ such that $\mathsf{unfold}(\mathds{I}) = \mathfrak{C}_{\Pi,E}$, let $E'=E\setminus E^-\cup E^+$, then after calling \textbf{\textsc{DRed}$_{MTL}$$(\Pi,E,E^{-},E^{+},\mathds{I})$}, we have
    \begin{align*}
    \mathsf{unfold}(\mathds{I})= \mathfrak{C}_{\Pi,E'}
    \end{align*}
\end{restatable}
\begin{psketch}
    We first show that $\mathsf{unfold}(\mathds{D})$ computed by the \textsc{Overdelete} contains $\mathfrak{C}_{\Pi,E}-\mathfrak{C}_{\Pi,E\setminus E^-\cup E^+}$, which means every fact that no longer holds because of deleting $E^-$ will be removed, and then we show that the union of  $\mathsf{unfold}(\mathds{R}),\mathsf{unfold}(\mathds{A})$ computed by the \textsc{Rederive},\textsc{Insert} and $\mathsf{unfold}(\mathds{I})$ after deletion equals $\mathfrak{C}_{\Pi,E\setminus E^-\cup E^+}$, which involves proving mistakenly deleted facts are rederived and every new fact that holds because of $E^+$ are inserted. 
\end{psketch}
\section{Evaluation}

We implemented the proposed DRed$_\text{MTL}$ algorithm based on an efficient DatalogMTL reasoner MeTeoR \cite{DBLP:conf/aaai/WangHWG22} and empirically tested the performance of our implementation on three publicly available datasets. We chose MeTeoR as its latest version supports materialisation for DatalogMTL with bounded intervals~\cite{DBLP:conf/aaai/WalegaZWG23}, which allows us to directly compare the performance of DRed$_\text{MTL}$ with rematerialisation from scratch. The source code of our implementation, the benchmarks we used, as well as an extended technical report containing all detailed proofs, are available online.

\begin{table}[b]
    \centering
    \renewcommand{\arraystretch}{1.5} 
    \begin{tabular}{c|c|c|c}
        ~ & $\text{LUBM}_t$ & iTemporal & Meteorological \\ \hline
        $|\Pi_{mtl}|$ & 29 & 3 & 2 \\ \hline
        $|\Pi_{no\_mtl}|$ & 56 & 8 & 2 \\ \hline
        $|\Pi|$ & 85 & 11 & 4 \\ \hline
        $Recursive$ & Yes & Yes & No \\ \hline
        $|E|$ & 630.5k & 46.41k & 62.01M \\ \hline
        $|I|$ & 1.426M & 30.62M & 62.48M \\
    \end{tabular}
    \caption{Dataset Statistics}
    \label{exp:Running times}
\end{table}

\begin{table*}[t]
\centering
\renewcommand{\arraystretch}{1.3}
{\small `Remat' stands for Rematerialisation from the updated set of explicitly given facts.}
\begin{tabular}{|c|c||c|c|c|c||c||c|c||c|}
\hline
\multirow{3}{*}{\centering Dataset} & \multirow{3}{*}{\centering $|E^{\pm}|$} & \multicolumn{5}{c||}{Deletion} & \multicolumn{3}{c|}{Insertion} \\
\cline{3-10}
 & & \multicolumn{4}{c||}{\textsc{DRed}$_{MTL}$} & Remat & \multicolumn{2}{c||}{\textsc{DRed}$_{MTL}$} & Remat \\ 
\cline{3-10}
 & & Time(s) & $|D|$ & $|R|$ & $|A|$ & Time(s) & Time(s) & $|A|$ & Time(s) \\ 
\hline
\multirow{2}{*}{$\text{LUBM}_t$} 
 & \centering 100 & 0.7k & 11.5k & 1.5k & 11.4k & 48.6k & 0.4k & 0.2k & 48.5k \\ 
 & \centering 63.1k & 3.4k & 372.1k & 143.5k & 251.9k & 45.0k & 1.1k & 190.2k & 48.2k \\ 
\hline
\multirow{2}{*}{iTemporal} 
 & \centering 100 & 8.1k & 244.5k & 274.0k & 274.4k & 52.7k & 8.7k & 231.2k & 52.8k \\ 
 & \centering 4.6k & 12.6k & 8.963M & 1.286M & 1.288M & 52.3k & 12.4k & 7.829M & 52.6k \\ 
\hline
\multirow{2}{*}{Meteorological} 
 & \centering 100 & 10 & 102 & 1 & 1 & 0.7k & 11 & 101 & 0.7k \\ 
 & \centering 6.201M & 1.4k & 6.271M & 1105 & 1105 & 0.7k & 1.2k & 6.270M & 0.7k \\ 
\hline
\end{tabular}
\caption{Evaluation Results}
\label{tab:combined_test}
\end{table*}

\begin{table*}[t]
\centering

\renewcommand{\arraystretch}{1.25}
\begin{tabular}{|c|c||>{\centering\arraybackslash}p{3.75cm}|>{\centering\arraybackslash}p{3.75cm}|>{\centering\arraybackslash}p{3.75cm}|}
\hline
Dataset & $|E^{-}|$ &
Overdeletion & Rederivation & Insertion \\
\hline

\multirow{2}{*}{$\text{LUBM}_t$}
& 100    & 59.0\% & 2.4\%  & 38.6\% \\ 
& 63.1k  & 15.7\% & 67.9\% & 16.4\% \\
\hline

\multirow{2}{*}{iTemporal}
& 100    & 70.5\% & 10.6\% & 18.9\% \\ 
& 4.6k   & 64.2\% & 29.9\% & 5.9\%  \\
\hline

\multirow{2}{*}{Meteorological}
& 100     & 62.8\% & 21.8\% & 15.4\% \\ 
& 6.201M & 92.2\% & 7.7\%  & 0.1\%  \\
\hline

\end{tabular}
\caption{Deletion Test Runtime Breakdown}
\label{tab:breakdown}
\end{table*}

\subsubsection{Benchmarks}
$\text{LUBM}_t$ ~\cite{DBLP:conf/aaai/WangHWG22} is a temporal version of the well-known LUBM benchmark~\cite{DBLP:journals/ws/GuoPH05LUBM}. It has a recursive program consisting of 85 rules, of which 29 have temporal operators in them (denoted as $|\Pi_{mtl}|$) and 56 do not (denoted as $|\Pi_{no\_mtl}|$). The iTemporal dataset is generated from a temporal benchmark generator developed by \citet{DBLP:conf/icde/BellomariniNS22}. Its program is highly recursive and consists of 12 rules. Finally, the Meteorological dataset~\cite{Weather} contains long-term meteorological observations. It has a nonrecursive program consisting of four rules, of which two contain temporal operators. These datasets were collected and made publicly available by~\citet{wang2025goal}, and we used them without any modification. The statistics of these datasets is given in Table \ref{exp:Running times}, where $|E|$ and $|I|$ are the number of explicitly given facts and the number of facts in the saturated interpretation, respectively. The ratio between $|I|$ and $|E|$ is usually a good indicator of the recursiveness of the corresponding rule set: the larger the ratio is, the more recursive and complex the rule set is and the more likely it is to generate a greater number of facts. Indeed, our choice of benchmarks is a nice mixture of highly recursive (iTemporal), mildly recursive (LUBM$_t$) and nonrecursive datasets, which allows us to test the potential of our reasoning algorithm in different possible application scenarios.



\subsubsection{Test Settings}
We conducted our experiment on a server with 256GB RAM and an Intel Xeon Platinum 8269CL 2.50GHz CPU, running Fedora Linux 40, kernel version Linux 6.10.10-200.fc40.x86\_64. Our evaluation primarily examines the capabilities of DRed$_\text{MTL}$ in handling both deletions and insertions. Table \ref{tab:combined_test} summarises the performance comparison of our algorithm against rematerialisation. For our approach, we record the wall-clock time of running Algorithm \ref{algo:DREDMTL} to handle the updates. For the baseline, we record the wall-clock time that MeTeoR spends on computing the canonical representation over the updated set of explicitly given facts from scratch. In addition to the time metrics, to better reflect the workload of DRed$_\text{MTL}$, we also record the number of facts derived in the three stages of our algorithm, i.e.,  overdeletion ($|D|$), rederivation ($|R|$) and insertion ($|A|$). Note that for insertion, the sets of overdeleted and rederived facts are always empty, so we only record the number of inserted facts, $|A|$.



Our evaluation considers both small-scale and large-scale updates. For small-scale update tests, we randomly selected 100 facts and ran DRed$_\text{MTL}$ to deal with the deletion; then we added these facts back and recorded the time DRed$_\text{MTL}$ took to handle the insertion. Large-scale tests were performed in a similar fashion, except that 10\% of the original dataset were removed and added back; the exact numbers of deleted facts are shown in Table 3. For each test case, three test runs with different (randomly selected) updates were performed; the results showed no significant variation in terms of running time. Therefore, for the ease of presentation we only reported the results of one run for each test case. Finally, for each test run, we made sure that the periodic materialisation produced by DRed$_\text{MTL}$ ($\mathds{I}_1$) is equivalent to that produced by rematerialisation ($\mathds{I}_2$): this is achieved by verifying that both ${\mathds{I}_1 \dblsetminus \mathds{I}_2}$ and ${\mathds{I}_2 \dblsetminus \mathds{I}_1}$ are empty.

\subsubsection{Results}

Our evaluation shows that DRed$_\text{MTL}$ consistently outperforms rematerialisation for all small deletions and insertions. As shown in Table \ref{tab:combined_test}, on $\text{LUBM}_t$, DRed$_\text{MTL}$ is 69.4 times faster than the baseline for small deletion, and 121.3 times faster for small insertion. On the nonrecursive Meteorological dataset, DRed$_\text{MTL}$ achieves similar performance improvement for small updates. On iTemporal, the performance improvement is more modest: the speedup is around six times. This is so since the program of iTemporal is highly recursive, making incremental maintenance especially challenging: as one can see, deleting only 100 facts leads to the overdeletion of over 244 thousand facts.


For large updates, DRed$_\text{MTL}$ achieved a significant performance boost on the $\text{LUBM}_t$ and iTemporal datasets. When deleting 10\% of the data, DRed$_\text{MTL}$ is 13.2 times faster on $\text{LUBM}_t$ and 4.2 times faster on iTemporal. For large insertion, the improvements are 43.8 and 4.2 times, respectively. Interestingly, on the Meteorological dataset, DRed$_\text{MTL}$ is slower than rematerialisation in dealing with large updates. Notice that for this dataset, the ratio between $|I|$ and $|E|$ is rather small, as depicted in Table~\ref{exp:Running times}, indicating that only a small fraction of the materialisation are inferred facts; moreover, the program is nonrecursive, so no effort is required for identifying the repeating pattern inside the materialisation. As such, for large updates on this dataset, DRed$_\text{MTL}$, which is based on efficient seminaive reasoning and period identification, losts its advantage.

To further analyse the performance of the proposed algorithm, we profiled our system on the deletion test cases and reported the runtime breakdown (by stage of reasoning) in Table~\ref{tab:breakdown}. It could be readily observed that in most cases, overdeletion is the most time consuming step of the algorithm. This is so since overdeletion produces the largest number of facts across the three stages. The only exception is the case of large deletion for LUBM$_t$: rederivation consumes more time than overdeletion and insertion combined. Indeed, although rederivation produces less facts than overdeletion, it involves evaluating rules `backwards' from head to body, which, as observed by~\citet{HuMH18}, could be more expensive than the semina\"ive evaluation taking place in overdeletion and insertion. In the Datalog setting, enhancing rederivation with counting could help alleviate this issue, but extending the counting technique to DatalogMTL is beyond the scope of this paper.


Overall our test results suggest that the proposed algorithm improves substantially and consistently over rematerialisation, for both small and large updates. In some test cases, the running time decreased from several hours to several minutes, or from an hour to a few seconds, demonstrating the potential of deploying the proposed approach in industrial applications where short service response time is highly desirable.

\section{Conclusion and Future Work}
In this paper, we have introduced a new technique for incrementally updating DatalogMTL materialisations. Compared with recomputing the materialisation from scratch, our technique achieves significant performance improvements, especially when the updates are small. 

We see many exciting future research directions. From a practical perspective, it would be interesting to see how well the proposed approach works in industrial scenarios such as IoT anomaly detection~\cite{zhangIoT}: in typical such applications, DatalogMTL can be used to model anomaly detection rules that are triggered by streams composed of timestamped sensor data; the ability to reason incrementally offered by our method is crucial for ensuring efficient and reliable fault alerts. Moreover, we shall consider comparing the performance of our system with well-established stream reasoning frameworks such as C-SPARQL~\cite{www09csparql} and CQELS~\cite{PhuocDPH11}. While the underlying languages are quite different, it should be possible to transform our reasoning workload (at least the nonrecursive cases) to a C-SPARQL or CQELS workload, and test it on the corresponding engine. Last but not least, DRed is not the only incremental update algorithm for Datalog materialisations; algorithms such as B/F, FBF, and DRed$^c$ are popular alternatives of DRed and sometimes offer superior performance. It would thus be useful to consider extending these incremental maintenance algorithms for Datalog to support DatalogMTL reasoning, and to compare the performance of the adapted algorithms with that of ours.

\section{Acknowledgements}

This work was generously funded by National Science and Technology Major Project of China under grant number 2025ZD1600800 and National Natural Science Foundation of China under grant number 62206169. We thank the anonymous reviewers for their constructive comments that helped greatly in shaping the final version of this paper.

\bibliography{aaai2026}
\onecolumn
\appendix
\begin{appendices}
\section{Proof of Correctness for Algorithms~\ref{algo:periodic_minus} 
and~\ref{algo:periodic_union}}

In this section, we prove that when applied to arbitrary periodic materialisations $\mathds{I}_1$ and $\mathds{I}_2$, Algorithms~\ref{algo:periodic_minus} 
and~\ref{algo:periodic_union} correctly compute $\mathds{I}_1 \dblsetminus \mathds{I}_2 $ and $\mathds{I}_1 \Cup\mathds{I}_2$, respectively.

\theoremoperatorcorrectness* 

We begin with a simple scenario where $\mathds{I}_1$ and $\mathds{I}_2$ share the same period intervals. We prove that the operators $\dblsetminus$ and $\Cup$ can be implemented using standard set difference
($\setminus$) and set union ($\cup$), respectively, as stated in Claim~\ref{claim:alignedMinus} and Claim~\ref{claim:alignedUnion}. Subsequently, in Claim~\ref{claim:minuscorrect} and Claim~\ref{claim:unioncorrect}, we prove that our algorithm ensures that $\mathds{I}_1$ and $\mathds{I}_2$ are aligned in terms of their period intervals before applying set difference and set union.

\begin{claim}\label{claim:alignedMinus}
    For periodic materialisation $\mathds{I}_1=\langle I_1,\varrho_{\mathsf
L},\varrho_{\mathsf
R}\rangle$ and $\mathds{I}_2=\langle I_2,\varrho_{\mathsf L},\varrho_{\mathsf R}\rangle$, such that $\mathsf{unfold}(\mathds{I}_2)\subseteq \mathsf{unfold}(\mathds{I}_1)$. Then the following holds:
    \begin{align*}
        \mathsf{unfold}(\mathds{I}_1) -\mathsf{unfold}(\mathds{I}_2) =\mathsf{unfold}(\langle I_1\setminus I_2,\varrho_{\mathsf L},\varrho_{\mathsf R}\rangle )
    \end{align*}
\end{claim}
\begin{proof}
    We prove the claim by classifying all possible sources of facts in $\mathsf{unfold}(\mathds{I}_1)$ and checking that both sides of the equation contain or exclude them consistently.

    Let $f$ be a fact. We analyse the cases where $f$ may appear in left-hand side and how it relates to the right-hand side.
    
    Case 1: $f\in I_2$.
    
    Then, $f$ will appear in $\mathsf{unfold}(\mathds{I}_2)$ via unfolding. Since $\mathsf{unfold}(\mathds{I}_2)\subseteq \mathsf{unfold}(\mathds{I}_1)$, the fact $f$ is removed by the minus on the left-hand side. On the right-hand side, since $f\in I_2$, we have $f\not\in I_1\setminus I_2$, and hence $f$ is not included in $\mathsf{unfold}(\langle I_1\setminus I_2,\varrho_{\mathsf L},\varrho_{\mathsf R}\rangle )$. Therefore,  both sides exclude $f$.

    Case 2: $f\in I_1$ and $f \not\in I_2$ 

    Then, $f\in I_1 \setminus I_2$, and hence $f$ is included in the right-hand side unfolding. On the left-hand side, $f\in \mathsf{unflod}(\mathds{I}_1)$ and $f\not\in\mathsf{unfold}(\mathds{I}_2)$, so it remains after the set difference. Thus, both sides include $f$.
    
    Case 3: $f\in \mathsf{unfold}(\mathds{I}_2)$ and $f \not\in I_2$.

    This means $f$ appears in $\mathsf{unfold}(\mathds{I}_2)$ due to periodic repetition from some $f'\in I_2$, i.e., $f$ is derived from base fact $f' \in I_2$ by unfolding over $\varrho_{\mathsf L},\varrho_{\mathsf R}$. Then $f \in \mathsf{unfold}(\mathds{I}_2)\subseteq\mathsf{unfold}(\mathds{I}_1)$, so it will be removed on the left-hand side. On the right-hand side, the corresponding base fact $f'\in I_2$, so $f'\not\in I_1\setminus I_2$, and $f$ is not generated in the unfolding of the right-hand side. Hence, both sides exclude $f$.

    Case 4: $f\in \mathsf{unfold}(\mathds{I}_1)$, $f\not\in I_1$ and $f \not\in \mathsf{unfold}(\mathds{I}_2)$

    This means $f$ comes from unfolding some base fact $f'\in I_1\setminus I_2$. Then $f$ appears in $\mathsf{unfold}(\mathds{I}_1)$ and is not removed in the left-hand side. Since $f'\in I_1 \setminus I_2$, $f$ also appears in the unfolding of the right-hand side. Thus, both sides include $f$.

    Since in all cases the membership of $f$ in both sides coincides, the two sets are equal.
\end{proof}
\begin{claim}\label{claim:alignedUnion}
      For periodic materialisation $\mathds{I}_1=\langle I_1,\varrho_{\mathsf L},\varrho_{\mathsf R}\rangle$ and $\mathds{I}_2=\langle I_2,\varrho^l,\varrho^r\rangle$. Then the following holds:
      \begin{align*}
         \mathsf{unfold}(\mathds{I}_1) \cup\mathsf{unfold}(\mathds{I}_2) =\mathsf{unfold}(\langle I_1\cup I_2,\varrho_{\mathsf L},\varrho_{\mathsf R}\rangle )
    \end{align*}
\end{claim}

\begin{proof}
    We prove the claim by classifying all possible sources of facts in $\mathsf{unfold}(\mathds{I}_1)$ and verifying that both sides of the equation include them consistently. Due to the commutativity of union, it suffices to reason about facts in $\mathds{I}_1$.

    Let $f$ be a fact. We analyse the following cases:

    Case 1: $f \in I_1$.

    Then $f$ appears in $\mathsf{unfold}(\mathds{I}_1)$ and thus in the left-hand side. Since $f \in I_1 \subseteq I_1 \cup I_2$, $f$ is also in $\mathsf{unfold}(\langle I_1 \cup I_2, \varrho_{\mathsf L},\varrho_{\mathsf R} \rangle)$. Hence, both sides include $f$.

    Case 2: $f \in \mathsf{unfold}(\mathds{I}_1)$ and $f \notin I_1$.

    Then $f$ is derived from some $f' \in I_1$ via unfolding over the periodic intervals. Since $f' \in I_1 \subseteq I_1 \cup I_2$, this same unfolding also exists in $\mathsf{unfold}(\langle I_1 \cup I_2, \varrho_{\mathsf L},\varrho_{\mathsf R} \rangle)$, so $f$ appears in both sides.

    Since all facts from $\mathsf{unfold}(\mathds{I}_1)$ are preserved in the right-hand side, and the same argument holds symmetrically for $\mathds{I}_2$, we conclude the two sides are equal.
\end{proof}

\begin{claim}\label{claim:minuscorrect}
    For a periodic materialisation $\mathds{I}_1$, and for each periodic materialisation  $\mathds{I}_2$ such that $\mathsf{unfold}(\mathds{I}_2) \subseteq \mathsf{unfold}(\mathds{I}_1)$, if we run Algorithm~\ref{algo:periodic_minus} on them and yield $\mathds{I}_3$, then
    \begin{align*}
        \mathsf{unfold}(\mathds{I}_1) -\mathsf{unfold}(\mathds{I}_2) = \mathsf{unfold}(\mathds{I}_3)
    \end{align*}
\end{claim}
\begin{proof}
    Based on Claim~\ref{claim:alignedMinus}, it suffices to prove that just before line~9 in algorithm~\ref{algo:periodic_minus}, we have $\varrho^{I_1}_\mathsf{L} = \varrho^{I_2}_\mathsf{L} = \varrho^{I_3}_\mathsf{L}$, and similarly for the right side.
    
    We first analyse the situation on the left side, considering the following cases:
    
    Case 1: $\varrho^{I_2}_\mathsf{L}$ exists.

    Since $\mathsf{unfold}(\mathds{I}_2) \subseteq \mathsf{unfold}(\mathds{I}_1)$, the left period $\varrho^{I_1}_\mathsf{L}$ must also exist. The procedure $\mathbb{E}$ ensures that $\varrho^{I_1}_\mathsf{L}$ and $\varrho^{I_2}_\mathsf{L}$ are of equal length. Then, $\mathbb{A}$ aligns the two intervals and sets the left period of the result, $\varrho^{I_3}_\mathsf{L}$, accordingly. Hence, we have $\varrho^{I_1}_\mathsf{L} = \varrho^{I_2}_\mathsf{L} = \varrho^{I_3}_\mathsf{L}$.

    Case 2: $\varrho^{I_2}_\mathsf{L}$ doesn't exists.

    Regardless of whether $\varrho^{I_1}_\mathsf{L}$ exists, the facts in $\mathsf{unfold}(\mathds{I}_2)$ have no influence on those unfolded by $\varrho^{I_1}_\mathsf{L}$. Conceptually, we may treat $\varrho^{I_1}_\mathsf{L}$ as a virtual interval equal to $\varrho^{I_1}_\mathsf{L}$, but which contributes no facts during unfolding. In this way, we can again conclude that $\varrho^{I_1}_\mathsf{L} = \varrho^{I_2}_\mathsf{L} = \varrho^{I_3}_\mathsf{L}$.

    The case on the right is symmetric to the left and can be proven similarly.
\end{proof}

\begin{claim}\label{claim:unioncorrect}
    For a periodic materialisation $\mathds{I}_1$, and for each periodic materialisation  $\mathds{I}_2$, if we run Algorithm~\ref{algo:periodic_union} on them and yield $\mathds{I}_3$, then
    \begin{align*}
        \mathsf{unfold}(\mathds{I}_1) \cup\mathsf{unfold}(\mathds{I}_2) = \mathsf{unfold}(\mathds{I}_3)
    \end{align*}
\end{claim}

\begin{proof}
    Based on Claim~\ref{claim:alignedUnion}, it suffices to prove that just before line~10 in algorithm~\ref{algo:periodic_union}, we have $\varrho^{I_1}_\mathsf{L} = \varrho^{I_2}_\mathsf{L} = \varrho^{I_3}_\mathsf{L}$, and similarly for the right side.

    We first analyse the situation on the left side, considering the following cases:

    Case 1: Both $\varrho^{I_1}_\mathsf{L}$ and $\varrho^{I_2}_\mathsf{L}$ exist.

    The algorithm first calls $\mathbb{E}$ to equalise the length of the two periodic regions, ensuring they can be aligned. Then, it applies $\mathbb{A}$ to compute a shared aligned interval $\varrho^{I_3}_\mathsf{L}$ for the result, and adjusts both datasets to match this alignment. As a result, we obtain $\varrho^{I_1}_\mathsf{L} = \varrho^{I_2}_\mathsf{L} = \varrho^{I_3}_\mathsf{L}$.

    Case 2: Only one of $\varrho^{I_1}_\mathsf{L}$ or $\varrho^{I_2}_\mathsf{L}$ exists.

    Without loss of generality, suppose only $\varrho^{I_2}_\mathsf{L}$ exists. Then the algorithm applies $\mathbb{A}$, treating the non-periodic dataset as if it had an empty periodic region, and aligns both datasets accordingly. Consequently, we again obtain $\varrho^{I_1}_\mathsf{L} = \varrho^{I_2}_\mathsf{L} = \varrho^{I_3}_\mathsf{L}$.

    Case 3: Neither $\varrho^{I_1}_\mathsf{L}$ nor $\varrho^{I_2}_\mathsf{L}$ exists.

    Here, both datasets are non-periodic. The algorithm simply sets $\varrho^{I_3}_\mathsf{L} = \emptyset$, so trivially $\varrho^{I_1}_\mathsf{L} = \varrho^{I_2}_\mathsf{L} = \varrho^{I_3}_\mathsf{L} = \emptyset$.

    The case on the right is symmetric to the left and can be proven similarly. This completes our proof for Claim~\ref{claim:unioncorrect}, as well as the proof for Theorem~\ref{theorem:periodicoperatorscorrect}.
\end{proof}

\section{Proof of Correctness for Algorithm~\ref{algo:DREDMTL}}

\subsection{Proof Overview}
The proof of correctness for Algorithm~\ref{algo:DREDMTL} is lengthy, so we organise it into subsections as follows.

In the Notations and Common Properties sub-section, we define and recall key notations that will be used throuout the proof. Additionally, we prove several properties that are common to the (over)deletion, rederivation, and insertion stages of Algorithm~\ref{algo:DREDMTL} as these properties will facilitate our discussion later.

In the Properties about Deletion sub-section, we develop a theory for the \textsc{Overdelete} procedure.
We abstract one iteration of the loop in the \textsc{Overdelete} procedure into an operator $D_{\Pi,E}$, and characterise the value of $D$ after each iteration of the loop by a transfinite sequence of interpretations. 
We show that this sequence converges to the pre-defined goal interpretation by Proposition~\ref{proposition:overdeletionleast}, that the limit of this sequence contains all the facts we hope to delete by Proposition~\ref{proposition:overdeletioncomplete}, that its limit can be finitely represented by Theorem~\ref{theo:deletionunfold}, and that the finite representation will be found in finite time by \textsc{Periods} by Proposition~\ref{proposition:deletion-saturated-align}. 
Moreover, we note that the operator which conducts semantic equivalence of the operations in one loop iteration in \textsc{Overdelete}, the semi-na\"ive version of $D_{\Pi,E}$, defines the same sequence of interpretations as $D_{\Pi,E}$, which is stated in Proposition~\ref{proposition:deletionseminaive=naive}.

In the Properties about Rederivation and Properties about Insertion sub-sections, we formalise properties that are similar to those presented in the sub-section for deletion, except that we show that rederivation rederives all mistakenly deleted facts~(Proposition~\ref{proposition:rederivationcorrect}) and that insertion adds all new facts to our result (Proposition~\ref{proposition:insertioncorrect}).

The last sub-section wraps things up and makes use of the properties proven thus far to finally establish the correctness of Algorithm~\ref{algo:DREDMTL}.

\subsection{Notations and Common Properties}
Recall that throughout the paper, we assume all programs and datasets are bounded. Moreover, an interval is punctual if it is of the form $[t, t]$, for some $t \in \mathbb{Q}$; 
we can abbreviate it as $t$. 
A DatalogMTL fact is punctual if its period is punctual.

For a bounded program $\Pi$ and a bounded dataset $\mathcal{D}$, a $(\Pi,\mathcal{D})$-interval is either a punctual interval over a time point on the $(\Pi,\mathcal{D})$-ruler, or an interval $(t_1,t_2)$ with $t_1$ and $t_2$ consecutive time points on the $(\Pi,\mathcal{D})$-ruler.
We say that an interpretation $\mathfrak{I}$ is a $(\Pi,\mathcal{D})$-interpretation if, for each fact $M@t$, $\mathfrak{I} \models M@t$ implies $\mathfrak{I}\models M@\varrho$,
where $\varrho$ is the $\Pi,\mathcal{D}$-interval containing $t$. 

The semina\"ive operator ${\Pi \langle I \myvdots \Delta \rangle}$ where $\Delta \subseteq I$ used in lines~\ref{delete-seminaive-operator}, \ref{rederive-seminaive-operator}, and \ref{insert-seminaive-operator} of Algorithm~\ref{algo:DREDMTL} are defined as the union of all facts derived by $(r\sigma, t) \langle I \myvdots \Delta \rangle$, where $r \sigma$ is a ground rule instance with $r$ a rule in $\Pi$ and $\sigma$ a substitution mapping each variable appearing in $r$ to a constant appearing in $I$, and $t$ is a rational time point. Operator $(r', t) \langle I \myvdots \Delta \rangle$ where $r'$ is a ground rule instance and $t$ is a rational time point is in turn defined as the minimal set of punctual facts $N$ such that $\mathfrak{I}_{I}, t \models body(r')$, $\mathfrak{I}_{I \setminus \Delta}, t \not \models body(r')$, and $\mathfrak{I}_{N}, t \models head(r')$.

\begin{proposition}\label{proposition:headexpandsion}
    For each ground instance $r'$ of each rule $r\in\Pi$, where $\Pi$ is bounded, and each time point $t$, let $head(r)$ contain relational atom $M$, and $head(r')$ contain $M\sigma$ for some substitution $\sigma$, and the least interpretation satisfying $head(r')$ at $t$ be $\mathfrak{I}$, then there exists some $\varrho$ such that $\varrho^- = t+t_1$ and $\varrho^+ = t+t_2$ for some $t_1,t_2$ that are multiples of $div(\Pi)$, for which the following holds for each $t'\in\mathbb{Q}$: 
    $$\mathfrak{I},t' \models M \iff t'\in\varrho.$$
\end{proposition}
\begin{proof}
    We show inductively on the nesting structure of $head(r')$ that the proposition is true.

    When $head(r) = M$, in this case the least interpretation that satisfies $head(r')$ at $t$ is $\mathfrak{I}_{\{M\sigma@t\}}$, and it is obvious that $\varrho = t$ satisfies the proposition.

    When $head(r) = \boxplus_{\langle t_1',t_2'\rangle} M'$, then by the induction hypothesis, the least interpretation $\mathfrak{I}'$ that satisfies $M'$ at each $t'$, satisfies that for some $\varrho'$ such that  $\varrho^- = t+t_1$ and $\varrho^+ = t+t_2$ for some $t_1,t_2$ that are multiples of $div(\Pi)$, $$\mathfrak{I}',t'' \models M \iff t''\in\varrho'$$

    Let the least interpretation that satisfies $head(r')$ at $t$ be $\mathfrak{I}$, then $\mathfrak{I}$ is the least interpretation satisfying $M'$ at $t'$ during $\langle t+t_1', t+t_2'\rangle$.
    Therefore, $\forall t' \in \langle t+t_1', t+t_2'\rangle$, $\mathfrak{I},t'\models M'$, and $\forall t''\in \langle t'+t_1,t'+t_2\rangle$, $\mathfrak{I},t''\models M$, because $\mathfrak{I},t'\models M'$.
    This in turn means that $\mathfrak{I}\models M@\langle t+t_1+t_1',t+t_2+t_2'\rangle$.
    On the other hand, for an interpretation $\mathfrak{I}''$ such that 
    $$\mathfrak{I}'',t'\models M \iff t'\in \langle t+t_1+t_1',t+t_2+t_2'\rangle$$, because it satisfies $M$ at $\langle t'+t_1,t'+t_2\rangle$ for each $t'\in \langle t+t_1,t+t_2\rangle$, it satisfies $M'$ for each $t'\in \langle t+t_1,t+t_2\rangle$, and $\mathfrak{I}'',t\models head(r')$.

    Therefore, because the least interpretation that satisfies $head(r')$ at $t$ contains $\mathfrak{I''}$, $\mathfrak{I}''$ is the the least interpretation that satisfies $head(r')$. 
    By definition, $t_1'$ and $t_2'$ are multiples of $div(\Pi)$, and therefore $t_1+t_1'$ and $t_2+t_2'$ are multiples of $div(\Pi)$.

    The case for $head(r) = \boxminus_{\langle t_1', t_2'\rangle}$ is analogous.
\end{proof}
\begin{proposition}\label{proposition:timepointsdiv(pi)apartcontainedbyintervalsofequallength}
    For two time points $t_1,t_2$ such that $t_1\equiv t_2\mod div(\Pi)$, let $\varrho_1$ and $\varrho_2$ be the two $(\Pi,\mathcal{D})$-intervals containing $t_1,t_2$, then $|\varrho_1| = |\varrho_2|$.
\end{proposition}
\begin{proof}
    If $t_1$ is on the $(\Pi,\mathcal{D})$-ruler, then $t_2$ is on the $(\Pi,\mathcal{D})$-ruler, and $|\varrho_1| = |\varrho_2| = 0$.
    Let $\varrho_1 = \langle t_l,t_r\rangle$, $\varrho_2 = \langle t_l',t_r'\rangle$ and $t_1-t_2 = t$.
    If $t_l-t < t_l'$, then $t_l < t_l'+t<t_1<t_r$, which contradicts $t_l,t_r$ being consecutive time points on the $(\Pi,\mathcal{D})$-ruler, and therefore $t_l-t \geq t_l'$
    Analogously, we can show $t_l-t \leq t_l'$, $t_r-t \leq t_r'$ and $t_r-t \geq t_r'$, which means $t_l- t =t_l'$, $t_r -t = t_r'$ and $ |\varrho_1| =t_r-t_l = t_r'-t_l' = |\varrho_2|$.
\end{proof}
\begin{proposition}\label{proposition:pi,dinterpretationatomsat}
    Let $M$ be a body atom of a ground rule instance $r'$ of a rule $r\in\Pi$ and $\Pi,\mathcal{D}$ be bounded, then for a $(\Pi,\mathcal{D})$-interpretation $\mathfrak{I}$ and a time point $t$, $\mathfrak{I},t\models M$ implies $\forall t'\in\varrho,\mathfrak{I},t'\models M$, where $\varrho$ is the $(\Pi,\mathcal{D})$-interval containing $t$.
\end{proposition}
\begin{proof}
    We only need to consider $t$ that are not on the $(\Pi,\mathcal{D})$-ruler.
    Let $t\in(t_l,t_r)$, where $(t_l,t_r)$ is the $(\Pi,\mathcal{D})$-interval containing $t$.
    We only consider non-negative endpoints in the intervals associated with the operators.
    
    We prove this by structural induction on $M$.

    In the base case, when $M$ is a ground relational atom, the statement is trivially true by definition of $(\Pi,\mathcal{D})$-interpretations.

    If $M = \boxminus_{\langle t_1,t_2\rangle} M_1$, where $\langle\in\{(,[\}$ and $\rangle\in\{),]\}$, then by definition of $(\Pi,\mathcal{D})$-interpretations, $t_1$ and $t_2$ are multiples of $div(\Pi)$.
    If $\mathfrak{I},t\models M$, then $\mathfrak{I},t' \models M_1$ for each $t'\in\langle t-t_2,t-t_1\rangle$.
    Because $t$ is not on the $(\Pi,\mathcal{D})$-ruler, $t-t_2$ and $t-t_1$ are not on the $(\Pi,\mathcal{D})$-ruler either.
    Let $t-t_1\in(t_l',t_r')$ and $t-t_2\in(t_l'',t_r'')$, where both $(t_l',t_r')$ and $t-t_2\in(t_l'',t_r'')$ are $(\Pi,\mathcal{D})$-intervals. 
    Then because for each $t'\in(t_l',t-t_1)\cup(t-t_2,t_r'')$, $\mathfrak{I},t\models M_1$, by the induction hypothesis we have for each $t'\in(t_l',t_r')\cup(t_l'',t_r'')$, $\mathfrak{I},t'\models M_1$, and therefore, $\forall t'\in (t_l'',t_r'),\mathfrak{I},t'\models M_1$.
    If $t_l'> t_l-t_1$, then $t_l'+t_1 >t_l$ is on the $(\Pi,\mathcal{D})$-ruler. 
    Moreover, since $t-t_1>t_l'$ and $t_r>t$, $t_r>t>t_l'+t_1>t_l$, $t_l,t_r$ are not consecutive time points on the $(\Pi,\mathcal{D})$-ruler, thereby contradicting that $(t_l,t_r)$ is a $(\Pi,\mathcal{D})$-interval.
    Therefore, $t_l'\leq t_l-t_1$.
    If $t_l'<t_l-t_1$, then $t_l-t_1$ is on the $(\Pi,\mathcal{D})$-ruler.
    Moreover, since $t-t_1<t_r'$ and $t_l<t$, $t_l'<t_l-t_1<t-t_1<t_r'$
    which means that $t_l'$ and $t_r'$ are not consecutive time points on the $(\Pi,\mathcal{D})$-ruler, thereby contradicting that $(t_l,t_r)$ is a $(\Pi,\mathcal{D})$-interval.
    Therefore, $t_l'\geq t_l-t_1$, and $t_l' = t_l-t_1$.
    We can prove $t_r'=t_r-t_1$, $t_l'' = t_l-t_2$, $t_r'' = t_r-t_2$ analogously. 
    For each $t'\in(t_l,t_r)$, $t'-t_1 < t_r-t_1 = t_r'$ $t'-t_2 > t_l -t_2 = t_l''$ and $\langle t'-t_1,t'-t_2\rangle \subseteq (t_l'',t_r')$, whereby $\forall t'' \in \langle t'-t_1,t'-t_2\rangle$ $\mathfrak{I},t''\models M_1$, and therefore $\mathfrak{I},t'\models M$
    
    The case for $M = \boxminus_{\langle t_1,t_2\rangle} M_1$ is analogous.

    If $M = \diamondminus_{\langle t_1,t_2\rangle} M_1$, where $\langle\in\{(,[\}$ and $\rangle\in\{),]\}$, then by definition of $(\Pi,\mathcal{D})$-interpretations, $t_1$ and $t_2$ are multiples of $div(\Pi)$. 
    If $\mathfrak{I},t\models M$, then there is some $t'\in\langle t-t_1,t-t_2\rangle$, such that $\mathfrak{I},t'\models M_1$. 
    Let $(t_l',t_r')$ be the $(\Pi,\mathcal{D})$-interval containing $t'$. 
    By the induction hypothesis, $\forall t'' \in (t_l',t_r')$, $\mathfrak{I},t''\models M_1$.
    If $t_1=t_2$, then $\langle t_1,t_2\rangle = [t_1,t_1]$ and $t' = t-t_1$. 
    We can show in a way similar to that in the $\boxminus$ case that $t_l' = t_l -t_1$, $t_r' - t$, and because $\forall t'' \in (t_l',t_r')$, $\mathfrak{I},t''\models M_1$, we have $\forall t'\in(t_l,t_r)$, $\mathfrak{I},t'\models M$.
    If $t_1\neq t_2$, then $t_2-t_1\geq div(\Pi)$, because  $\forall t''\in (t_l',t_r')$, $\mathfrak{I},t''\models M_1$, and $\mathfrak{I},t''\models M_1$ implies $\forall t'\in \langle t''+t_1,t''+t_2\rangle,\mathfrak{I},t'\models M$, we have $\forall t'\in (t_l'+t_1,t_r'+t_2),\mathfrak{I},t'\models M$.
    Because $t_l'+t_1<t'+t_1<t<t'+t_2<t_r'+t_2$, we have $t\in (t_l'+t_1,t_r'+t_2)$.
    If $(t_l,t_r)\not\subseteq (t_l'+t_1,t_r'+t_2)$, then either $t_l<t_l'+t_1$ or $t_r>t_r'+t_2$, but either case contradicts $t_l,t_r$ being consecutive time points on the $(\Pi,\mathcal{D})$-ruler, since $t_l'+t_1$ and $t_r'+t_2$ are also time points on the $(\Pi,\mathcal{D})$-ruler.
    Therefore, $(t_l,t_r)\subseteq (t_l'+t_1,t_r'+t_2)$, and $\forall t'\in (t_l,t_r)$, $\mathfrak{I},t'\models M$.
    
    The case for $M = \diamondplus_{\langle t_1,t_2\rangle}$ is analogous.

    If $M = M_2 \mathcal{S}_{\langle t_1,t_2\rangle} M_1$,  then by definition of $(\Pi,\mathcal{D})$-interpretations, $t_1$ and $t_2$ are multiples of $div(\Pi)$.
    If $\mathfrak{I},t\models M$, then there must be some $t' \in \langle t-t_1,t-t_2\rangle$, such that $\mathfrak{I},t'\models M_1$, and $\forall t'' \in (t',t),\mathfrak{I},t''\models M_2$.
    Let the $(\Pi,\mathcal{D})$-interval containing $t'$ be $(t_l',t_r')$, and by the induction hypothesis $\forall t'' \in (t_l',t_r'), \mathfrak{I},t''\models M_1$, and $\forall t'' \in (t_l',t_r),\mathfrak{I},t''\models M_2$.
    If $t_1=t_2$, then $\langle t_1,t_2\rangle = [t_1,t_1]$.
    Again we can infer that $t_l-t_1 = t_l'$ and $t_r-t_1 =t_r'$, and for each $t''\in(t_l,t_r)$, we know that $\mathfrak{I},t'' - t_1 \models M_1$, because $t''-t_1\in (t_l',t_r')$, and $\forall t'''\in (t''-t_1,t''),\mathfrak{I},t'''\models M_2$, because $(t''-t_1,t'')\subseteq (t_l',t_r)$, and therefore, $\mathfrak{I},t''\models M$.
    If $t_1\neq t_2$, then $t_2-t_1\geq div(\Pi)$, and because  $\forall t''\in (t_l',t_r')$, $\mathfrak{I},t''\models M_1$ , and $\mathfrak{I},t''\models M_1$ implies $\forall t'''\in \langle t''+t_1,t''+t_2\rangle$, $t''$ is a time point in $\langle t'''-t_1,t'''-t_2\rangle$, such that $\mathfrak{I},t''\models M_1$, and because $\forall t'' \in (t_l',t_r),\mathfrak{I},t''\models M_2$ , we have $\forall t_4'\in (t_l'+t_1,t_r'+t_2)\cap (t_l',t_r),\mathfrak{I},t_4'\models M$.
    By the same reasoning from the $\diamondminus$ case, we have $(t_l,t_r)\subseteq (t_l'+t_1,t_r'+t_2)$, because $(t_l,t_r)\subseteq (t_l',t_r)$, we have $\forall t_4'\in (t_l,t_r),\mathfrak{I},t_4'\models M$.

    The case for $M = M_2 \mathcal{U}_{\langle t_1,t_2\rangle} M_1$ is analogous.
    
\end{proof}

\begin{proposition}\label{proposition:pi,dinterpretationatomnotsat}
    Let $M$ be a body atom of a ground rule instance $r'$ of a rule $r\in\Pi$, and $\Pi,\mathcal{D}$ be bounded, then for a $(\Pi,\mathcal{D})$-interpretation $\mathfrak{I}$ and a time point $t$, $\mathfrak{I},t\not\models M$ implies $\forall t'\in \varrho$, $\mathfrak{I},t'\not\models M$, where $\varrho$ is the $\Pi,\mathcal{D}$-interval containing $t$.
\end{proposition}
\begin{proof}
    If $\exists t'\in \varrho$, $\mathfrak{I},t'\models M$, then by Proposition~\ref{proposition:pi,dinterpretationatomsat} $\forall t'\in \varrho$, $\mathfrak{I},t'\models M$, which contradicts $\mathfrak{I},t\not\models M$.
\end{proof}
\begin{proposition}\label{proposition:punctualfactsareenoughforcontainment}
    For two interpretations $\mathfrak{I}_1,\mathfrak{I}_2$, if $\mathfrak{I}_1$ satisfies every punctual fact $f$ satisfied by $\mathfrak{I}_2$, then $\mathfrak{I}_2\subseteq\mathfrak{I}_1$ 
\end{proposition}
\begin{proof}
    For each fact $f'= M@\varrho$ such that $\mathfrak{I}_2\models f'$, it must be so that $\mathfrak{I}_2,t \models M$ for each $t\in \varrho$, which means for each punctual fact $f''=M@t$ such that $t\in\varrho$, $\mathfrak{I}_2\models f''$ and $\mathfrak{I}_1\models f''$.
    This in turn means that for each $t\in \varrho$, $\mathfrak{I}_1,t \models f$, and finally, $\mathfrak{I}_1\models f'$.
\end{proof}

For an interpretation $\mathfrak{I}$, we say a dataset $\mathcal{D}$ \textbf{represents} $\mathfrak{I}$, denoted by $\mathcal{D} = \mathsf{rep}(\mathfrak{I})$, if $\mathfrak{I} = \mathfrak{I}_{\mathcal{D}}$, for each $M@\varrho_1,M@\varrho_2\in\mathcal{D}$ with the same relational atom either $\varrho_1\cup\varrho_2$ is not an interval if.

To help capture the derivations of each DatalogMTL facts, we define a structure \textbf{derivation tree}. 

We say that a ground rule $r$ \textbf{derives} a fact $f$ from a dataset $\mathcal{D}$ at $t$, for some time point $t$, if $\mathfrak{I}_{\mathcal{D}},t\models body(r')$, and for any interpretation $\mathfrak{I}$ such that $\mathfrak{I},t\models head(r')$, $\mathfrak{I}\models f$. 
When the specific value of $t$ is of no interest, we can avoid mentioning the existence of some $t$, and say $r$ derives $f$ from $\mathcal{D}$ if $r$ derives $f$ from $\mathcal{D}$ at some $t$.
Furthermore, we say a ground rule $r$ does not derive $f$ from $\mathcal{D}$, if no such $t$ exists such $r$ derives $f$ from $\mathcal{D}$ at $t$.

 We say a ground rule $r$ \textbf{minimally derives} $f$ from $\mathcal{D}$, if $r$ derives $f$ from $\mathcal{D}$ and for each $\mathcal{D}'$ such that $\mathfrak{I}_{\mathcal{D}'}\subsetneq\mathfrak{I}_{\mathcal{D}}$, $r$ does not derive $f$ from $\mathcal{D}'$.
\begin{proposition}\label{proposition:derivesimpliesminimal}
    If $r$ derives $f$ from $\mathcal{D}$ at $t$, then there must be $\mathcal{D}'$ such that $r$ minimally derives $f$ from $\mathcal{D}'$ at $t$ and $\mathfrak{I}_{\mathcal{D}'}\subseteq \mathfrak{I}_{\mathcal{D}}$.
\end{proposition}
\begin{proof}
    Let the least model satisfying $body(r)$ at $t$ be $\mathfrak{I}$, then $\mathfrak{I}\subseteq \mathfrak{I}_{\mathcal{D}}$, and $r$ minimally derives $f$ from $\mathsf{rep}(\mathfrak{I})$ at $t$.
\end{proof}

A derivation tree represents DatalogMTL facts and derivations of facts from sets of facts by its nodes and edges, each node is labelled with a single fact. Different nodes can be labelled with the same fact.
Its directed edges are labelled with ground rules, and they go into the node representing the derived fact.
For a non-leaf derivation tree node, let the fact it represents be $f$, and the set of facts represented by its in-neighbours be $\mathcal{IN}$, 
then all its incoming edges must be labelled by a ground rule instance $r'$ of a rule $r\in\Pi$, 
such that $r'$ derives $f$ from $\mathcal{IN}$. 
Formally, a derivation tree is an ordered pair $(V,E)$, where $V$ is a set of labelled nodes and $E$ a set of labelled edges.
For a labelled node $v$ or a labelled edge $e$ of a derivation tree, we denote the label of $v$ by $l(v)$, and the label of $e$ by $l(e)$.

Given a fixed bounded DatalogMTL program $\Pi$ and a dataset $\mathcal{D}$, for each DatalogMTL fact $f$ and an ordinal $i$, we define a set of derivation trees, denoted by $\mathfrak{D}_{\Pi,\mathcal{D},i}(f)$, via transfinite induction on the index $i$ of the transfinite sequence of interpretations, $\mathfrak{I}_0, \mathfrak{I}_1,...,\mathfrak{I}_{\omega_1},....$, defined by successive applications of $T_\Pi$ to $\mathfrak{I}_\mathcal{D}$:
\begin{itemize}
    \item when i=0, $\mathfrak{D}_{\Pi,\mathcal{D},0}(f)=\emptyset$ if $\mathfrak{I}_{0}\not\models f$, and $\mathfrak{D}_{\Pi,\mathcal{D},0}(f)=\{(\{V_f\},\emptyset)\}$ otherwise.
    \item When $i=\alpha+1$, where $\alpha$ is a successor ordinal, for each dataset $\mathcal{B}$ such that $\mathfrak{I}_{\alpha}\models\mathcal{B}$,
    and there is a ground instance $r'$ of a rule $r\in\Pi$, such that $r'$ \textbf{minimally derives} $f$ from $\mathcal{B}$, we construct a set of derivation trees $T(\Pi,\mathcal{D},\alpha+1,r',\mathcal{B},f)$ as follows
    \begin{enumerate}
        \item let $\mathcal{B}=\{f_1,f_2,...,f_n\}$ for some n, and let
        \begin{align*}
            \bar{S} = \{(T_1,T_2,...,T_n)\mid T_i\in \mathfrak{D}_{\Pi,\mathcal{D},\alpha}(f_i)\}.
        \end{align*}
        \item For a tuple $\bm{s}\in \bar{S}$, let $\bm{s} = (T_1,T_2,...,T_n)$, $T_i = (V_i,E_i)$, the root of $T_i$ be $t_i$ then, let
        \begin{align*}
            T(\Pi,\mathcal{D},\alpha+1,r',\mathcal{B},f) = \{T\mid T = (V,E), \text{ such that}\\ \text{for some $\bm{s}\in\bar{S}$, } V = \bigcup_{i=1}^n V_i\cup\{t_f\}, t_f \notin \bigcup_{i=1}^n V_i
            \\ E = \bigcup_{i=1}^n E_i \cup\{e_i\mid e_i = (t_i,t_f), e_i\text{ labelled by } r'\}
            \}
        \end{align*}
    \end{enumerate}
    And we have
    \begin{align*}
        \mathfrak{D}_{\Pi,\mathcal{D},\alpha+1}(f) = \mathfrak{D}_{\Pi,\mathcal{D},\alpha}(f)\cup \bigcup _{\mathcal{B},r'} T(\Pi,\mathcal{D},\alpha+1,r',\mathcal{B},f)
    \end{align*}
    \item When $i=\beta$, where $\beta$ is a limit ordinal, then
    \begin{align*}
    \mathfrak{D}_{\Pi,\mathcal{D},\beta}(f) = \bigcup_{\alpha<\beta}\mathfrak{D}_{\Pi,\mathcal{D},\alpha}(f),
    \end{align*}
    where $\alpha$ iterates over successor ordinals.
\end{itemize}

For the first limit ordinal $\omega_1$, we denote $\mathfrak{D}_{\Pi,\mathcal{D},\omega_1}(f)$ by $\mathfrak{D}_{\Pi,\mathcal{D}}(f)$. Clearly, for a fact $f$, \begin{align*}
    \mathfrak{D}_{\Pi,E}(f)\neq \emptyset \iff \mathfrak{C}_{\Pi,E}\models f
\end{align*} 

For a DatalogMTL fact $f$, program $\Pi$, datasets $E,E^-$ such that $\mathfrak{I}_{E^-}\subseteq \mathfrak{I}_{E}$, we say $f$ is \textbf{partially rooted} in $E^-$ regarding $\Pi,E$, 
if there exists $(V,E)\in\mathfrak{D}_{\Pi,E}(f)$, such that there exists $v\in V$, for which $\mathfrak{I}_{E^-}\cap \mathfrak{I}_{\{l(v)\}}\neq\mathfrak{I}_{\emptyset}$. 
We omit the `regarding $\Pi,E$'  when the context is clear.
\begin{proposition}
\label{proposition:partiallyrootedsub}
    If there is a derivation tree $T=(V,\mathcal{E})$, a node $v\in V$ such that $T\in\mathfrak{D}_{\Pi,E}(f)$ and $l(v)$ is partially rooted in $E^-$, then $f$ is also partially rooted in $E^-$.
\end{proposition}
\begin{proof}
    Let the tree be $T=(V,\mathcal{E})$ and the partially rooted fact be $f'$, the node representing it be $v_{f'}\in V$. 
    Then there must a derivation tree $T'\in\mathfrak{D}_{\Pi,E}(f')$, such that $T'$ contains a node $v'$, such that $\mathfrak{I}_{E^-}\cap\mathfrak{I}_{l(v')}\neq \mathfrak{I}_{\emptyset}$.
    By the way we construct derivation trees we know that if we replace the subtree of $T$ rooted at $v_{f'}$ with $T'$, the tree we derive is also in $\mathfrak{D}_{\Pi,E}(f)$, and this tree contains the node $v'$, which means $f$ is partially rooted in $E^-$.
\end{proof}

\begin{proposition}\label{proposition:leafrange}
    For a fact $f$ and a dataset $\mathcal{D'}$ such that $ \mathfrak{I}_{\mathcal{D'}}\subseteq\mathfrak{I}_\mathcal{D}$, if $\exists T = (V,E)$, $T\in\mathfrak{D}_{\Pi,\mathcal{D}}(f)$, such that $\forall v\in \{v\in V \mid \forall(v_1,v_2)\in E, v_2\neq v\}$ , $\mathfrak{I}_{\mathcal{D'}}\models l(v)$, then $\mathfrak{C}_{\Pi,\mathcal{D}'}\models f$.
\end{proposition}
\begin{proof}
    We define the height of each node in $T$ inductively, denoted by $h(v)$ for a $v\in V$, as follows:
    \begin{itemize}
        \item The height of a leaf of $T$ is 0.
        \item For a non-leaf node, its height is the maximum height of all its in-neighbours plus 1.
    \end{itemize}
    Let the root node of $T$ be $root$, and let $h(root) = max$.

    We show inductively for each $i\in \{0,\dots,max\}$ that $\forall v\in \{v\in V\mid h(v)=i\}$, $T^i_\Pi(\mathfrak{I}_{\mathcal{D}'})\models l(v)$.

    The base case, where $i=0$, is trivial because the label of each leaf node of $T$ is satisfied by $\mathfrak{I}_{\mathcal{D}'}$.

    For the inductive step, consider a node $v$ such that $h(v) =i+1$. 
    Then all its in-neighbours must be satisfied by $T^i_{\Pi}(\mathfrak{I}_{\mathcal{D}'})$ by the induction hypothesis.
    Let the label of all its incoming edges be $r'$, and we have $T^i_{\Pi}(\mathfrak{I}_{\mathcal{D}'})\models body(r')$, whereby $T^{i+1}_{\Pi}(\mathfrak{I}_{\mathcal{D}'})\models head(r')$. 
    Since by definition of a derivation tree, $r'$ minimally derives $l(v)$ from the dataset represented by the in-neighbours of $v$, we have $T^{i+1}_{\Pi}(\mathfrak{I}_{\mathcal{D}'})\models l(v)$.
\end{proof}

\subsection{Properties about Deletion}
Assuming the same context, we say the \textbf{max deletion-impact interpretation}, abbreviated as the mdi-interpretation, denoted by $\mathfrak{M}^-_{\Pi,E,E^-}$, of $E^-$ regarding $\Pi,E$ is the least in $\bm{\mathfrak{I}}$, 
such that $\forall\mathfrak{I}\in\bm{\mathfrak{I}}$, for each fact $f$ partially rooted in $E^-$ regarding $\Pi,E$, $\mathfrak{I}\models f$. 
We may call the interpretations in $\bm{\mathfrak{I}}$ the \textbf{max deletion-impact models} of $E^-$ regarding $\Pi,E$. 

\begin{proposition}\label{proposition:mdiiincanon}
$\mathfrak{M}_{\Pi,E,E^-}^-\subseteq\mathfrak{C}_{\Pi,E}$
\end{proposition}
\begin{proof}
    For a fact $f$, if it is satisfied by $\mathfrak{M}_{\Pi,E,E^-}^-$, then by definition $\mathfrak{D}_{\Pi,E}(f)$ is not empty, and therefore
    $\mathfrak{C}_{\Pi,E}\models f$.
\end{proof}

\begin{proposition}\label{proposition:overdeletioncomplete}
$\mathfrak{C}_{\Pi,E} - \mathfrak{C}_{\Pi,E\setminus E^-\cup E^+} \subseteq \mathfrak{M}_{\Pi,E,E^-}^-$    
\end{proposition}
\begin{proof}
    Consider a fact $f$ such that $\mathfrak{C}_{\Pi,E}\models f$ and $\mathfrak{C}_{\Pi,E\setminus E^-\cup E^+}\not\models f$.
    Because $\mathfrak{C}_{\Pi,E\setminus E^-}\subseteq \mathfrak{C}_{\Pi,E\setminus E^-\cup E^+}$, $\mathfrak{C}_{\Pi,E\setminus E^-}\not\models f$.
    Let $T\in \mathfrak{D}_{\Pi,E}(f)$, $T = (V,\mathcal{E})$. 
    Then by Proposition~\ref{proposition:leafrange}, $\exists v\in \{v\in V\mid \forall (v_1,v_2)\in \mathcal{E}, v_2\neq v\}$, $\mathfrak{I}_{E\setminus E^-}\not\models l(v)$, and therefore $\mathfrak{I}_{E^-}\models l(v)$, $f$ is partially rooted in $E^-$, which means $ \mathfrak{M}_{\Pi,E,E^-}^-\models f$.
\end{proof}
Given a DatalogMTL program $\Pi$, dataset $E$, we define the \textbf{immediate deletion-consequence operator} $D_{\Pi,E}$, 
mapping an interpretation $\mathfrak{I}$ into the least interpretation $D_{\Pi,E}(\mathfrak{I})$ containing $\mathfrak{I}$ satisfying for each ground rule instance $r'$ of a rule in $\Pi$, a fact $f$,
a dataset $\mathcal{D}$ such that if $r'$ minimally derives $f$ from $\mathcal{D}$, $\mathfrak{C}_{\Pi,E}\models \mathcal{D}$, 
and $\mathfrak{I}\cap\mathfrak{I}_{\mathcal{D}}\neq\mathfrak{I}_{\emptyset}$, then $D_{\Pi,E}(\mathfrak{I})\models f$. 

For a given $E^-$ such that $\mathfrak{I}_{E^-}\subseteq \mathfrak{I}_{E}$, successive applications of $D_{\Pi,E}$ to $\mathfrak{I}_{E^-}$ yield a transfinite sequence of interpretations such that 
\begin{enumerate*}[label=(\roman*)]
    \item $D_{\Pi,E}^{0}(\mathfrak{I}_{E^-}) = \mathfrak{I}_{E^-}$,
    \item $D_{\Pi,E}^{\alpha+1}(\mathfrak{I}_{E^-}) = D_{\Pi,E}(D_{\Pi,E}^{\alpha}(\mathfrak{I}_{E^-}))$, for $\alpha$ a successor ordinal, and
    \item $D_{\Pi,E}^{\beta}(\mathfrak{I}_{E^-}) = \bigcup_{\alpha<\beta}D_{\Pi,E}^{\alpha}(\mathfrak{I}_{E^-})$, for $\beta$ a limit ordinal
\end{enumerate*}
\begin{proposition}\label{proposition:overdeletionruler}
    For each $\alpha\leq\omega_1,D^{\alpha}_{\Pi,E}(\mathfrak{I}_{E^-})$ is a $(\Pi,E\cup E^+)$-interpretation.
\end{proposition}
\begin{proof}
    We prove this by induction.

    When $\alpha = 0$, this is trivially true because $\mathfrak{I}_{E^-}$ is a $(\Pi,E\cup E^+)$-interpretation.

    For the induction step, consider a fact $f=M@t$ such that $D^{\alpha+1}_{\Pi,E}(\mathfrak{I}_{E^-})\models f$ and $D^{\alpha}_{\Pi,E}(\mathfrak{I}_{E^-})\not\models f$. 
    If $t$ is on the $(\Pi,E)$-ruler, then $D^{\alpha+1}_{\Pi,E}(\mathfrak{I}_{E^-})\models M@t$ already implies $D^{\alpha+1}_{\Pi,E}(\mathfrak{I}_{E^-})\models M@\varrho$ for the $(\Pi,E)$-interval $\varrho$ containing $t$, which is $t$.
    If $t$ is not on the $(\Pi,E)$-ruler,
    then there must be a rule $r'$ a dataset $\mathcal{D}$ and a time point $t'$ such that $r'$ minimally derives $f$ at $t'$, $\mathfrak{C}_{\Pi,E}\models \mathcal{D}$ and $\mathfrak{I}_{\mathcal{D}}\cap D^{\alpha}_{\Pi,E}(\mathfrak{I}_{E^-})\neq \mathfrak{I}_{\emptyset}$. 
    Let $\varrho'$ be the $(\Pi,E)$-interval containing $t'$. 
    Then by Proposition~\ref{proposition:pi,dinterpretationatomsat}, $\forall t''\in \varrho'$, $\mathfrak{C}_{\Pi,E},t''\models body(r')$.
    Because $\mathfrak{I}_\mathcal{D}-D^{\alpha}_{\Pi,E}(\mathfrak{I}_{E^-}),t'\not\models body(r')$, $\mathfrak{C}_{\Pi,E}-D^{\alpha}_{\Pi,E}(\mathfrak{I}_{E^-}),t'\not\models body(r')$, and by Proposition~\ref{proposition:pi,dinterpretationatomnotsat}, $\forall t''\in \varrho'$, $\mathfrak{C}_{\Pi,E}-D^{\alpha}_{\Pi,E}(\mathfrak{I}_{E^-}),t''\not\models body(r')$. 
    And therefore, $\forall t''\in \varrho'$, $\exists \mathcal{D}'$ such that $r'$ minimally derives $M@t+t''-t'$ from $\mathcal{D}'$ at $t''$, $\mathfrak{C}_{\Pi,E}\models\mathcal{D}'$, $\mathfrak{I}_{\mathcal{D}'}$ is a shift of $\mathfrak{I}_{\mathcal{D}}$, and $\mathfrak{I}_{\mathcal{D}'}\cap D^{\alpha}_{\Pi,E}(\mathfrak{I}_{E^-})\neq \mathfrak{I}_{\emptyset}$, which is because $\mathfrak{I}_{\mathcal{D}'}-D^{\alpha}_{\Pi,E}(\mathfrak{I}_{E^-}),t''\not\models body(r')$.
    
    Let $\mathfrak{I}$ be the least interpretation that satisfies $head(r')$ at $t'$.  
    By Proposition~\ref{proposition:headexpandsion}, for some $\varrho$ such that $\varrho^- = t'+t_1$ $\varrho^+ = t'+t_2$, where $t_1$ and $t_2$ are multiples of $div(\Pi)$, $\mathfrak{I}$ maps exactly each $t''\in\varrho$ to $\{M\}$.
    If $t_1 = t_2$, then it must be so that $t = t'+t_1$, and by Proposition~\ref{proposition:timepointsdiv(pi)apartcontainedbyintervalsofequallength}, $\varrho'$ and $\varrho$ are of the same length, which means $\forall t''\in \varrho$, there $\exists \mathcal{D}'$ such that $r'$ minimally derives $M@t''$ from $\mathcal{D}'$ at $t''-t_1$, $\mathfrak{C}_{\Pi,E}\models\mathcal{D}'$, and $\mathfrak{I}_{\mathcal{D}'}\cap D^{\alpha}_{\Pi,E}(\mathfrak{I}_{E^-})\neq \mathfrak{I}_{\emptyset}$.

    If $t_1\neq t_2$, then because $\forall t''\in \varrho'$, there $\exists \mathcal{D}'$ such that $r'$ minimally derives $M@t+t''-t'$ from $\mathcal{D}'$ at $t''$, $\mathfrak{C}_{\Pi,E}\models\mathcal{D}'$, and $\mathfrak{I}_{\mathcal{D}'}\cap D^{\alpha}_{\Pi,E}(\mathfrak{I}_{E^-})\neq \mathfrak{I}_{\emptyset}$, $\forall t''\in \varrho'$, there $\exists \mathcal{D}'$ such that $r'$ minimally derives $M@(t''+t_1,t''+t_2)$ from $\mathcal{D}'$ at $t''$, $\mathfrak{C}_{\Pi,E}\models\mathcal{D}'$, and $\mathfrak{I}_{\mathcal{D}'}\cap D^{\alpha}_{\Pi,E}(\mathfrak{I}_{E^-})\neq \mathfrak{I}_{\emptyset}$,
    and $D^{\alpha+1}_{\Pi,E}(\mathfrak{I}_{E^-})\models M@(\varrho'+t_1,\varrho'+t_2)$.
    Because $t'\in \varrho'$, $t\in(\varrho'+t_1,\varrho'+t_2)$, and because $\varrho'+t_1,\varrho'+t_2$ are both time points on the $(\Pi,E)$-ruler, $(\varrho'+t_1,\varrho'+t_2)$ contains $\varrho$, and we have $D^{\alpha+1}_{\Pi,E}(\mathfrak{I}_{E^-})\models M@t \implies D^{\alpha+1}_{\Pi,E}(\mathfrak{I}_{E^-})\models M@\varrho$. 
    
\end{proof}
\begin{proposition}\label{proposition:overdeletionleast}
    $D_{\Pi,E}^{\omega_1}(\mathfrak{I}_{E^-})=\mathfrak{M}^-_{\Pi,E,E^-}$ 
\end{proposition}
\begin{proof}
    We prove by induction on the ordinals less or equal to $\omega_1$ that for each fact $f$ such that $D^{\alpha}_{\Pi,E}\models f$, $\mathfrak{M}_{\Pi,E,E^-}^- \models f$. 
    For $\alpha=0$, it is trivial that a fact satisfied by $\mathfrak{I}_{E^-}$ is partially rooted in $E^-$, since $\mathfrak{D}_{\Pi,E}(f)$ contains the derivation tree $(\{v_f\},\emptyset)$.

    For $\alpha+1$ where $\alpha<\omega_1$, we consider a fact $f$ satisfied by $D^{\alpha+1}_{\Pi,E}(\mathfrak{I}_{E^-})$ but not by $D^{\alpha}_{\Pi,E}(\mathfrak{I}_{E^-})$. 
    Then there must be a ground instance $r'$ of a rule in $\Pi$, a dataset $\mathcal{D}$ satisfied by $\mathfrak{C}_{\Pi,E}$, and $D^{\alpha}_{\Pi,E}(\mathfrak{I}_{E^-})\cap\mathfrak{I}_{\mathcal{D}}\neq \mathfrak{I}_{\emptyset}$, such that $r'$ minimally derives $f$ from $\mathcal{D}$.
    Then there must be a fact $f'\in\mathcal{D}$ such that $\mathfrak{I}_{\{f'\}}\cap D^{\alpha}_{\Pi,E}(\mathfrak{I}_{E^-})\neq \emptyset$. 
    Let $\mathcal{D}_1'$ represent $\mathfrak{I}_{\{f\}}\cap D^{\alpha}_{\Pi,E}(\mathfrak{I}_{E^-})$ and $\mathcal{D}_2'$ represent $\mathfrak{I}_{\{f\}}- D^{\alpha}_{\Pi,E}(\mathfrak{I}_{E^-})$, 
    and let $\mathcal{D}' = \mathcal{D}-\{f'\}\cup \mathcal{D}_1'\cup \mathcal{D}_2'$, then it is clear that $\mathfrak{I}_{\mathcal{D}} = \mathfrak{I}_{\mathcal{D}'}$, $\mathfrak{C}_{\Pi,E}\models\mathcal{D}'$, and $r'$ minimally derives $f$ from $\mathcal{D}'$. Let $f_1' \in \mathcal{D}_1'$, then $\mathfrak{D}_{\Pi,E}(f)$ contains a tree that has a node labelled $f_1'$. By the induction hypothesis, $f_1'$ is partially rooted in $E^-$, and by Claim~\ref{proposition:partiallyrootedsub}, $f$ is partially rooted in $E^-$.
    
    Then case when $\alpha=\omega_1$ is trivial, since $D^{\omega_1}_{\Pi,E}(\mathfrak{I}_{E^-})=\bigcup_{\beta<\omega_1}D^{\beta}_{\Pi,E}(\mathfrak{I}_{E^-})$ 

    Consequently, $D^{\omega_1}_{\Pi,E}(\mathfrak{I}_{E^-})\subseteq\mathfrak{M}^-_{\Pi,E,E^-}$.
    
    For a fact $f$, if $f$ is partially rooted in $E^-$ regarding $\Pi,E$, then $\mathfrak{D}_{\Pi,E}(f)\neq\emptyset$, 
    and there must a tree $T\in\mathfrak{D}_{\Pi,E}(f)$ such that T contains a node $v_{f'}$ representing a fact $f'$ satisfying $\mathfrak{I}_{\{f'\}}\cap\mathfrak{I}_{E^-}\neq\mathfrak{I}_{\emptyset}$. 
    Let the (reverse) path in $T$ from its root $v$ to $v_{f'}$ be $v,v_1,v_2,...,v_k,v_{f'}$ for some $k$.
    Then by Proposition~\ref{proposition:partiallyrootedsub}, $l(v_i)$ is partially rooted in $E^-$, for $i=1,2...,k$.
    Let $l(v_i) = f_i$.
    If $D^{\omega_1}_{\Pi,E}(\mathfrak{I}_{E^-})\not\models f$, then $D^{\omega_1}_{\Pi,E}(\mathfrak{I}_{E^-})\not\models f_1$, otherwise because $f$ is minimally derived by $l(v_1,v)$ from the dataset containing exactly the facts represented by the in-neighbours of $v$, a dataset which is satisfied by $\mathfrak{C}_{\Pi,E}$ and $D^{\omega_1}_{\Pi,E}(\mathfrak{I}_{E^-})\models f_1$, which means $D^{\alpha}_{\Pi,E}(\mathfrak{I}_{E^-})\models f_1$ for some $\alpha<\omega_1$ because $D^{\omega_1}_{\Pi,E}(\mathfrak{I}_{E^-})$ is a $(\Pi,E)$-interpretation,
    $D^{\omega_1}_{\Pi,E}(\mathfrak{I}_{E^-})\models f$, which contradicts $D^{\omega_1}_{\Pi,E}(\mathfrak{I}_{E^-})\not\models f$. 
    Likewise, $D^{\omega_1}_{\Pi,E}(\mathfrak{I}_{E^-})\not\models f_2,f_3,...,f_k$.
    There must be a dataset $\mathcal{D}'$ such that $l((v_{f'},v_k))$ minimally derives $f_k$ from $\mathcal{D}'$, and the facts in $\mathcal{D}'$ are exactly represented by the in-neighbour of $v_{f_k}$ in $T$.
    Then, because $\mathfrak{I}_{\{f'\}}\cap\mathfrak{I}_{E^-}\neq \mathfrak{I}_{\emptyset}$ and $\mathfrak{C}_{\Pi,E}\models\mathcal{D}'$, $D_{\Pi,E}(\mathfrak{I}_{E^-})\models f_k$, which means $D^{\omega_1}_{\Pi,E}(\mathfrak{I}_{E^-})\models f_k$, and this contradicts $D^{\omega_1}_{\Pi,E}(\mathfrak{I}_{E^-})\not\models f_k$. 
    Therefore, $D^{\omega_1}_{\Pi,E}(\mathfrak{I}_{E^-})\models f$, $D^{\omega_1}_{\Pi,E}(\mathfrak{I}_{E^-})$ is a max deletion-impact model of $E^-$ regarding $\Pi,E$, and $\mathfrak{M}_{\Pi,E,E^-}^-\subseteq D^{\omega_1}_{\Pi,E}(\mathfrak{I}_{E^-})$
    
\end{proof}
Note that from this proof, we can also derive that
\begin{proposition}\label{proposition:AllfactsinDomegapartiallyrooted}
    For a fact $f$, if $D^{\omega_1}_{\Pi,E}(\mathfrak{I}_{E^-})\models f$, then $f$ is partially rooted in $E^-$.
\end{proposition}
From the above proposition, we have that
\begin{proposition}\label{proposition:subsumedfactsalsopartiallyrooted}
If a fact $f$ is partially rooted in $E^-$ regarding $\Pi,E$, then for any fact $f'$ such that $\mathfrak{I}_{\{f\}}\models f'$, we have $f'$ is partially rooted in $E^-$.
\end{proposition}
\begin{proof}
    Because $f$ is partially rooted in $E^-$, by Proposition~\ref{proposition:overdeletionleast} we have $D^{\omega_1}_{\Pi,E}(\mathfrak{I}_{E^-})\models f$, and $\mathfrak{I}_{\{f\}}\subseteq D^{\omega_1}_{\Pi,E}(\mathfrak{I}_{E^-})$.
    Therefore, $D^{\omega_1}_{\Pi,E}(\mathfrak{I}_{E^-})\models f'$
    , and by Proposition~\ref{proposition:AllfactsinDomegapartiallyrooted}, we have that $f'$ is partially rooted in $E^-$.
\end{proof}
We say that an interpretation $\mathfrak{I}$ \textbf{satisfies with over-deletion} of $E^-$ a ground rule $r$ regarding $\Pi,E$, if for each time point $t$, whenever there is a dataset $\mathcal{D}$ such that $\mathfrak{C}_{\Pi,E} \models \mathcal{D}$, $\mathfrak{I}_{\mathcal{D}},t\models body(r)$, and there exists a fact $f'\in\mathcal{D}$ partially rooted in $E^-$ such that $\mathfrak{I}_{\mathcal{D}-\{f'\}},t\not\models body(r)$, then $\mathfrak{I},t\models head(r)$.

Moreover, We say that an interpretation $\mathfrak{I}$ \textbf{satisfies with over-deletion} of $E^-$ a ground rule $r$ regarding $\Pi,E$ during $\varrho$, if for each time point $t\in\varrho$, whenever there is a dataset $\mathcal{D}$ such that $\mathfrak{C}_{\Pi,E} \models \mathcal{D}$, $\mathfrak{I}_{\mathcal{D}
},t\models body(r)$, and there exists a fact $f'\in\mathcal{D}$ partially rooted in $E^-$ such that $\mathfrak{I}_{\mathcal{D}-\{f'\}},t\not\models body(r)$, then $\mathfrak{I},t\models head(r)$.

We say that an interpretation $\mathfrak{I}$
\textbf{satisfies with over-deletion} of $E^-$ program $\Pi$ regarding $\Pi,E$ (during $\varrho$), if it satisfies with over-deletion of $E^-$ each ground instance $r'$ of each rule $r\in\Pi$ (during $\varrho$), and $\mathfrak{I}$ is an \textbf{over-deletion model} of $\Pi$ regarding $\Pi,E,E^-$.
If a model $\mathfrak{I}$ is both a model of $E^-$ and an over-deletion model of $\Pi$ regarding $\Pi,E,E^-$, then $\mathfrak{I}$ is the \textbf{over-deletion model} of $\Pi,E,E^-$.

Next, we show over-deletion models and max-deletion models describe the same interpretations for $\Pi,E,E^-$.
\begin{proposition}\label{proposition:deletionmodelsaremaxmodels}
    Given a DatalogMTL program $\Pi$,  datasets $E,E^-$ such that $\mathfrak{I}_{E^-} \subseteq \mathfrak{I}_E$, if an interpretation $\mathfrak{I}$ is an over-deletion model  of $\Pi,E,E^-$, then $\mathfrak{I}$ is a max deletion-impact model of $E^-$ regarding $\Pi,E$. 
\end{proposition}
\begin{proof}
    We show that $D^{\omega_1}_{\Pi,E}(\mathfrak{I}_{E^-})\subseteq\mathfrak{I}$ by proving inductively that for each ordinal $\alpha<\omega_1$, $D^{\alpha}_{\Pi,E}(\mathfrak{I}_{E^-})\subseteq\mathfrak{I}$, and each $f$.

    The base case when $\alpha=0$ is trivial since $\mathfrak{I}$ is a model of $E^-$,

    In the inductive step for $\alpha+1$, for a fact $f$ such that $D^{\alpha+1}_{\Pi,E}(\mathfrak{I}_{E^-})\models f$, we only need to consider when $D^{\alpha}_{\Pi,E}(\mathfrak{I}_{E^-})\not\models f$. 
    Then, there must be a ground instance $r'$ of a rule $r\in\Pi$, a time point $t$, a dataset $\mathcal{D}$ satisfying $\mathfrak{C}_{\Pi,E}\models\mathcal{D}$, and $D^{\alpha}_{\Pi,E}(\mathfrak{I}_{E^-})\cap \mathfrak{I}_\mathcal{D}\neq \mathfrak{I}_\emptyset$ such that $r'$ minimally derives $f$ from $\mathcal{D}$ at $t$. 
    Then there must be a fact $f'\in\mathcal{D}$ such that $\mathfrak{I}_{\{f'\}}\cap D^{\alpha}_{\Pi,E}(\mathfrak{I}_{E^-})\neq\mathfrak{I}_\emptyset $.
    Let $\mathcal{D}_1$ represent $\mathfrak{I}_{\{f'\}}\cap D^{\alpha}_{\Pi,E}(\mathfrak{I}_{E^-})$ and $\mathcal{D}_2$ represent $\mathfrak{I}_{\{f'\}}- D^{\alpha}_{\Pi,E}(\mathfrak{I}_{E^-})$, and let $\mathcal{D}'=\mathcal{D}-\{f'\}\cup \mathcal{D}_1\cup\mathcal{D}_2$, then $\mathcal{D}_1\neq\emptyset$ and $r'$ minimally derives $f$ from $\mathcal{D}'$, and $\mathfrak{C}_{\Pi,E}\models \mathcal{D}'$.
    Let $f_1'\in \mathcal{D}_1$, then $f_1'\in \mathcal{D}'$, $D^{\alpha}_{\Pi,E}(\mathfrak{I}_{E^-})\models f_1'$ and $f_1'$ is partially rooted in $E^-$ by Claim~\ref{proposition:AllfactsinDomegapartiallyrooted}.
    Also, because $r'$ minimally derives $f$ from $\mathcal{D}$, $\mathcal{D}-\{f_1'\},t\not\models body(r')$.
    Therefore, by the definition of an over-deletion model, $\mathfrak{I},t\models head(r)$, which in turn means $\mathfrak{I}\models f$, because $\mathfrak{I}$ satisfies with over-deletion of $E^-$ $r'$ regarding $\Pi,E$. 
\end{proof}

\begin{proposition}\label{proposition:maxmodelsaredeletionmodels}
    Given a DatalogMTL program $\Pi$,  datasets $E,E^-$ such that $\mathfrak{I}_{E^-} \subseteq \mathfrak{I}_E$, if an interpretation $\mathfrak{I}$ is a max deletion-impact model of $E^-$ regarding $\Pi,E$, then $\mathfrak{I}$ is an over-deletion model  of $\Pi,E,E^-$. 
\end{proposition}
\begin{proof}
    We show that if $\mathfrak{I}$ contains $D^{\omega_1}_{\Pi,E}(\mathfrak{I}_{E^-})$, then it is a over-deletion model of $\Pi,E,E^-$.
    For each ground instance $r'$ of each rule $r$ and each time point $t$, if there is a dataset $\mathcal{D}$ such that $\mathfrak{C}_{\Pi,E}\models\mathcal{D}$, $\mathfrak{I}_{\mathcal{D}},t\models body(r')$ and there exists $f'\in\mathcal{D}$ partially rooted in $E^-$ such that $\mathfrak{I}_{\mathcal{D}-\{f'\}}\not\models body(r')$, then let the set of all facts derived by $r'$ from $\mathcal{D}$ at $t$ be $\mathcal{F}$. 
    Clearly, $\mathfrak{I}_{\mathcal{F}},t\models head(r)$.

    Let $\mathcal{D}'$ be a set of facts such that $\mathfrak{I}_{\mathcal{D}'}\subsetneq\mathfrak{I}_{\mathcal{D}}$, $\mathfrak{I}_{\mathcal{D}'},t\models body(r')$, and for any $\mathfrak{I}\subsetneq \mathfrak{I}_{\mathcal{D}'}$, $\mathfrak{I},t\not\models body(r')$. 
    Clearly, $\forall f\in \mathcal{F}$, $r'$ minimally derives $f$ from $\mathcal{D}'$.
    We have $\mathfrak{I}_{\mathcal{D}'}\cap\mathfrak{I}_{\{f'\}}\neq \mathfrak{I}_{\emptyset}$ because if not, then we have $\mathfrak{I}_{\mathcal{D}'} = \mathfrak{I}_{\mathcal{D}'}-\mathfrak{I}_{\{f'\}}\subseteq\mathfrak{I}_{\mathcal{D}}-\mathfrak{I}_{\{f'\}}$, but since $\mathfrak{I}_{\mathcal{D}'},t\models body(r')$, $\mathfrak{I}_{\mathcal{D'}}\not\subseteq \mathfrak{I}_{\mathcal{D}}-\mathfrak{I}_{\{f'\}}$, we have a contradiction.
    Let $\mathcal{D}_1'$ represent $\mathfrak{I}_{\mathcal{D}'} \cap \mathfrak{I}_{\{f'\}}$, and $\mathcal{D}_2'$ represent $\mathfrak{I}_{\mathcal{D}'} - \mathfrak{I}_{\{f'\}}$, and let $\mathcal{D}_3' = \mathcal{D}' - \{f'\} \cup \mathcal{D}_1'\cup \mathcal{D}_2'$.
    Then we have that $r'$ minimally derives $f$ from $\mathcal{D}_3'$. 
    Let $f_1'\in\mathcal{D}'_1$, then $f_1'\in\mathcal{D}_3'$, and $f_1'$ is partially rooted in $E^-$ because $\mathfrak{I}_{\{f'\}}\models f_1'$, which means that $\forall f\in \mathcal{F}$, $f$ is partially rooted in $E^-$. 
    Consequently, $\forall f\in \mathcal{F}$, $D^{\omega_1}_{\Pi,E}(\mathfrak{I}_{E^-})\models f$, and therefore, $D^{\omega_1}_{\Pi,E}(\mathfrak{I}_{E^-}),t\models head(r)$, and we have $D^{\omega_1}_{\Pi,E}(\mathfrak{I}_{E^-})$ and $\mathfrak{I}$ satisfies with over-deletion of $E^-$ ground rule $r'$ regarding $\Pi,E$ at $t$.
    Consequently, $\mathfrak{I}$ is an over-deletion model of $\Pi,E,E^-$.
\end{proof}
From Propositions~\ref{proposition:deletionmodelsaremaxmodels} and~\ref{proposition:maxmodelsaredeletionmodels}, we can also conclude that 
\begin{corollary}\label{coro:leastdeletionmodel}
 $\mathfrak{M}_{\Pi,E,E^-}^-$ is the least over-deletion model of $\Pi,E,E^-$. 
\end{corollary}
Given bounded $\Pi$, $E$ and $E^-$ such that $\mathfrak{I}_{E^-}\subseteq\mathfrak{I}_{E}$, we prove the following lemma
\begin{lemma}\label{lemma:depthrep}
   Let $\varrho$ be a closed interval of length $\mathsf{depth}(\Pi)$
and $E'$ a dataset representing $\mathfrak{M}^-_{\Pi,E,E^-}\mid_{\varrho}$. Then both of the
following hold: 
\begin{enumerate}
    \item if $\varrho^+> t_E^+$, then $\mathfrak{M}^-_{\Pi,E,E^-}\mid_{(\varrho^+,\infty)} = \mathfrak{M}^-_{\Pi,E,E'}\mid_{(\varrho^+,\infty)} $,
    \item if $\varrho^-< t_E^-$, then $\mathfrak{M}^-_{\Pi,E,E^-}\mid_{(-\infty,\varrho^-)} = \mathfrak{M}^-_{\Pi,E,E'}\mid_{(-\infty,\varrho^-)} $.
\end{enumerate}
\end{lemma}
\begin{proof}
    We consider $\varrho^+>t^+_E$.
    
    Since $\mathfrak{I}_{E'}\subseteq\mathfrak{M}_{\Pi,E,E^-}^-$, we have $\mathfrak{M}_{\Pi,E,E'}^-\subseteq \mathfrak{M}_{\Pi,E,E^-}^-$, and therefore $\mathfrak{M}_{\Pi,E,E'}^-\mid_{(\varrho^+,\infty)}\subseteq \mathfrak{M}_{\Pi,E,E^-}^-\mid_{(\varrho^+,\infty)}$.

    To prove $\mathfrak{M}_{\Pi,E,E^-}^-\mid_{(\varrho^+,\infty)}\subseteq \mathfrak{M}_{\Pi,E,E'}^-\mid_{(\varrho^+,\infty)}$, we show inductively that for each ordinal $\alpha$, each ground relational atom $M$, each time point $t>\varrho^+$ $D^\alpha_{\Pi,E}(\mathfrak{I}_{E^-}),t\models M$ implies $D^\alpha_{\Pi,E}(\mathfrak{I}_{E'}),t\models M$.

    In the base case, the implication holds trivially because for each $t>\varrho^+$, there is no $M$ such that $\mathfrak{I}_{E'},t\models M$.

    In the induction step, consider $\alpha+1$ for an ordinal $\alpha$. 
    For a time point $t>\varrho^+$ and ground relational atom $M$, if $D^\alpha_{\Pi,E}(\mathfrak{I}_{E^-}),t\models M$, then the implication is guaranteed by the induction hypothesis. 
    If $D^\alpha_{\Pi,E}(\mathfrak{I}_{E^-}),t\not\models M$ and $D^{\alpha+1}_{\Pi,E}(\mathfrak{I}_{E^-}),t\models M$, then there must a ground rule instance $r'$, a dataset $\mathcal{D}$ such that $\mathfrak{C}_{\Pi,E}\models\mathfrak{I}_{\mathcal{D}}$, $\mathfrak{I}_D\cap D^\alpha_{\Pi,E}(\mathfrak{I}_{E^-})\neq\mathfrak{I}_{\emptyset}$, and $r'$ minimally derives $M@t$ from $\mathcal{D}$ at some $t'$. 

    Let the sum of every right endpoint of every interval mentioned in $head(r')$ be $m_1$, and the sum of right every endpoint mentioned in $body(r')$ be $m_2$. 
    Then by definition of $\mathsf{depth}(\Pi)$, $m_1+m_2 \leq \mathsf{depth}(\Pi)$.
    
     Because $r'$ derives $f$ at $t'$, we know that $t'\in [t-m_1,t+m_1]$ and $\forall M'@\varrho'\in\mathcal{D}$, $\varrho \subseteq [t'-m_2,t'+m_2]$.
     Since $m_1+m_2 \leq \mathsf{depth}(\Pi)$, and by the minimality of $\mathcal{D}$, every interval mentioned in $\mathcal{D}$ must be in $[t-\mathsf{depth}(\Pi),t+\mathsf{depth}(\Pi)]$. 
     And because $t>\varrho^+$ and $\varrho^+-\varrho^- = \mathsf{depth}(\Pi)$, $[t-\mathsf{depth}(\Pi),t+\mathsf{depth}(\Pi)]\subseteq [\varrho^-,t+\mathsf{depth}(\Pi)]$, and by the induction hypothesis, 
     $\mathfrak{I}_{\mathcal{D}}\cap D^{\alpha}_{\Pi,E}(\mathfrak{I}_{E^-})\subseteq D^{\alpha}_{\Pi,E}(\mathfrak{I}_{E'})$, and
     $D^{\alpha}_{\Pi,E}(\mathfrak{I}_{E'})\cap\mathfrak{I}_{\mathcal{D}}\neq \mathfrak{I}_{\emptyset}$, by which we have $D^{\alpha+1}_{\Pi,E}(\mathfrak{I}_{E'})\models f$.

     For a limit ordinal $\alpha$, the implication holds trivially because $D^{\alpha}_{\Pi,E}(\mathfrak{I}_{E'}) = \bigcup_{\beta < \alpha}D^{\alpha}_{\Pi,E}(\mathfrak{I}_{E'})$.

    The case for $\varrho^- < t_E^-$ is symmetric.
\end{proof}

From Lemma~\ref{lemma:depthrep}, Propositions~\ref{proposition:AllfactsinDomegapartiallyrooted} and ~\ref{proposition:mdiiincanon} we can conclude
\begin{corollary}\label{coro:recursivepartialrootedness}
    Let $\varrho$ be a closed interval of length $\mathsf{depth}(\Pi)$
and $E'$ a dataset representing $\mathfrak{M}^-_{\Pi,E,E^-}\mid_{\varrho}$. Then both of the
following hold: 
\begin{enumerate}
    \item if $\varrho^+> t_E^+$, then each fact $M@\varrho_1$ with $\varrho_1^->\varrho^-$ partially rooted in $E^-$ is also partially rooted in $E'$ and vice versa.
    \item if $\varrho^-< t_E^-$, then each fact $M@\varrho_1$ with $\varrho_1^+<\varrho^+$ partially rooted in $E^-$ is also partially rooted in $E'$ and vice versa.
\end{enumerate}
\end{corollary}
The following definitions are all made in the context that $\Pi,E,E^-$ are bounded. 
\begin{definition}
    Interpretation $D^k_{\Pi,E}(\mathfrak{I}_{E^-})$ is \textbf{deletion-saturated} regarding $\Pi,E,E^-$, if there exist $\varrho_1,\varrho_2,\varrho_3$ and $\varrho_4$ of length $2\mathsf{depth}(\Pi)$,
whose endpoints are located on the $(\Pi, E)$-ruler and satisfy
$\varrho_1^+<\varrho_2^+<t_E^-$
and $t_E^+<\varrho^-_3<\varrho_4^-$, and for which the
following properties hold:
\begin{enumerate}
    \item $\mathfrak{C}_{\Pi,E}\mid_{[\varrho_1^-,\varrho_4^+]}$ is a saturated interpretation with periods $[\varrho_1^-,\varrho_2^-)$ and $(\varrho_3^+,\varrho_4^+]$.
    \item $D^k_{\Pi,E}(\mathfrak{I}_{E^-})$ is an over-deletion model of $\Pi,E,E^-$ during $[\varrho_1^- +\mathsf{depth}(\Pi),\varrho_4^+-\mathsf{depth}(\Pi)]$.
    \item $D^k_{\Pi,E}(\mathfrak{I}_{E^-})\mid_{\varrho_2}$ is a shift of $D^k_{\Pi,E}(\mathfrak{I}_{E^-})\mid_{\varrho_1}$, 
    \item $D^k_{\Pi,E}(\mathfrak{I}_{E^-})\mid_{\varrho_4}$ is a shift of $D^k_{\Pi,E}(\mathfrak{I}_{E^-})\mid_{\varrho_3}$.
\end{enumerate}
Each pair of intervals $[\varrho_1^-,\varrho_2^-)$ and $(\varrho^+_3,\varrho_4^+]$ for each quadruple $\varrho_1,\varrho_2,\varrho_3,\varrho_4$ as above, can be referred to as the \textbf{deletion periods} of $D^k_{\Pi,E}(\mathfrak{I}_{E^-})$. 
Alternatively, we can denote them as $(\varrho_{\mathit{left}},\varrho_{\mathit{right}})$
\end{definition}
\begin{proposition} \label{proposition:deletion-saturated}
        Interpretation $D^k_{\Pi,E}(\mathfrak{I}_{E^-})$ is also \textbf{deletion-saturated} regarding $\Pi,E,E^-$, if there exist $\varrho_1,\varrho_2,\varrho_3$ and $\varrho_4$ of length $2\mathsf{depth}(\Pi)$,
whose endpoints are located on the $(\Pi, E)$-ruler and satisfy
$\varrho_1^+<\varrho_2^+<t_E^-$
and $t_E^+<\varrho^-_3<\varrho_4^-$, and for which the
following properties hold:
\begin{enumerate}
    \item $\mathfrak{C}_{\Pi,E}\mid_{[\varrho_1^-,\varrho_4^+]}$ is a saturated interpretation with periods $[\varrho_1^-,\varrho_2^-)$ and $(\varrho_3^+,\varrho_4^+]$.
    \item $D^{k+1}_{\Pi,E}(\mathfrak{I}_{E^-})\mid_{[\varrho^-_1,\varrho_4^+]} = D^{k}_{\Pi,E}(\mathfrak{I}_{E^-})\mid_{[\varrho^-_1,\varrho_4^+]} $.
    \item $D^k_{\Pi,E}(\mathfrak{I}_{E^-})\mid_{\varrho_2}$ is a shift of $D^k_{\Pi,E}(\mathfrak{I}_{E^-})\mid_{\varrho_1}$, 
    \item $D^k_{\Pi,E}(\mathfrak{I}_{E^-})\mid_{\varrho_4}$ is a shift of $D^k_{\Pi,E}(\mathfrak{I}_{E^-})\mid_{\varrho_3}$.
\end{enumerate}
\end{proposition}
\begin{proof}
    We show that $D^{k}_{\Pi,E}(\mathfrak{I}_{E^-})$ is an over-deletion model of $\Pi,E,E^-$ during $[\varrho_1^- +\mathsf{depth}(\Pi),\varrho_4^+ - \mathsf{depth}(\Pi)]$.
    For $t\in [\varrho_1^- +\mathsf{depth}(\Pi),\varrho_4^+ - \mathsf{depth}(\Pi)]$, if there is a ground instance $r'$ of some rule $r\in\Pi$ such that $D^{k}_{\Pi,E}(\mathfrak{I}_{E^-}),t\models body(r')$ and $D^{k}_{\Pi,E}(\mathfrak{I}_{E^-})\not\models head(r')$, then there must be a punctual fact $f$ such that $r'$ derives $f$ from $\mathsf{rep}(D^{k}_{\Pi,E}(\mathfrak{I}_{E^-}))$ at $t$ and $D^{k}_{\Pi,E}(\mathfrak{I}_{E^-})\not\models f$. 
    Since $f$ is derived at $t$, let $f = M@t'$, then $t'\in[t-\mathsf{depth}(\Pi),t+\mathsf{depth}(\Pi)]$, and therefore $t\in[\varrho_1^-,\varrho_4^+]$, and we have $D^{k+1}_{\Pi,E}(\mathfrak{I}_{E^-})\models f$, and $D^{k}_{\Pi,E}(\mathfrak{I}_{E^-})\mid_{[\varrho_1^-,\varrho_4^+]}\neq D^{k+1}_{\Pi,E}(\mathfrak{I}_{E^-})\mid_{[\varrho_1^-,\varrho_4^+]}$, which is a contradiction to our assumption.
    Therefore, $D^{k}_{\Pi,E}(\mathfrak{I}_{E^-})$ is deletion-saturated regarding $\Pi,E,E^-$.
\end{proof}
\begin{definition}\label{claim:deleteunfold}
    The $(\varrho_{\mathit{left}},\varrho_{\mathit{right}})$-unfolding of deletion-saturated interpretation $D^{k}_{\Pi,E}(\mathfrak{I}_{E^-})$ is the interpretation $\mathfrak{I}$ such that
    \begin{itemize}
        \item $\mathfrak{I}\mid_{[\varrho^-_{\mathit{left}},\varrho^+_{\mathit{right}}]} =D^{k}_{\Pi,E}(\mathfrak{I}_{E^-})\mid_{[\varrho^-_{\mathit{left}},\varrho^+_{\mathit{right}}]} $
        \item $\forall n\in \mathbb{N}$,
        $\mathfrak{I}\mid_{\varrho_{\mathit{left}}-n|\varrho_{\mathit{left}}|}$ is a shift of $D^{k}_{\Pi,E}(\mathfrak{I}_{E^-})\mid_{\varrho_{\mathit{left}}}$,
        \item $\forall n\in \mathbb{N}$, $\mathfrak{I}\mid_{\varrho_{\mathit{right}}+n|\varrho_{\mathit{right}}|}$ is a shift of $D^{k}_{\Pi,E}(\mathfrak{I}_{E^-})\mid_{\varrho_{\mathit{right}}}$
\end{itemize}
\end{definition}

Next, we show that 
\begin{theorem}\label{theo:deletionunfold}
    For a deletion-saturated interpretation $D^k_{\Pi,E}(\mathfrak{I}_{E^-})$ and a pair of its deletion periods $\varrho_{\mathit{left}},\varrho_{\mathit{right}}$, the $(\varrho_{\mathit{left}},\varrho_{\mathit{right}})$-unfolding $\mathfrak{I}$ of $D^k_{\Pi,E}(\mathfrak{I}_{E^-})$ coincides with $\mathfrak{M}^-_{\Pi,E,E^-}$.
\end{theorem}
\begin{proof}
    First, we prove $$\mathfrak{M}^-_{\Pi,E,E^-}\subseteq\mathfrak{I}.$$ 
    Because $\mathfrak{M}^-_{\Pi,E,E^-}$ is the least over-deletion model of $\Pi,E,E^-$, we only need to show that $\mathfrak{I}$ is a model of $E^-$ and an over-deletion model of $\Pi$ regarding $\Pi,E,E^-$.

    Because $D^k_{\Pi,E}(\mathfrak{I}_{E^-})$ contains $\mathfrak{I}_{E^-}$, $D^k_{\Pi,E}(\mathfrak{I}_{E^-})$, $\varrho_\mathsf{left}^-<t_{E}^-$, $\varrho_\mathsf{right}^+>t_{E}^+$, and $\mathfrak{I}\mid_{[\varrho_\mathsf{left}^-,\varrho_\mathsf{{right}}^+]} =D^k_{\Pi,E}(\mathfrak{I}_{E^-})\mid_{[\varrho_\mathsf{left}^-,\varrho_\mathsf{{right}}^+]}$, $D^k_{\Pi,E}(\mathfrak{I}_{E^-}) $, we have $\mathfrak{I}_{[\varrho_\mathsf{left}^-,\varrho_\mathsf{{right}}^+]}$ is a model of $E^-$, and $\mathfrak{I}$ is a model of $E^-$.

    To show that $\mathfrak{I}$ is an over-deletion model of $\Pi$ regarding $\Pi,E,E^-$, we first show that $\mathfrak{I}$ satisfies with over-deletion of $E^-$ $\Pi$ regarding $\Pi,E$ during $\varrho=[\varrho_\mathsf{left}^-+\mathsf{depth}(\Pi),\varrho_\mathsf{right}^+ - \mathsf{depth}(\Pi)]$.
    This is the case because $\mathfrak{I}\mid_{[\varrho_\mathsf{left}^-,\varrho_\mathsf{{right}}^+]} = D^k_{\Pi,E}(\mathfrak{I}_{E^-})\mid_{[\varrho_\mathsf{left}^-,\varrho_\mathsf{{right}}^+]}$, and by the definition of a deletion-saturated interpretation, $D^k_{\Pi,E}(\mathfrak{I}_{E^-})$ satisfies with over-deletion of $E^-$ $\Pi$ regarding $\Pi,E$ during $[\varrho_\mathsf{left}^-+\mathsf{depth}(\Pi),\varrho_\mathsf{right}^+ - \mathsf{depth}(\Pi)]$.
    Then we show that $\forall t \notin\varrho$, $\mathfrak{I}$ satisfies with over-deletion of $E^-$ $\Pi$ regarding $\Pi,E$ at $t$.
    Suppose $t<\varrho^-$, then $t< \varrho_\mathsf{left}^-+\mathsf{depth}(\Pi)$. 
    Then there must a natural number $n$ and and a time point $t'\in [\varrho_\mathsf{left}^-+\mathsf{depth}(\Pi),\varrho_\mathsf{left}^+ +\mathsf{depth}(\Pi)]$ such that $t' = t+n\cdot|\varrho_\mathsf{left}|$.
    By the definition of $\mathfrak{I}$, we know that $\mathfrak{I}\mid_{[t-\mathsf{depth}(\Pi),t+\mathsf{depth}(\Pi)]}$ is a shift of $\mathfrak{I}\mid_{[t'-\mathsf{depth}(\Pi),t'+\mathsf{depth}(\Pi)]}$, and also $\mathfrak{C}_{\Pi,E}\mid_{[t-\mathsf{depth}(\Pi),t+\mathsf{depth}(\Pi)]}$ is a shift of $\mathfrak{C}_{\Pi,E}\mid_{[t'-\mathsf{depth}(\Pi),t'+\mathsf{depth}(\Pi)]}$. 
    Therefore, for any rule instance $r'$ of a rule $r\in\Pi$, if it is satisfied with over-deletion of $E^-$ by $\mathfrak{I}$ at $t'$, then it is satisfied with over-deletion of $E^-$ by $\mathfrak{I}$ at $t$, and vice versa, and we have $\mathfrak{I}$ satisfies with over-deletion of $E^-$ $\Pi$ regarding $\Pi,E$ at $t$ because $t'\in\varrho$. 
    The proof for $t>\varrho^+$ is symmetric.
    Consequently, $\mathfrak{M}^-_{\Pi,E,E^-}\subseteq\mathfrak{I}$
    
    Next, we prove $$\mathfrak{I}\subseteq\mathfrak{M}^-_{\Pi,E,E^-}.$$
    Because $D^k_{\Pi,E}(\mathfrak{I}_{E^-})\subseteq D^{\omega_1}_{\Pi,E}(\mathfrak{I}_{E^-}) = \mathfrak{M}^-_{\Pi,E,E^-}$ and $\mathfrak{I}\mid_{[\varrho_\mathsf{left}^-,\varrho_\mathsf{{right}}^+]} = D^k_{\Pi,E}(\mathfrak{I}_{E^-})\mid_{[\varrho_\mathsf{left}^-,\varrho_\mathsf{{right}}^+]}$, we have $\mathfrak{I}\mid_{[\varrho_\mathsf{left}^-,\varrho_\mathsf{{right}}^+]}\subseteq\mathfrak{M}^-_{\Pi,E,E^-}$.
    Consider $\mathfrak{I}\mid_{[\varrho_\mathsf{right}^+,\infty)}$, we show inductively that for each $t_i$ on the $(\Pi,E)$-ruler that $\mathfrak{I}\mid_{[\varrho_\mathsf{left}^-,t_i]}\subseteq\mathfrak{M}^-_{\Pi,E,E^-}\mid_{[\varrho_\mathsf{left}^-,t_i]} $, where $t_0 = \varrho_\mathsf{right}^+$.
    The base case is already proven.
    For the inductive step, let $\varrho_A = (t_{i-1},t_i]$, and $\varrho_B = [t_{i-1}-\mathsf{depth}(\Pi),t_{i-1}]$, $\varrho_A' = \varrho_A - |\varrho_\mathsf{right}|$, $\varrho_B' = \varrho_B - |\varrho_\mathsf{right}|$. 
    We only need to show that $\mathfrak{I}\mid_{\varrho_A}\subseteq \mathfrak{M}^-_{\Pi,E,E^-}\mid_{\varrho_A}$. 
    Notice that by definition of $\mathfrak{I}$, $\mathfrak{I}\mid_{\varrho_B}$ is a shift of $\mathfrak{I}\mid_{\varrho_B'}$, and by the induction hypothesis and $\mathfrak{M}^-_{\Pi,E,E^-}\subseteq\mathfrak{I}$, we have $\mathfrak{M}^-_{\Pi,E,E^-}\mid_{\varrho_B'}$ is a shift of $\mathfrak{M}^-_{\Pi,E,E^-}\mid_{\varrho_B}$.
    Since by definition of $\mathfrak{I}$, $\mathfrak{C}_{\Pi,E}\mid_{[\varrho_B'^-,\infty)}$ is a shift of $\mathfrak{C}_{\Pi,E}\mid_{[\varrho_B^-,\infty)}$ and $|\varrho_B'|=|\varrho_B|=\mathsf{depth}(\Pi)$, by Lemma~\ref{lemma:depthrep} we can show that $\mathfrak{M}^-_{\Pi,E,E^-}\mid_{\varrho_A}$ is a shift of $\mathfrak{M}^-_{\Pi,E,E^-}\mid_{\varrho_A'}$, and because $\mathfrak{I}\mid_{\varrho_A'}$ is a shift of $\mathfrak{I}\mid_{\varrho_A}$, and $\mathfrak{I}\mid_{\varrho_A'}\subseteq\mathfrak{M}^-_{\Pi,E,E^-}\mid_{\varrho_A'}$, we have $\mathfrak{I}\mid_{\varrho_A}\subseteq\mathfrak{M}^-_{\Pi,E,E^-}\mid_{\varrho_A}$. 
    The case for $\mathfrak{I}\mid_{(-\infty,\varrho_\mathsf{right}^+)}$ is symmetric.
    Therefore, $\mathfrak{I}\subseteq\mathfrak{M}^-_{\Pi,E,E^-}$.
\end{proof}

Given a DatalogMTL program $\Pi$, dataset $E$, we define the \textbf{semi-na\"ive immediate deletion-consequence operator} $\dot{D}_{\Pi,E}$, 
mapping two interpretations $\mathfrak{I}_1,\mathfrak{I}_2$ into the least interpretation $\dot{D}_{\Pi,E}(\mathfrak{I}_1,\mathfrak{I}_2)$ containing $\mathfrak{I}_2$ satisfying for each ground rule instance $r'$ of a rule in $\Pi$, a fact $f$,
a dataset $\mathcal{D}$ such that if $r'$ minimally derives $f$ from $\mathcal{D}$, $\mathfrak{C}_{\Pi,E}-\mathfrak{I}_2\models \mathcal{D}$, 
and $(\mathfrak{I}_2-\mathfrak{I}_1)\mathfrak{}\cap\mathfrak{I}_{\mathcal{D}}\neq\mathfrak{I}_{\emptyset}$, then $\dot{D}_{\Pi,E}(\mathfrak{I}_1,\mathfrak{I}_2)\models f$. 

For a given $E^-$ such that $\mathfrak{I}_{E^-}\subseteq \mathfrak{I}_{E}$, successive applications of $\dot{D}_{\Pi,E}$ to $\mathfrak{I}_{E^-}$ yield a transfinite sequence of interpretations, $\dot{\mathfrak{I}_0},\dot{\mathfrak{I}_1},\dot{\mathfrak{I}_2},....$ such that 
\begin{enumerate*}[label=(\roman*)]
    \item $\dot{\mathfrak{I}}_0 = \mathfrak{I}_{E^-}$,
    \item $\dot{\mathfrak{I}}_1 = \dot{D}_{\Pi,E}(\mathfrak{I}_{\emptyset},\mathfrak{I}_{0})$
    \item $\dot{\mathfrak{I}}_{\alpha+1} = \dot{D}_{\Pi,E}(\mathfrak{I}_{\alpha-1},\mathfrak{I}_{\alpha})$, for $\alpha$ a successor ordinal, and
    \item $\dot{\mathfrak{I}}_\beta = \bigcup_{\alpha<\beta}\dot{\mathfrak{I}}_\alpha$, for $\beta$ a limit ordinal.
\end{enumerate*}

We show that the sequence defined by the semi-na\"ive immediate deletion-consequence operator for $\mathfrak{I}_{E^-}$ is identical to the sequence defined by the immediate deletion-consequence operator for $\mathfrak{I}_{E^-}$.
\begin{proposition}\label{proposition:deletionseminaive=naive}
For each ordinal $\alpha$,
    $$\dot{\mathfrak{I}}_\alpha = D^{\alpha}_{\Pi,E}(\mathfrak{I}_{E^-})$$
\end{proposition}
\begin{proof}
    We prove the statement by induction on $\alpha$.

    When $\alpha = 0$, $\dot{\mathfrak{I}}_0 = D^{0}_{\Pi,E}(\mathfrak{I}_{E^-}) = \mathfrak{I}_{E^-}$

    When $\alpha = 1$, because $\dot{D}_{\Pi,E}(\mathfrak{I}_\emptyset,*)$ is the same operator as $D_{\Pi,E}(*)$, we have $\dot{\mathfrak{I}}_1 = D^{1}_{\Pi,E}(\mathfrak{I}_{E^-})$

    For the inductive step, we first show that $\dot{\mathfrak{I}}_{\alpha+1} \subseteq D^{\alpha+1}_{\Pi,E}(\mathfrak{I}_{E^-})$.
    By definition $\dot{\mathfrak{I}}_{\alpha+1} = \dot
    {D}_{\Pi,E}(\dot{\mathfrak{I}}_{\alpha-1},\dot{\mathfrak{I}}_{\alpha})$.
    Consider each fact $f=M@t$ such that there is a ground rule instance $r'$ of a rule $r\in\Pi$, a dataset $\mathcal{D}$ such that $r'$ minimally derives $f$ from $\mathcal{D}$, $\mathfrak{C}_{\Pi,E}-\dot{\mathfrak{I}}_{\alpha-1}\models \mathcal{D}$ and $(\dot{\mathfrak{I}}_{\alpha}-\dot{\mathfrak{I}}_{\alpha-1})\cap\mathfrak{I}_\mathcal{D}\neq \mathfrak{I}_\emptyset $. 
    Then we must have that $\mathfrak{C}_{\Pi,E}\models \mathcal{D}$, and by the induction hypothesis, $\dot
    {\mathfrak{I}}_{\alpha}= D^{\alpha}_{\Pi,E}(\mathfrak{I}_{E^-})$, which means that $D^{\alpha}_{\Pi,E}(\mathfrak{I}_{E^-})\cap\mathfrak{I}_\mathcal{D}\neq \mathfrak{I}_\emptyset $, and since $r'$ minimally derives $f$ from $\mathcal{D}$, by definition of $D_{\Pi,E}$, we have $D^{\alpha+1}_{\Pi,E}(\mathfrak{I}_{E^-})\models f$.
    Therefore, $\dot{\mathfrak{I}}_{\alpha+1}\subseteq D^{\alpha+1}_{\Pi,E}(\mathfrak{I}_{E^-})$.

    Next, we show that $D^{\alpha+1}_{\Pi,E}(\mathfrak{I}_{E^-})\subseteq\dot{\mathfrak{I}}_{\alpha+1} $.
    Consider each fact $f=M@t$ such that there is a ground rule instance $r'$ of a rule $r\in\Pi$, a dataset $\mathcal{D}$ such that $r'$ minimally derives $f$ from $\mathcal{D}$, $\mathfrak{C}_{\Pi,E}\models \mathcal{D}$, and $\mathfrak{I}\cap D^{\alpha}_{\Pi,E}(\mathfrak{I}_{E^-})\neq \mathfrak{I}_{\emptyset}$.
    If $D^{\alpha}_{\Pi,E}(\mathfrak{I}_{E^-})\models f$ then by the induction hypothesis $\dot{\mathfrak{I}}_\alpha\models f$, and $\dot{\mathfrak{I}}_{\alpha+1}\models f$.
    If $D^{\alpha}_{\Pi,E}(\mathfrak{I}_{E^-})\not\models f$, then let $\mathcal{D}_1'$ represent $\mathfrak{I}_{\mathcal{D}}\cap D^{\alpha}_{\Pi,E}(\mathfrak{I}_{E^-})$ and $\mathcal{D}_2'$ represent $\mathfrak{I}_{\mathcal{D}}- D^{\alpha}_{\Pi,E}(\mathfrak{I}_{E^-})$, 
    and let $\mathcal{D}' = \mathcal{D}_1'\cup \mathcal{D}_2'$, then $\mathfrak{C}_{\Pi,E}\models \mathcal{D}'$, $\mathfrak{C}_{\Pi,E} - \dot{\mathfrak{I}}_{\alpha-1}\models \mathcal{D}_2'$, and $r'$ minimally derives $f$ from $\mathcal{D}'$.
    We show that $\mathfrak{C}_{\Pi,E}-\dot{\mathfrak{I}}_{\alpha-1}\models \mathcal{D}_1'$ and $\dot{\mathfrak{I}}_{\alpha} - \dot{\mathfrak{I}}_{\alpha-1}\cap \mathfrak{I}_{\mathcal{D}_1'}\neq \emptyset$.
    To prove both, we only need to show that $\forall f'\in\mathcal{D}_1'$, $D_{\Pi,E}^\alpha(\mathfrak{I}_{E^-})\models f'$ and $D_{\Pi,E}^{\alpha-1}(\mathfrak{I}_{E^-})\not\models f'$.
    We have $D_{\Pi,E}^\alpha(\mathfrak{I}_{E^-})\models f'$ because $\mathfrak{I}_{\mathcal{D}_1'}\subseteq D_{\Pi,E}^\alpha(\mathfrak{I}_{E^-})$.
    To show $D_{\Pi,E}^{\alpha-1}(\mathfrak{I}_{E^-})\not\models f'$, consider if $D_{\Pi,E}^{\alpha-1}(\mathfrak{I}_{E^-})\models f'$, then $\mathfrak{I}_{\mathcal{D}'}\cap D_{\Pi,E}^{\alpha-1}(\mathfrak{I}_{E^-})\neq \mathfrak{I}_\emptyset$, and since $\mathfrak{C}_{\Pi,E}\models \mathcal{D}'$, $r'$ minimally derives $f$ from $\mathcal{D}'$, we have that $D_{\Pi,E}^{\alpha}(\mathfrak{I}_{E^-})\models f$, which contradicts $D^{\alpha}_{\Pi,E}(\mathfrak{I}_{E^-})\not\models f$.
    Therefore, $\mathfrak{C}_{\Pi,E}-\dot{\mathfrak{I}}_{\alpha-1}\models \mathcal{D}_1'$ and $\dot{\mathfrak{I}}_{\alpha} - \dot{\mathfrak{I}}_{\alpha-1}\cap \mathfrak{I}_{\mathcal{D}_1'}\neq \emptyset$, $\mathfrak{C}_{\Pi,E}-\dot{\mathfrak{I}}_{\alpha-1}\models \mathcal{D}'$, and by definition of $\dot{D}_{\Pi,E}$, $\dot{\mathfrak{I}}_{\alpha+1}\models f$
    Consequently, $\dot{\mathfrak{I}}_{\alpha+1}\models f$, and $D^{\alpha+1}_{\Pi,E}(\mathfrak{I}_{E^-})\subseteq\dot{\mathfrak{I}}_{\alpha+1} $.

    When $\alpha=\beta$ a limit ordinal, the equivalence holds trivially because $D^{\beta}_{\Pi,E}(\mathfrak{I}_{E^-}) = \bigcup_{\gamma<\beta} D^{\gamma}_{\Pi,E}(\mathfrak{I}_{E^-})$ and $\dot{\mathfrak{I}}_{\beta} = \bigcup_{\gamma<\beta}\dot{\mathfrak{I}}_\gamma$
\end{proof}

A deletion of $\mathfrak{M}_{\Pi,E,E^-}^-$ from $\mathfrak{C}_{\Pi,E}$ will remove every fact that is partially rooted in $E^-$ from $\mathfrak{C}_{\Pi,E}$.

Let $\mathfrak{L}_{\Pi,E,E^-} = \mathfrak{C}_{\Pi,E}-\mathfrak{M}_{\Pi,E,E^-}$, then
\begin{proposition}\label{proposition:deletioncomplete}
     If a fact $f$ is partially rooted in $E^-$ regarding $\Pi,E$, then $\mathfrak{L}_{\Pi,E,E^-}\not\models f$.
\end{proposition}
\begin{proof}
    By the definition of $\mathfrak{M}_{\Pi,E,E^-}$, $\mathfrak{M}_{\Pi,E,E^-}\models f$.
    If $\mathfrak{L}_{\Pi,E,E^-}\models f$, then $\mathfrak{M}_{\Pi,E,E^-}\cap\mathfrak{L}_{\Pi,E,E^-}\neq \mathfrak{I}_\emptyset$, and we arrive at a contradiction.
\end{proof}

\begin{proposition}\label{proposition:remainderinsmallcanon}
    $\mathfrak{L}_{\Pi,E,E^-}\subseteq \mathfrak{C}_{\Pi,E\setminus E^-}$
\end{proposition}
\begin{proof}
    For each fact $f$ such that $\mathfrak{L}_{\Pi,E,E^-}\models f$, $\mathfrak{D}_{\Pi,E}(f)\neq\emptyset$. 
    Let $T\in \mathfrak{D}_{\Pi,E}(f),T = (V,\mathcal{E})$, then $\forall v\in \{v\in V\mid \forall (v_1,v_2)\in \mathcal{E}, v_2\neq v\}$, if $\mathfrak{I}_{\{l(v)\}}\cap \mathfrak{I}_{E^-}\neq \emptyset$, then $f$ is partially rooted in $E^-$, which means $\mathfrak{L}_{\Pi,E,E^-}\not\models f$, and we arrive at a contradiction.
    Therefore, $\mathfrak{I}_{\{l(v)\}}\cap \mathfrak{I}_{E^-}= \emptyset$, and since $\mathfrak{I}_{E}\models l(v)$, $\mathfrak{I}_{E\setminus E^-}\models l(v)$, and by Proposition~\ref{proposition:leafrange}, $\mathfrak{C}_{\Pi,E\setminus E^-} \models f$.
\end{proof}
Since by definition, A deletion-saturated interpretation $D^k_{\Pi,E}(\mathfrak{I}_{E^-})$ shares the same periods with a saturated interpretation of $\Pi,E$,  from $\mathfrak{C}_{\Pi,E}- \mathfrak{M}_{\Pi,E,E^-}^-$ must also be periodic, meaning that the same $\varrho_1,\varrho_2,\varrho_3$ and $\varrho_4$ of length $2\cdot\mathsf{depth}(\Pi)$, which were used to identify both the periods of the saturated interpretation and the periods of $D^k_{\Pi,E}(\mathfrak{I}_{E^-})$, satisfy that 
\begin{itemize}
    \item $\mathfrak{L}_{\Pi,E,E^-}\mid_{\varrho_1}$ is a shift of $\mathfrak{L}_{\Pi,E,E^-}\mid_{\varrho_2}$
    \item $\mathfrak{L}_{\Pi,E,E^-}\mid_{\varrho_3}$ is a shift of $\mathfrak{L}_{\Pi,E,E^-}\mid_{\varrho_4}$
\end{itemize}
Let $\varrho_\mathsf{left} = [\varrho_1^-,\varrho_2^-)$, $\varrho_\mathsf{right} = (\varrho_3^+,\varrho_4^+]$, then there must be an interpretation $\mathfrak{L}$, such that
\begin{itemize}
    \item $\mathfrak{L}\mid_{[\varrho_{\mathsf{left}}^-,\varrho_{\mathsf{right}}^+]} = \mathfrak{L}_{\Pi,E,E^-}\mid_{[\varrho_{\mathsf{left}}^-,\varrho_{\mathsf{right}}^+]}$
    \item $\mathfrak{L}\mid_{\varrho_1}$ is a shift of $\mathfrak{L}\mid_{\varrho_2}$
    \item $\mathfrak{L}\mid_{\varrho_3}$ is a shift of $\mathfrak{L}\mid_{\varrho_4}$
\end{itemize}
And obviously, the $(\varrho_\mathsf{left},\varrho_\mathsf{right})$-unfolding of $\mathfrak{L}$ coincides with $\mathfrak{L}_{\Pi,E,E^-}$. 
For each such $\varrho_1,\varrho_2,\varrho_3$ and $\varrho_4$ satisfying the above, we say $(\varrho_\mathsf{left},\varrho_\mathsf{right})$ are the \textbf{periods} of $\mathfrak{L}_{\Pi,E,E^-}$

\begin{proposition} \label{proposition:deletion-saturated-align}
Given a periodic materialisation $\mathds{I}$ such that $\mathsf{unfold}(\mathds{I}) =\mathfrak{C}_{\Pi,E}$, there exists $k \leq k_{\mathsf{max}}$ such that $D^k_{\Pi,E}(\mathfrak{I}_{E^-})$ 
is deletion-saturated, where the bound $k_{max}$ is defined as follows.
Let $A$ be the number of ground relational atoms in the
grounding of $\Pi$ with constants from $\Pi$ and $E$, let $B$ be the
number of $(\Pi, E)$-intervals within $[t_E^-, t_E^-
 + 2depth(\Pi)]$, let $C$ be the max number of $(\Pi,E)$-endpoints contained by $\mathds{I}.\varrho_{\mathsf{L}}$ or $\mathds{I}.\varrho_{\mathsf{R}}$,
and let $\varrho$ = $[t_{-w} - 2C\mathsf{depth}(\Pi), t_w + 2C\mathsf{depth}(\Pi)]$, where $\cdots < t_{-2}<t_{-1}<t_{E}^-$ and $t_E^+<t_1<t_2<\cdots$ are sequences of consecutive time points on the $(\Pi, E)$-ruler, while
$w$ = $C +B ·C· (2^A)^B$. 
Then $k_\mathsf{max}$ is the product of $A$ and the
number of $(\Pi, E)$-intervals contained in $\varrho$.
\end{proposition}
\begin{proof}
We prove that $D^k_{\Pi,E} (\mathfrak{I}_{E^-})$ satisfies the conditions outlined in Proposition~\ref{proposition:deletion-saturated} and is thus deletion-saturated. There are at most $A$ facts in each ${(\Pi, E)}$-interval contained in $\varrho$, and so by the definition of $k_\mathsf{max}$, interpretation $D^k_{\Pi,E} (\mathfrak{I}_{E^-})$ satisfies at most $k_\mathsf{max}$ facts over these ${(\Pi, E)}$-intervals in $\varrho$. Since each application of $D^k_{\Pi,E}$ introduces at least one new relational fact over a ${(\Pi, E)}$-interval before fixpoint is reached, there must exist $k$ with $k \leq k_\mathsf{max}$ such that ${D^{k+1}_{\Pi,E} (\mathfrak{I}_{E^-})|_{\varrho} = D^{k}_{\Pi,E} (\mathfrak{I}_{E^-})|_{\varrho}}$ holds. We show that $D^k_{\Pi,E} (\mathfrak{I}_{E^-})$ is deletion-saturated for such $k$ by proving that $\varrho$ contains $\varrho_1$, $\varrho_2$, $\varrho_3$, and $\varrho_4$ such that conditions~1--4 in Proposition~\ref{proposition:deletion-saturated} hold.

Consider first the left-hand side of the interval $\varrho$. Let ${w' = 1 + B·(2^A)^B}$. We argue that there exist two distinct intervals $\varrho'_1$ and $\varrho'_2$ to the right of $t_{-w'}- 2\mathsf{depth}(\Pi)$ and to the left of $t'_E$ such that ${D^{k}_{\Pi,E} (\mathfrak{I}_{E^-})|_{\varrho'_1}}$ is a shift of ${D^{k}_{\Pi,E} (\mathfrak{I}_{E^-})|_{\varrho'_2}}$: each interval of the form
$[t, t + 2\mathsf{depth}(\Pi)]$, for $t$ on the $(\Pi, E)$-ruler, contains the same number of $(\Pi, E)$-intervals as $[t_E^-, t_E^- + 2\mathsf{depth}(\Pi)]$ does, which equals $B$; in each of these $(\Pi, E)$-intervals there can hold at most $2^A$ combinations of relational atoms, which gives rise to $(2^A)^B$ different contents of an interval of the form $[t, t + 2\mathsf{depth}(\Pi)]$; additionally, these intervals can differ depending on the location of $(\Pi, E)$-intervals they contain; by the definition of the $(\Pi, E)$-ruler, there are at most $B$ different layouts of $(\Pi, E)$-intervals contained in an interval of the form $[t, t + 2\mathsf{depth}(\Pi)]$; hence,
in total, there are at most $B · (2^A)^B$ intervals of the form
$[t, t + 2\mathsf{depth}(\Pi)]$ with different contents; however, inside $[t_{-w'} - 2\mathsf{depth}(\Pi), t_E^-)$ there are $w' = 1 + B · (2^A)^B$ such intervals, so $\varrho'_1$ and $\varrho'_2$ can surely be found. Following the same reasoning, intervals $\varrho'_3$ and $\varrho'_4$ can be found inside $[t_{-w'} - 2\mathsf{depth}(\Pi), t_E^-)$ such that the model ${\mathfrak{C}_{\Pi,E}}$ satisfies ${\mathfrak{C}_{\Pi,E}|_{\varrho'_3} = \mathfrak{C}_{\Pi,E}|_{\varrho'_4}}$. Without loss of generality, assume that ${\varrho'_1}$ is to the left of ${\varrho'_2}$, and that ${\varrho'_3}$ is to the left of ${\varrho'_4}$, we will be able to extend the periodic region of ${[\varrho'_1, \varrho'_2)}$ to the size of $|C||[\varrho'_1, \varrho'_2)|$ and align ${[\varrho'_3, \varrho'_4)}$ with the extended interval; either operation would clearly not go beyond $t_{-w} - 2\mathsf{depth}(\Pi)$, and we will identify $\varrho_1$ and $\varrho_2$ of length ${2\mathsf{depth}(\Pi)}$ such that condition 3 of Proposition~\ref{proposition:deletion-saturated} is satisfied. Analogously, $\varrho_3$ and $\varrho_4$ can be found such that condition 4 of Proposition~\ref{proposition:deletion-saturated} is satisfied. Moreover, conditions~1--2 clearly hold as well at this point. This completes our proof for Proposition~\ref{proposition:deletion-saturated-align}.
\end{proof}

\subsection{Properties about Rederivation}
In the following subsection, we theorise rederivation.

Given two interpretations $\mathfrak{I}_1$ and $\mathfrak{I}_2$ and a program $\Pi$, we define the \textbf{immediate insertion-consequence operator} $I_{\Pi,\mathfrak{I}_1}$,
mapping $\mathfrak{I}_2$ into the least interpretation $I_{\Pi,\mathfrak{I}_1}(\mathfrak{I}_2)$ containing $\mathfrak{I}_2$ satisfying for each ground rule instance $r'$ of a rule $r\in\Pi$, a fact $f$, a dataset $\mathcal{D}$ such that if $r'$ minimally derives $f$ from $\mathcal{D}$, $\mathfrak{I}_1\cup\mathfrak{I}_2\models \mathcal{D}$ and $\mathfrak{I}_2\cap\mathfrak{I}_{\mathcal{D}}\neq \mathfrak{I}_\emptyset$, then $I_{\Pi,\mathfrak{I}_1}(\mathfrak{I}_2)\models f$.

Given $\mathfrak{L}_{\Pi,E,E^-}$, $\varrho_\mathsf{left}$, and $\varrho_\mathsf{right}$, for an ordinal $\alpha<\omega_1$, let $\mathfrak{I}^p_\alpha=(\mathfrak{M}_{\Pi,E,E^-}^- \cap \mathfrak{L}_{\Pi,E,E^-}\mid_{[\varrho_\mathsf{left}^--\alpha\cdot|\varrho_\mathsf{left}|,\varrho_\mathsf{right}^++\alpha\cdot|\varrho_\mathsf{right}|]})$. 
We define a transfinite sequence of interpretations $\mathfrak{R}_0,\mathfrak{R}_1,\dots$ as follows
\begin{enumerate*}[label=(\roman*)]
    \item $\mathfrak{R}_0 = \mathfrak{M}_{\Pi,E,E^-}^- \cap(\mathfrak{I}_{E\setminus E^-}\cup \mathfrak{I}^p_0) $,
    \item  
     $\mathfrak{R}_{\alpha+1} = I_{\Pi,\mathfrak{L}_{\Pi,E,E^-}}(\mathfrak{R}_{\alpha}\cup \mathfrak{I}^p_{\alpha+1})\cup \mathfrak{I}^p_{\alpha+1}$, for $\alpha$ a successor ordinal, and
    \item $\mathfrak{R}_\beta = \bigcup_{\alpha<\beta}\mathfrak{R}_\alpha$, for $\beta$ a limit ordinal.
\end{enumerate*}

\begin{proposition}\label{proposition:rederivationruler}
    For each $\alpha\leq\omega_1,\mathfrak{R}_\alpha$ is a $(\Pi,E\cup E^+)$-interpretation.
\end{proposition}

\begin{proof}
    We prove this by induction.

    When $\alpha = 0$, 
    \begin{align*}
        &\mathfrak{R}_0 = \mathfrak{M}_{\Pi,E,E^-}^- \cap (\mathfrak{I}_{E\setminus E^-} \cup \mathfrak{L}_{\Pi,E,E^-})\\
        =&\mathfrak{M}_{\Pi,E,E^-}^- \cap (\mathfrak{I}_{E\setminus E^-} \cup (\mathfrak{C}_{\Pi,E}-\mathfrak{M}_{\Pi,E,E^-}))\\
        =&\mathfrak{M}_{\Pi,E,E^-}^- \cap \mathfrak{I}_{E\setminus E^-}
    \end{align*}
    By Proposition~\ref{proposition:overdeletionleast}, we have that $\mathfrak{M}_{\Pi,E,E^-}^-= D_{\Pi,E}^{\beta}(\mathfrak{I}_{E^-})$ for some $\beta$. And $D_{\Pi,E}^{\beta}(\mathfrak{I}_{E^-})$ is a $(\Pi,E\cup E^+)$-interpretation by Proposition~\ref{proposition:overdeletionruler}. Obviously, $\mathfrak{I}_{E\setminus E^-}$ is a $(\Pi,E\cup E^+)$-interpretation. Then we have that $\mathfrak{R}_0$ is a $(\Pi,E\cup E^+)$-interpretation.

    For the induction step, consider a fact $f=M@t$ such that $\mathfrak{R}_{\alpha+1}\models f$ and $\mathfrak{R}_{\alpha+1}\not\models f$. 
    If $t$ is on the $(\Pi,\mathcal{D})$-ruler, then $\mathfrak{R}_{\alpha+1}\models M@t$ already implies $\mathfrak{R}_{\alpha+1}\models M@\varrho$ for the $(\Pi,\mathcal{D})$-interval $\varrho$ containing $t$, which is $t$.
    If $t$ is not on the $(\Pi,\mathcal{D})$-ruler, we proof by classifying cases where $f$ belongs to. When $\mathfrak{I}^p_{\alpha+1}\models f$, it is trivial since $\mathfrak{I}^p_{\alpha+1}=(\mathfrak{M}_{\Pi,E,E^-}^- \cap \mathfrak{L}_{\Pi,E,E^-}\mid_{[\varrho_\mathsf{left}^--\alpha\cdot|\varrho_\mathsf{left}|,\varrho_\mathsf{right}^++\alpha\cdot|\varrho_\mathsf{right}|]})$ and both parts of the right side are $(\Pi,E\cup E^+)$-interpretation. If $I_{\Pi,\mathfrak{L}_{\Pi,E,E^-}}(\mathfrak{R}_{\alpha}\cup \mathfrak{I}^p_{\alpha+1})\models f$, then there must be a rule $r'$, a dataset $\mathcal{D}$ and a time point $t'$ such that $r'$ minimally derives $f$ at $t'$, $\mathfrak{L}_{\Pi,E,E^-}\cup \mathfrak{R}_\alpha \cup\mathfrak{I}_\alpha^p\models\mathcal{D}$ and $(\mathfrak{R}_\alpha \cup\mathfrak{I}_\alpha^p)\cap\mathfrak{I}_\mathcal{D}\not\models\mathfrak{I}_\emptyset$. Let $\varrho'$ be the $(\Pi,E)$-interval containing $t'$. Then by Proposition~\ref{proposition:pi,dinterpretationatomsat}, $\forall t''\in\varrho'$, $\mathfrak{L}_{\Pi,E,E^-}\cup \mathfrak{R}_\alpha \cup\mathfrak{I}_\alpha^p,t'\models body(r')$. Because $\mathfrak{I}_{\mathcal{D}}-(\mathfrak{R}_\alpha \cup\mathfrak{I}_\alpha^p),t'\not\models body(r')$, and by Proposition~\ref{proposition:pi,dinterpretationatomnotsat}, $\forall t''\in \varrho', \mathfrak{I}_{\mathcal{D}}-(\mathfrak{R}_\alpha \cup\mathfrak{I}_\alpha^p),t''\not\models body(r')$. And therefore, $\forall t''\in\varrho'$, $\exists \mathcal{D}$ such that $r'$ minimally derives $M@t+t''-t'$ from $\mathcal{D'}$ at $t''$, $\mathfrak{L}_{\Pi,E,E^-}\cup \mathfrak{R}_\alpha \cup\mathfrak{I}_\alpha^p\models\mathcal{D'}$, $\mathfrak{I}_{\mathcal{D'}}$ is a shift of $\mathfrak{I}_{\mathcal{D}}$, and $\mathfrak{I}_{\mathcal{D'}}\cap(\mathfrak{R}_\alpha \cup\mathfrak{I}_\alpha^p)\not\models\mathfrak{I}_\emptyset$, which is because $\mathfrak{I}_{\mathcal{D'}}-(\mathfrak{R}_\alpha \cup\mathfrak{I}_\alpha^p),t''\not\models body(r')$.

    Let $\mathfrak{I}$ be the lest interpretation that satisfies $head(r')$ at $t'$. By Proposition~\ref{proposition:headexpandsion}, for some $\varrho$ such that $\varrho^-=t'+t_1, \varrho^+=t'+t_2$, where $t_1$ and $t_2$ are multiples of $div(\Pi)$, $\mathfrak{I}$ maps exactly each $t''\in\varrho$ to ${M}$. If $t_1=t_2$, then it must be so that $t=t'+t_1$, and by Proposition~\ref{proposition:timepointsdiv(pi)apartcontainedbyintervalsofequallength}, $\lvert\varrho'\rvert=\lvert\varrho\rvert$, which means $\forall t''\in\varrho$, there $\exists\mathcal{D'}$ such that $r'$ minimally derives $M@t''$ from $\mathcal{D'}$ at $t''-t_1, \mathfrak{L}_{\Pi,E,E^-}\cup \mathfrak{R}_\alpha \cup\mathfrak{I}_\alpha^p\models\mathcal{D'}$, and $\mathfrak{I}_{\mathcal{D'}}\cap(\mathfrak{R}_\alpha \cup\mathfrak{I}_\alpha^p)\not\models\mathfrak{I}_\emptyset$.

    If $t_1\neq t_2$, then because $\forall t'' \in \varrho'$, there $\exists\mathcal{D'}$ such that $r'$ minimally derives $M@t+t''-t'$ from $\mathcal{D'}$ at $t''$, $\mathfrak{L}_{\Pi,E,E^-}\cup \mathfrak{R}_\alpha \cup\mathfrak{I}_\alpha^p\models\mathcal{D'}$, and $\mathfrak{I}_{\mathcal{D'}}\cap(\mathfrak{R}_\alpha \cup\mathfrak{I}_\alpha^p)\not\models\mathfrak{I}_\emptyset$, $\forall t''\in\varrho'$ there $\exists \mathcal{D'}$ such that $r'$ minimally derives $M@(t''+t_1,t''+t_2)$ from $\mathcal{D'}$ at $t''$, $\mathfrak{L}_{\Pi,E,E^-}\cup \mathfrak{R}_\alpha \cup\mathfrak{I}_\alpha^p\models\mathcal{D'}$, and $\mathfrak{I}_{\mathcal{D'}}\cap(\mathfrak{R}_\alpha \cup\mathfrak{I}_\alpha^p)\not\models\mathfrak{I}_\emptyset$. Because $t'\in\varrho',t\in(\varrho'+t_1,\varrho'+t_2)$, and because $\varrho'+t_1, \varrho'+t_2$ are both time points on the $(\Pi,E)$-ruler, $(\varrho'+t_1,\varrho'+t_2)$ contains $\varrho$, and we have $\mathfrak{R}_{\alpha+1}\models M@t \implies \mathfrak{R}_{\alpha+1}\models M@\varrho$

\end{proof}

Interpretation $\mathfrak{R}_{\omega_1}$ is what we rederive.
First, we note that the rederivation is sound.
\begin{proposition}\label{proposition:rederivationsound}
    $\mathfrak{R}_{\omega_1}\subseteq \mathfrak{C}_{\Pi,E\setminus E^-}$.
\end{proposition}
\begin{proof}
    Because $\mathfrak{R}_{\omega_1} = \bigcup_{\alpha<\omega_1}\mathfrak{R}_\alpha$,
    we only need to show inductively that for each $\alpha<\omega_1$ $\mathfrak{R}^\alpha \subseteq\mathfrak{C}_{\Pi,E\setminus E^-}$.

    When $\alpha=0$, because $\mathfrak{I}_{E\setminus E^-}\subseteq \mathfrak{C}_{\Pi,E\setminus E^-}$ and  $\mathfrak{I}^p_\alpha=(\mathfrak{M}_{\Pi,E,E^-}^- \cap \mathfrak{L}_{\Pi,E,E^-}\mid_{[\varrho_\mathsf{left}^--\alpha\cdot|\varrho_\mathsf{left}|,\varrho_\mathsf{right}^++\alpha\cdot|\varrho_\mathsf{right}|]})$, and by Proposition~\ref{proposition:remainderinsmallcanon}, $\mathfrak{L}_{\Pi,E,E^-}\subseteq \mathfrak{C}_{\Pi,E\setminus E^-}$, $\mathfrak{I}^p_0\subseteq \mathfrak{C}_{\Pi,E\setminus E^-}$ and therefore, $\mathfrak{R}_0\subseteq\mathfrak{C}_{\Pi,E\setminus E^-}$.

    For the inductive step, consider each punctual fact $f$ such that $\mathfrak{R}_{\alpha+1}\models f$.
    If $\mathfrak{R}_{\alpha}\models f$, then by the induction hypothesis, $\mathfrak{C}_{\Pi,E\setminus E^-}\models f$.
    If $\mathfrak{R}_{\alpha}\not\models f$, then $\mathfrak{I}^p_{\alpha+1} \models f$ or $I_{\Pi,\mathfrak{L}_{\Pi,E,E^-}}(\mathfrak{R}_{\alpha}\cup \mathfrak{I}^p_{\alpha+1})\models f$.
    If $\mathfrak{I}^p_{\alpha+1} \models f$, again by Proposition~\ref{proposition:remainderinsmallcanon} we have $\mathfrak{I}^p_{\alpha+1}\subseteq \mathfrak{C}_{\Pi,E\setminus E^-}$.
    If $I_{\Pi,\mathfrak{L}_{\Pi,E,E^-}}(\mathfrak{R}_{\alpha}\cup \mathfrak{I}^p_{\alpha+1})\models f$, then there must be a ground instance $r'$ of $r$, a dataset $\mathcal{D}$ such that $r'$ minimally derives $f$ from $\mathcal{D}$, and $\mathfrak{L}_{\Pi,E,E^-}\cup\mathfrak{R}_{\alpha}\cup \mathfrak{I}^p_{\alpha+1}\models \mathcal{D} $. 
    Because $\mathfrak{L}_{\Pi,E,E^-}\cup\mathfrak{R}_{\alpha}\cup \mathfrak{I}^p_{\alpha+1}\subseteq\mathfrak{C}_{\Pi,E\setminus E^-}$, and $\mathfrak{C}_{\Pi,E\setminus E^-}$ is a model of $\Pi$, we have $\mathfrak{C}_{\Pi,E\setminus E^-}\models f$.
    Consequently, $\mathfrak{R}_{\alpha+1}\subseteq \mathfrak{C}_{\Pi,E\setminus E^-}$.
\end{proof}

Then, we show that the rederivation is complete.
\begin{proposition}\label{proposition:rederivationcomplete}
    $\mathfrak{C}_{\Pi,E\setminus E^-}\cap \mathfrak{M}_{\Pi,E,E^-}^-\subseteq  \mathfrak{R}_{\omega_1}$
\end{proposition}
\begin{proof}
    Because $\mathfrak{M}_{\Pi,E,E^-}\cap\mathfrak{L}_{\Pi,E,E^-}=\mathfrak{I}_\emptyset$, we only need to show that $\mathfrak{C}_{\Pi,E\setminus E^-}\subseteq \mathfrak{R}_{\omega_1}\cup\mathfrak{L}_{\Pi,E,E-}$.

    We show inductively that for each ordinal $\alpha<\omega_1$, each fact $f$ such that $T^\alpha_\Pi(\mathfrak{I}_{E\setminus E^-})\models f$, $\exists\beta<\omega_1$ such that $\mathfrak{R}_\beta\cup \mathfrak{L}_{\Pi,E,E^-}\models f$, and either $\beta = 0$ or $\mathfrak{R}_{\beta-1}\cup \mathfrak{L}_{\Pi,E,E^-}\not\models f$.

    When $\alpha=0$, because $ \mathfrak{I}_{E\setminus E-} - \mathfrak{M}_{\Pi,E,E^-}^-\subseteq \mathfrak{L}_{\Pi,E,E^-}$ and $\mathfrak{M}_{\Pi,E,E^-}^-\cap \mathfrak{I}_{E\setminus E^-}\subseteq\mathfrak{R}_0$, we have $\mathfrak{I}_{E\setminus E^-}\subseteq \mathfrak{R}_{0}\cup \mathfrak{L}_{\Pi,E,E^-}$, and for each fact that $T^0_\Pi(\mathfrak{I}_{E\setminus E^-})\models f$, $\mathfrak{R}_0\cup \mathfrak{L}_{\Pi,E,E^-}\models f$

   For the inductive step, we only need to consider each punctual fact $f$ such that $T^{\alpha+1}_\Pi(\mathfrak{I}_{E\setminus E^-})\models f$ but $T^{\alpha}_\Pi(\mathfrak{I}_{E\setminus E^-})\not\models f$.
    Then there must be a ground rule instance $r'$ of a rule $r\in\Pi$, a dataset $\mathcal{D}$ and a time point $t$ such that $r'$ minimally derives $f$ from $\mathcal{D}$, and $T^{\alpha}_\Pi(\mathfrak{I}_{E\setminus E^-})\models \mathcal{D}$.
    Let $\mathcal{D} = \{f_1,f_2,\dots,f_n\}$, and then by the induction hypothesis, there must be $\beta_1,\beta_2,\dots,\beta_n$ such that $\mathfrak{R}_{\beta_i}\cup \mathfrak{L}_{\Pi,E,E^-}\models f_i$, for $i\in\{1,2,\dots,n\}$.

    Let $k\in \{1,2,\dots,n\}$ such that $\forall i\in \{1,2,\dots,n\}$, $\beta_k\geq\beta_i$.
    If $\beta_k>0$, then  $\mathfrak{R}_{\beta_k}\models f$, $\mathfrak{R}_{\beta_k}\cap \mathfrak{I}_{\mathcal{D}}\neq\mathfrak{I}_\emptyset$.
    Because $\forall i\in \{1,2,\dots,n\}$, $\beta_k\geq\beta_i$, we have $\mathfrak{R}_{\beta_k}\cup \mathfrak{L}_{\Pi,E,E^-}\models f_i$, for $i\in\{1,2,\dots,n\}$, $\mathfrak{R}_{\beta_k}\cup \mathfrak{L}_{\Pi,E,E^-}\models \mathcal{D}$, and $\mathfrak{R}_{\beta_k+1}\cup \mathfrak{L}_{\Pi,E,E^-}\models f$.
    If $\beta_k=0$, then if $\exists i \in\{1,2,\dots,n\}$ such that $\mathfrak{I}_{\{f_i\}}\cap \mathfrak{R}_0\neq \mathfrak{I}_{\emptyset}$, then because $\mathfrak{R}_{0}\cup \mathfrak{L}_{\Pi,E,E^-}\models \mathcal{D}$, $\mathfrak{R}_1\cup \mathfrak{L}_{\Pi,E,E^-}\models f$.
    If $\forall i \in \{1,2,\dots,n\}$,  $\mathfrak{I}_{\{f_i\}}\cap \mathfrak{R}_0=\mathfrak{I}_{\emptyset}$,
    then $\mathfrak{L}_{\Pi,E,E^-}\models \mathcal{D}$.
    Let $f = M@t$. Because $\mathfrak{R}_0\not\models f$, $t\notin [\varrho_\mathsf{left}^-,\varrho_\mathsf{right}^+]$.
    Then there must be a natural number $n$ such that $t\in [\varrho_\mathsf{left}^--n\cdot|\varrho_\mathsf{left}|,\varrho_\mathsf{right}^++n\cdot|\varrho_\mathsf{right}|]$ and $t\notin [\varrho_\mathsf{left}^--(n-1)\cdot|\varrho_\mathsf{left}|,\varrho_\mathsf{right}^++(n-1)\cdot|\varrho_\mathsf{right}|]$.
    Because $\mathfrak{L}_{\Pi,E,E^-}\mid_{[\varrho_\mathsf{left}^--n\cdot|\varrho_\mathsf{left}|,\varrho_\mathsf{right}^++n\cdot|\varrho_\mathsf{right}|]}\models f$, we have $\mathfrak{R}_{n+1}\models f$.

    Finally, because $\mathfrak{R}_{\omega_1} = \bigcup_{\alpha<\omega_1}\mathfrak{R}_{\alpha}$ and $\mathfrak{C}_{\Pi,E\setminus E^-} = \bigcup_{\alpha<\omega_1}T^\alpha_{\Pi}(\mathfrak{I}_{E\setminus E^-})$, $\mathfrak{C}_{\Pi,E\setminus E^-}\subseteq \mathfrak{R}_{\omega_1}\cup \mathfrak{L}_{\Pi,E,E^-}$.
\end{proof}

Now we show rederivation is correct.
\begin{proposition}\label{proposition:rederivationcorrect}
    $\mathfrak{L}_{\Pi,E,E^-}\cup\mathfrak{R}_{\omega_1}=\mathfrak{C}_{\Pi,E\setminus E^-}$
\end{proposition}
\begin{proof}
    By Propositions~\ref{proposition:rederivationsound} and~\ref{proposition:remainderinsmallcanon}, we have 
$$\mathfrak{L}_{\Pi,E,E^-}\cup\mathfrak{R}_{\omega_1}\subseteq\mathfrak{C}_{\Pi,E\setminus E^-}$$.
    By Proposition~\ref{proposition:rederivationcomplete}, we have 
    \begin{align*}
        &\mathfrak{C}_{\Pi,E\setminus E^-}\cap \mathfrak{M}_{\Pi,E,E^-}^-\cup\mathfrak{L}_{\Pi,E,E^-}\\
         =& (\mathfrak{C}_{\Pi,E\setminus E^-}\cap \mathfrak{M}_{\Pi,E,E^-}^-)\cup (\mathfrak{C}_{\Pi,E\setminus E^-}- \mathfrak{M}_{\Pi,E,E^-}^-)\\
         =& \mathfrak{C}_{\Pi,E\setminus E^-}\\
        \subseteq& \mathfrak{R}_{\omega_1}\cup\mathfrak{L}_{\Pi,E,E^-}
    \end{align*}
\end{proof}

\begin{lemma}\label{lemma:rederivedepthrep}
    Let $\varrho$ be a closed interval of length $\mathsf{depth}(\Pi)$, and $E'$ a dataset representing $\mathfrak{R}_{\omega_1}\mid_{\varrho}$, and let a transfinite sequence of $\mathfrak{R}_1',\mathfrak{R}_2',\dots$ be defined identically to $\mathfrak{R}_\alpha$, except that $\mathfrak{R}_0' = \mathfrak{I}_{E'}\cup\mathfrak{I}^p_0$.
    Then both of the following hold:
    \begin{itemize}
        \item if $\varrho^+> t_E^+$, then $\mathfrak{R}_{\omega_1}\mid_{(\varrho^+,\infty)} = \mathfrak{R}_{\omega_1}'\mid_{(\varrho^+,\infty)}$;
        \item if $\varrho^-< t_E^-$, then $\mathfrak{R}_{\omega_1}\mid_{(-\infty,\varrho^-)} = \mathfrak{R}_{\omega_1}'\mid_{(-\infty,\varrho^-)}$
    \end{itemize}
\end{lemma}
\begin{proof}
    Consider $\varrho^+>t_E^+$. 
    First, we show that $\mathfrak{R}_{\omega_1}'\mid_{(\varrho^+,\infty)}\subseteq \mathfrak{R_{\omega_1}}\mid_{(\varrho^+,\infty)}$.

    To do this, we show inductively that for each ordinal $\alpha<\omega_1$, $\mathfrak{R}'_{\alpha}\mid_{(\varrho^+,\infty)}\subseteq\mathfrak{R}_{\omega_1}\mid_{(\varrho^+,\infty)}$
    
    When $\alpha=0$, since $\mathfrak{I}_{E'}\subseteq\mathfrak{R}_{\omega_1}$, we have $\mathfrak{R}_0'\subseteq\mathfrak{R}_{\omega_1}$

    For the inductive step, consider each punctual fact $f=M@t$, $t\in(\varrho^+,\infty)$ such that $\mathfrak{R}_{\alpha+1}\models f$ and $\mathfrak{R}_{\alpha}\not\models f$.
    If $\mathfrak{I}^p_{\alpha}\models f$, then because $\mathfrak{I}^p_{\alpha}\subseteq\mathfrak{R}_\alpha$, $\mathfrak{R}_{\omega_1}\models f$.
    If $\mathfrak{I}^p_{\alpha}\not\models f$, then $I_{\Pi,\mathfrak{L}_{\Pi,E,E^-}}(\mathfrak{R}_{\alpha}'\cup\mathfrak{I}^p_{\alpha})\models f$.
    Therefore, there must be a rule instance $r'$ of a rule $r\in\Pi$, a dataset $\mathcal{D}$ and a time point $t$, such that $r'$ derives $f$ from $t$, $\mathfrak{R}_{\alpha}'\cup\mathfrak{L}_{\Pi,E,E^-}\models \mathcal{D}$ and $\mathfrak{R}_{\alpha}'\cup\mathfrak{I}^p_{\alpha}\cap\mathfrak{I}_{\mathcal{D}}\neq\mathfrak{I}_\emptyset$.
    $\forall f'\in\mathcal{D}$, by the induction hypothesis, $\mathfrak{R}_{\omega_1}\cup\mathfrak{L}_{\Pi,E,E^-}\models f'$, and by Proposition~\ref{proposition:rederivationruler} there must be an $\alpha'$ such that $\mathfrak{R}_{\alpha'}\cup\mathfrak{L}_{\Pi,E,E^-}\models f'$, and the largest $\alpha'$ for each $f'$, $\beta$ satisfies that $\mathfrak{R}_\beta\cup\mathfrak{I}^p_\beta\cap \mathfrak{I}_{\mathcal{D}}\neq\mathfrak{I}_\emptyset$ and $\mathfrak{R}_\beta\cup \mathfrak{L}_{\Pi,E,E^-}\models \mathcal{D}$.
    Consequently, $\mathfrak{R}_{\beta+1}\models f$, and $\mathfrak{R}_{\beta+1}\models f$.

    Next, we show that $\mathfrak{R}_{\omega_1}\mid_{(\varrho^+,\infty)} \subseteq \mathfrak{R}_{\omega_1}'\mid_{(\varrho^+,\infty)}$.

    We prove inductively that for each ordinal $\alpha$, each ground relational atom $M$, each time point $t>\varrho^+$ $\mathfrak{R}_{\alpha},t\models M$ implies $\mathfrak{R}_{\alpha}',t \models M$. 

    In the base case, the implication holds trivially because for each $t>\varrho^+$, if $\mathfrak{R}_0,t \models M$, then $\mathfrak{I}_0^p,t\models M$.

    For the induction step, consider $\alpha+1$ for an ordinal $\alpha$. 
    For a time point $t>\varrho^+$ and a ground relational atom $M$, we only need to consider when $\mathfrak{R}_\alpha,t\not\models M$ and $\mathfrak{R}_{\alpha+1},t\models M$.
    If $\mathfrak{I}^p_{\alpha+1},t\models M$, then the implication holds trivially.
    If $I_{\Pi,\mathfrak{L}_{\Pi,E,E^-}}(\mathfrak{R}_\alpha\cup\mathfrak{I}_{\alpha+1}^p)$,
    then there must be a ground rule instance $r'$ of a rule $r\in\Pi$, a dataset $\mathcal{D}$ and a time point $t'$ such that $\mathfrak{L}_{\Pi,E,E^-}\cup\mathfrak{R}_\alpha\cup\mathfrak{I}_\alpha^p \models \mathcal{D}$, $\mathfrak{I}_{\mathcal{D}}\cap (\mathfrak{I}_{\alpha}^p\cup \mathfrak{R}_{\alpha})\neq\mathfrak{I}_{\emptyset}$, and $r'$ minimally derives $f$ from $\mathcal{D}$ at $t'$.

    Let the sum of every right endpoint of every interval mentioned in $head(r')$ be $m_1$, and the sum of right every endpoint mentioned in $body(r')$ be $m_2$. 
    Then by definition of $\mathsf{depth}(\Pi)$, $m_1+m_2 \leq \mathsf{depth}(\Pi)$.
    
     Because $r'$ derives $f$ at $t'$, we know that $t'\in [t-m_1,t+m_1]$ and $\forall M'@\varrho'\in\mathcal{D}$, $\varrho \subseteq [t'-m_2,t'+m_2]$.
     Since $m_1+m_2 \leq \mathsf{depth}(\Pi)$, and by the minimality of $\mathcal{D}$, every interval mentioned in $\mathcal{D}$ must be in $[t-\mathsf{depth}(\Pi),t+\mathsf{depth}(\Pi)]$. 
     And because $t>\varrho^+$ and $\varrho^+-\varrho^- = \mathsf{depth}(\Pi)$, $[t-\mathsf{depth}(\Pi),t+\mathsf{depth}(\Pi)]\subseteq [\varrho^-,t+\mathsf{depth}(\Pi)]$, and by the induction hypothesis, $\mathfrak{L}_{\Pi,E,E^-}\cup\mathfrak{R}'_\alpha\cup\mathfrak{I}_\alpha^p \models \mathcal{D}$,
     $\mathfrak{I}_{\mathcal{D}}\cap (\mathfrak{I}_{\alpha}^p\cup \mathfrak{R}_{\alpha})\subseteq (\mathfrak{I}_{\alpha}^p\cup \mathfrak{R}'_{\alpha})$, and
     $(\mathfrak{I}_{\alpha}^p\cup \mathfrak{R}'_{\alpha})\cap\mathfrak{I}_{\mathcal{D}}\neq \mathfrak{I}_{\emptyset}$, by which we have $\mathfrak{R}_{\alpha+1}'\models f$.

     For a limit ordinal $\alpha$, the implication holds trivially because $\mathfrak{R}_\alpha^{(')} =\bigcup_{\beta<\alpha}\mathfrak{R}^{(')}_\beta$.
\end{proof}

Given two interpretations $\mathfrak{I}_b$ and $\mathfrak{I}_0$,
successive applications of $I_{\Pi,\mathfrak{I}_{b}}$ to $\mathfrak{I}_0$ yield a transfinite sequence of interpretations $I_{\Pi,\mathfrak{I}_b}^{0}(\mathfrak{I}_0),I_{\Pi,\mathfrak{I}_b}^1(\mathfrak{I}_0),\dots$ such that 
\begin{enumerate*}[label=(\roman*)]
    \item $I_{\Pi,\mathfrak{I}_b}^{0}(\mathfrak{I}_0) = \mathfrak{I}_{0}$,
    \item $I_{\Pi,\mathfrak{I}_b}^{\alpha+1}(\mathfrak{I}_0) = I_{\Pi,\mathfrak{I}_b}(I_{\Pi,\mathfrak{I}_b}^{\alpha{}}(\mathfrak{I}_0))$, for $\alpha$ a successor ordinal, and
    \item $I_{\Pi,\mathfrak{I}_b}^{\beta}(\mathfrak{I}_0) = \bigcup_{\alpha<\beta}I_{\Pi,\mathfrak{I}_b}^{\alpha}(\mathfrak{I}_0)$, for $\beta$ a limit ordinal.
\end{enumerate*}

We say an interpretation $\mathfrak{I}$ \textbf{satisfies with insertion} a ground rule $r'$ (during $\varrho$) regarding $\mathfrak{I}_b$, if for each time point $t$ (in $\varrho$), $\mathfrak{I}_{b}\cup\mathfrak{I},t\models body(r')$ and $\mathfrak{I}_{b}\not\models body(r')$ implies $\mathfrak{I},t\models head(r')$.

Moreover, we say that interpretation $\mathfrak{I}$ \textbf{satisfies with insertion} a program $\Pi$ (during $\varrho$) regarding $\mathfrak{I}_b$, if satisfies every ground instance of every rule in $\Pi$ (during $\varrho$), and we say that $\mathfrak{I}$ is an \textbf{insertion model} of $\Pi$ (during $\varrho$). 
Interpretation is an insertion model of $\Pi$, $\mathfrak{I}_b$ and $\mathfrak{I}_0$ if it contains $\mathfrak{I}_0$ and is an insertion model of $\Pi$ regarding $\mathfrak{I}_{b}$.

By a proof analogous to that of Proposition~\ref{proposition:overdeletionruler},
we have 
\begin{proposition}\label{proposition:insertionruler}
    For each $\alpha\leq\omega_1,I^{\alpha}_{\Pi,\mathfrak{C}_{\Pi,E\setminus E^-}}(\mathfrak{I}_{E^+})$ is a $(\Pi,E\cup E^+)$-interpretation.
\end{proposition}
Furthermore, we claim that 
\begin{proposition}\label{proposition:mininsertmodeldefbyintertop}
    $I^{\omega_1}_{\Pi,\mathfrak{I}_b}(\mathfrak{I}_0)$ is the least insertion model of $\Pi$, $\mathfrak{I}_b$ and $\mathfrak{I}_0$.
\end{proposition}
\begin{proof}
    First, we show that $I^{\omega_1}_{\Pi,\mathfrak{I}_b}(\mathfrak{I}_0)$ is an insertion model of $\Pi$, $\mathfrak{I}_b$ and $\mathfrak{I}_0$.

    Because $I^{\omega_1}_{\Pi,\mathfrak{I}_b}(\mathfrak{I}_0)$ contains $\mathfrak{I}_0$, we only need to show that $I^{\omega_1}_{\Pi,\mathfrak{I}_b}(\mathfrak{I}_0)$ is an insertion model of $\Pi$ regarding $\mathfrak{I}_{b}$.

    For each fact $f$, ground rule instance $r'$ of a rule $r\in\Pi$, a dataset $\mathcal{D}$, a time point $t$ such that $r'$ minimally derives $f$ from $\mathcal{D}$ at $t$, $\mathfrak{I}_b\cup I^{\omega_1}_{\Pi,\mathfrak{I}_b}(\mathfrak{I}_0)\models \mathcal{D}$, $I^{\omega_1}_{\Pi,\mathfrak{I}_b}(\mathfrak{I}_0)\cap\mathfrak{I}_{\mathcal{D}}\neq \mathfrak{I}_\emptyset$.
    Let $\mathsf{rep}(I^{\omega_1}_{\Pi,\mathfrak{I}_b}(\mathfrak{I}_0)\cap\mathfrak{I}_{\mathcal{D}}) = \{f_1,\dots,f_m\}$ for some $m$,
    For each fact $f_i$ with $i = 1,\dots,m$, by Proposition~\ref{proposition:insertionruler} there must be $\alpha_i$ such that $I^{\alpha}_{\Pi,\mathfrak{I}_b}(\mathfrak{I}_0)\models f_i$.
    Let $\alpha = \max(\{\alpha_i\mid i\in\{1,\dots,m\}\})$.
    Then $\mathfrak{I}_b\cup I^{\alpha}_{\Pi,\mathfrak{I}_b}(\mathfrak{I}_0)\models \mathcal{D}$ and $I^{\alpha}_{\Pi,\mathfrak{I}_b}(\mathfrak{I}_0)\cap \mathfrak{I}_{\mathcal{D}}\neq \mathfrak{I}_{\emptyset}$. 
    Therefore, $I^{\alpha+1}_{\Pi,\mathfrak{I}_b}(\mathfrak{I}_0)\models f$, and $I^{\omega_1}_{\Pi,\mathfrak{I}_b}(\mathfrak{I}_0)\models f$.

    Next, we show inductively that if for each ordinal $\alpha\leq \omega_1$ and each fact $f$, if $I^{\alpha}_{\Pi,\mathfrak{I}_b}(\mathfrak{I}_0)\models f$, then the least insertion model of $\Pi$, $\mathfrak{I}_b$ and $\mathfrak{I}_0$ also satisfies $f$.
    Let $\mathfrak{I}$ be the least insertion model of $\Pi$, $\mathfrak{I}_b$ and $\mathfrak{I}_0$.

    The base case is trivial since $\mathfrak{I}$ contains $\mathfrak{I}_0$

    For the induction step, consider $f$ such that $I^{\alpha}_{\Pi,\mathfrak{I}_b}(\mathfrak{I}_0)\not\models f$ and $I^{\alpha+1}_{\Pi,\mathfrak{I}_b}(\mathfrak{I}_0)\models f$. 
    Then there must be a ground instance $r'$ of a rule $r\in\Pi$, a dataset $\mathcal{D}$ and a time point $t$ such that $\mathfrak{I}_b\cup I^{\alpha}_{\Pi,\mathfrak{I}_b}(\mathfrak{I}_0)\models f$, $I^{\alpha}_{\Pi,\mathfrak{I}_b}(\mathfrak{I}_0)\cap \mathfrak{I}_{\mathcal{D}}\neq \mathfrak{I}_{\emptyset}$.
    By the induction hypothesis, $I^{\alpha}_{\Pi,\mathfrak{I}_b}(\mathfrak{I}_0)\subseteq \mathfrak{I}$, and therefore, $\mathfrak{I}\cup \mathfrak{I}_b\models \mathcal{D}$, $\mathfrak{I}\cup \mathfrak{I}_b,t\models body(r')$, $\mathfrak{I}_b \not \models \mathcal{D}$ $\mathfrak{I}_b,t \not \models body(r')$. 
    By the definition of $\mathfrak{I}$, we have $\mathfrak{I},t\models head(r)'$ and $\mathfrak{I}\models f$.
\end{proof}

We show that inserting $T_\Pi(\mathfrak{L}_{\Pi,E,E^-})$ all at once is the same as inserting it over infinite many iterations.
\begin{proposition}\label{proposition:insertsameasrederive}
    Let $\mathfrak{I}_0 = \mathfrak{M}_{\Pi,E,E^-}^-\cap (\mathfrak{I}_{E\setminus E^-}\cup T_\Pi(\mathfrak{L}_{\Pi,E,E^-}))$, then
    $I^{\omega_1}_{\Pi,\mathfrak{L}_{\Pi,E,E^-}}(\mathfrak{I}_0) = \mathfrak{R}_{\omega_1}$.
\end{proposition}
\begin{proof}
    First, we show inductively that for each ordinal $\alpha$ and each fact bounded $f$, if $\mathfrak{R}_{\alpha}\models f$, then there must be $\beta<\omega_1$ such that $I^{\beta}_{\Pi,\mathfrak{L}_{\Pi,E,E^-}}(\mathfrak{I}_0)\models f$

    In the base case, the statement is trivially true because $\mathfrak{I}_0^p \subseteq T_\Pi(\mathfrak{L}_{\Pi,E,E^-})$.

    For the induction step, consider $\alpha+1$ for an ordinal $\alpha$. 
    We only need to consider each fact $f$ such that $\mathfrak{R}_{\alpha}\not\models f$ and $\mathfrak{R}_{\alpha+1}\models f$. 
    If $\mathfrak{I}_\alpha^p\models f$, then $I^{0}_{\Pi,\mathfrak{L}_{\Pi,E,E^-}}(\mathfrak{I}_0)\models f$.

    If $I_{\Pi,\mathfrak{L}_{\Pi,E,E^-}}(\mathfrak{R}_{\alpha}\cup \mathfrak{I}^p_{\alpha+1})$, then there must be a ground rule instance $r'$ of a rule $r\in\Pi$, a dataset $\mathcal{D}$ and a time point $t'$ such that $\mathfrak{L}_{\Pi,E,E^-}\cup\mathfrak{R}_\alpha\cup\mathfrak{I}_\alpha^p \models \mathcal{D}$, $\mathfrak{I}_{\mathcal{D}}\cap (\mathfrak{I}_{\alpha}^p\cup \mathfrak{R}_{\alpha})\neq\mathfrak{I}_{\emptyset}$, and $r'$ minimally derives $f$ from $\mathcal{D}$ at $t'$.
    Then by the induction hypothesis, there is some $\beta<\omega_1$ such that $\mathfrak{L}_{\Pi,E,E^-}\cup I^{\beta}_{\Pi,\mathfrak{L}_{\Pi,E,E^-}}(\mathfrak{I}_0)\models \mathcal{D}$, and $I^{\beta}_{\Pi,\mathfrak{L}_{\Pi,E,E^-}}(\mathfrak{I}_0)\cap \mathfrak{I}_{\mathcal{D}}\neq \mathfrak{I}_\emptyset$.
    Therefore, $I^{\beta+1}_{\Pi,\mathfrak{L}_{\Pi,E,E^-}}(\mathfrak{I}_0)\models f$.

    Next, we show inductively that for each ordinal $\alpha$ and each fact bounded $f$, if $I^{\alpha}_{\Pi,\mathfrak{L}_{\Pi,E,E^-}}(\mathfrak{I}_0)\models f$, then there must be $\beta<\omega_1$ such that $\mathfrak{R}_{\beta}\models f$

    In the base case the statement holds because either $\mathfrak{M}^-_{\Pi,E,E^-}\cap \mathfrak{I}_{E\setminus E^-}\subseteq \mathfrak{R}_0$, and for each bounded $f=M@\varrho$ such that $T_\Pi(\mathfrak{L}_{\Pi,E,E^-})$, there must be natural number $\beta$ such that $\varrho\subseteq [\varrho_{\mathsf{left}}^--\beta\cdot|\varrho_{\mathsf{left}}|,\varrho_{\mathsf{right}}^++\beta\cdot|\varrho_{\mathsf{right}}|]$, then $\mathfrak{R}_\beta \models f$.

    For the induction step, consider $\alpha+1$ for an ordinal $\alpha$. 
    We only need to consider each fact bounded $f$ such that $I^{\alpha}_{\Pi,\mathfrak{L}_{\Pi,E,E^-}}(\mathfrak{I}_0)\not\models f$ and $I^{\alpha+1}_{\Pi,\mathfrak{L}_{\Pi,E,E^-}}(\mathfrak{I}_0)\models f$. 
    Then there must be a ground rule instance $r'$ of a rule $r\in\Pi$, a dataset $\mathcal{D}$, a time point $t'$ such that $\mathfrak{L}_{\Pi,E,E^-}\cup I^{\alpha}_{\Pi,\mathfrak{L}_{\Pi,E,E^-}}(\mathfrak{I}_0)\models \mathcal{D} $, 
    $I^{\alpha}_{\Pi,\mathfrak{L}_{\Pi,E,E^-}}(\mathfrak{I}_0)\cap\mathfrak{I}_{\mathcal{D}}\neq \mathfrak{I}_{\emptyset}$, and $r'$ minimally derives $f$ from $\mathcal{D}$.
    Then by the induction hypothesis, there must an $\beta<\omega_1$, such that $\mathfrak{R}_\beta\cup\mathfrak{L}_{\Pi,E,E^-}\models\mathcal{D}$ and $\mathfrak{R}_\beta\cap\mathfrak{I}_{\mathcal{D}}\neq \mathfrak{I}_{\emptyset}$.
    Then, $\mathfrak{R}_{\beta+1}\models f$.
\end{proof}
Each pair of intervals $[\varrho_1^-,\varrho_2^-)$ and $(\varrho^+_3,\varrho_4^+]$ for each quadruple $\varrho_1,\varrho_2,\varrho_3,\varrho_4$ as above, can be referred to as the \textbf{periods} of $\mathfrak{R}_k$. 
Alternatively, we can denote them as $(\varrho_{\mathit{left}},\varrho_{\mathit{right}})$
\begin{definition}
    Interpretation $\mathfrak{R}_k$ is \textbf{rederivation-saturated} regarding $\Pi,E,E^-$, if there exist $\varrho_1,\varrho_2,\varrho_3$ and $\varrho_4$ of length $2\mathsf{depth}(\Pi)$,
whose endpoints are located on the $(\Pi, E)$-ruler and satisfy
$\varrho_1^+<\varrho_2^+<t_E^-$
and $t_E^+<\varrho^-_3<\varrho_4^-$, and for which the
following properties hold:
\begin{enumerate}
    \item Intervals $([\varrho_1^-,\varrho_2^-),(\varrho_3^+,\varrho_4^+])$ are periods of $\mathfrak{L}_{\Pi,E,E^-}$.
    \item $\mathfrak{R}_k$ is an insertion model of $\Pi,E,E^-$ during $[\varrho_1^- +\mathsf{depth}(\Pi),\varrho_4^+-\mathsf{depth}(\Pi)]$.
    \item $\mathfrak{R}_k\mid_{\varrho_2}$ is a shift of $\mathfrak{R}_k\mid_{\varrho_1}$, 
    \item $\mathfrak{R}_k\mid_{\varrho_4}$ is a shift of $\mathfrak{R}_k\mid_{\varrho_3}$.
\end{enumerate}

\end{definition}
\begin{proposition}
        Interpretation $\mathfrak{R}_k$ is also \textbf{rederivation-saturated} regarding $\Pi,E,E^-$, if there exist $\varrho_1,\varrho_2,\varrho_3$ and $\varrho_4$ of length $2\mathsf{depth}(\Pi)$,
whose endpoints are located on the $(\Pi, E)$-ruler and satisfy
$\varrho_1^+<\varrho_2^+<t_E^-$
and $t_E^+<\varrho^-_3<\varrho_4^-$, and for which the
following properties hold:
\begin{enumerate}
    \item Intervals $([\varrho_1^-,\varrho_2^-),(\varrho_3^+,\varrho_4^+])$ are periods of $\mathfrak{L}_{\Pi,E,E^-}$.
    \item $\mathfrak{R}_k\mid_{[\varrho_1^-,\varrho_4^+]}=\mathfrak{R}_{k+1}\mid_{[\varrho_1^-,\varrho_4^+]}$ 
    \item $\mathfrak{R}_k\mid_{\varrho_2}$ is a shift of $\mathfrak{R}_k\mid_{\varrho_1}$, 
    \item $\mathfrak{R}_k\mid_{\varrho_4}$ is a shift of $\mathfrak{R}_k\mid_{\varrho_3}$.
\end{enumerate}
\end{proposition}
\begin{proof}
The proof is analogous to that of Proposition~\ref{proposition:deletion-saturated}.
\end{proof}
\begin{theorem}\label{theo:rederivitionunfold}
    For a rederivation-saturated interpretation $\mathfrak{R}_k$ and a pair of its periods $\varrho_{\mathit{left}},\varrho_{\mathit{right}}$, the $(\varrho_{\mathit{left}},\varrho_{\mathit{right}})$-unfolding $\mathfrak{I}$ of $\mathfrak{R}_k$ coincides with $\mathfrak{R}_{\omega_1}$.
\end{theorem}
\begin{proof}
Let $\mathfrak{I}_0 = \mathfrak{M}_{\Pi,E,E^-}^-\cap (\mathfrak{I}_{E\setminus E^-}\cup T_\Pi(\mathfrak{L}_{\Pi,E,E^-}))$.
    
    First, we prove $$I^{\omega_1}_{\Pi,\mathfrak{L}_{\Pi,E,E^-}}(\mathfrak{I}_0)\subseteq\mathfrak{I}.$$ 
    Because $I^{\omega_1}_{\Pi,\mathfrak{L}_{\Pi,E,E^-}}$ is the least insertion-deletion model of $\Pi,\mathfrak{L}_{\Pi,E,E^-},\mathfrak{I}_0$, we only need to show that $\mathfrak{I}$ contains $\mathfrak{I}_0$ and is an insertion model of $\Pi$ regarding $\mathfrak{L}_{\Pi,E,E^-}$.

    Because $\mathfrak{R}_k$ contains $\mathfrak{M}^-_{\Pi,E,E^-}\cap\mathfrak{I}_{E\setminus E^-}$, $\varrho_\mathsf{left}$ and $\varrho_\mathsf{right}$ are periods of $\mathfrak{L}_{\Pi,E,E^-}$, $T_\Pi(\mathfrak{L}_{\Pi,E,E^-})\subseteq \mathfrak{I}$, and $\mathfrak{I}_0\subseteq \mathfrak{I}$.

    To show that $\mathfrak{I}$ is an insertion model of $\Pi$ regarding $\mathfrak{L}_{\Pi,E,E^-}$, we first show that $\mathfrak{I}$ satisfies with insertion $\Pi$ regarding $\mathfrak{L}_{\Pi,E,E^-}$ during $\varrho=[\varrho_\mathsf{left}^-+\mathsf{depth}(\Pi),\varrho_\mathsf{right}^+ - \mathsf{depth}(\Pi)]$.
    This is the case because $\mathfrak{I}\mid_{[\varrho_\mathsf{left}^-,\varrho_\mathsf{{right}}^+]} = \mathfrak{R}_k\mid_{[\varrho_\mathsf{left}^-,\varrho_\mathsf{{right}}^+]}$, and by the definition of a deletion-saturated interpretation, $\mathfrak{R}_k$ satisfies with insertion $\Pi$ regarding $\mathfrak{L}_{\Pi,E,E^-}$ during $[\varrho_\mathsf{left}^-+\mathsf{depth}(\Pi),\varrho_\mathsf{right}^+ - \mathsf{depth}(\Pi)]$.
    Then we show that $\forall t \notin\varrho$, $\mathfrak{I}$ satisfies with insertion $\Pi$ regarding $\mathfrak{L}_{\Pi,E,E^-}$ at $t$.
    Suppose $t<\varrho^-$, then $t< \varrho_\mathsf{left}^-+\mathsf{depth}(\Pi)$. 
    Then there must a natural number $n$ and and a time point $t'\in [\varrho_\mathsf{left}^-+\mathsf{depth}(\Pi),\varrho_\mathsf{left}^+ +\mathsf{depth}(\Pi)]$ such that $t' = t+n\cdot|\varrho_\mathsf{left}|$.
    By the definition of $\mathfrak{I}$, we know that $\mathfrak{I}\mid_{[t-\mathsf{depth}(\Pi),t+\mathsf{depth}(\Pi)]}$ is a shift of $\mathfrak{I}\mid_{[t'-\mathsf{depth}(\Pi),t'+\mathsf{depth}(\Pi)]}$, and also $\mathfrak{L}_{\Pi,E,E^-}\mid_{[t-\mathsf{depth}(\Pi),t+\mathsf{depth}(\Pi)]}$ is a shift of $\mathfrak{L}_{\Pi,E,E^-}\mid_{[t'-\mathsf{depth}(\Pi),t'+\mathsf{depth}(\Pi)]}$. 
    Therefore, for any rule instance $r'$ of a rule $r\in\Pi$, if it is satisfied with insertion by $\mathfrak{I}$ at $t'$, then it is satisfied with insertion by $\mathfrak{I}$ at $t$, and vice versa, and we have $\mathfrak{I}$ satisfies with insertion $\Pi$ regarding $\mathfrak{L}_{\Pi,E,E^-}$ at $t$ because $t'\in\varrho$. 
    The proof for $t>\varrho^+$ is symmetric.
    Consequently, $I^{\omega_1}_{\Pi,\mathfrak{L}_{\Pi,E,E^-}}(\mathfrak{I}_0)\subseteq\mathfrak{I}$
    
    Next, we prove $$\mathfrak{I}\subseteq I^{\omega_1}_{\Pi,\mathfrak{L}_{\Pi,E,E^-}}(\mathfrak{I}_0).$$
    Because $\mathfrak{R}_k\subseteq I^{\omega_1}_{\Pi,\mathfrak{L}_{\Pi,E,E^-}}(\mathfrak{I}_0)$ and $\mathfrak{I}\mid_{[\varrho_\mathsf{left}^-,\varrho_\mathsf{{right}}^+]} =\mathfrak{R}_k\mid_{[\varrho_\mathsf{left}^-,\varrho_\mathsf{{right}}^+]}$, we have $\mathfrak{I}\mid_{[\varrho_\mathsf{left}^-,\varrho_\mathsf{{right}}^+]}\subseteq I^{\omega_1}_{\Pi,\mathfrak{L}_{\Pi,E,E^-}}(\mathfrak{I}_0)$.
    Consider $\mathfrak{I}\mid_{[\varrho_\mathsf{right}^+,\infty)}$, we show inductively that for each $t_i$ on the $(\Pi,E)$-ruler that $\mathfrak{I}\mid_{[\varrho_\mathsf{left}^-,t_i]}\subseteq\mathfrak{M}^-_{\Pi,E,E^-}\mid_{[\varrho_\mathsf{left}^-,t_i]} $, where $t_0 = \varrho_\mathsf{right}^+$.
    The base case is already proven.
    For the inductive step, let $\varrho_A = (t_{i-1},t_i]$, and $\varrho_B = [t_{i-1}-\mathsf{depth}(\Pi),t_{i-1}]$, $\varrho_A' = \varrho_A - |\varrho_\mathsf{right}|$, $\varrho_B' = \varrho_B - |\varrho_\mathsf{right}|$. 
    We only need to show that $\mathfrak{I}\mid_{\varrho_A}\subseteq I^{\omega_1}_{\Pi,\mathfrak{L}_{\Pi,E,E^-}}(\mathfrak{I}_0)\mid_{\varrho_A}$. 
    Notice that by definition of $\mathfrak{I}$, $\mathfrak{I}\mid_{\varrho_B}$ is a shift of $\mathfrak{I}\mid_{\varrho_B'}$, and by the induction hypothesis and $I^{\omega_1}_{\Pi,\mathfrak{L}_{\Pi,E,E^-}}(\mathfrak{I}_0)\subseteq\mathfrak{I}$, we have $I^{\omega_1}_{\Pi,\mathfrak{L}_{\Pi,E,E^-}}(\mathfrak{I}_0)\mid_{\varrho_B'}$ is a shift of $I^{\omega_1}_{\Pi,\mathfrak{L}_{\Pi,E,E^-}}(\mathfrak{I}_0)\mid_{\varrho_B}$.
    Since by definition of $\mathfrak{I}$, $\mathfrak{L}_{\Pi,E,E^-}\mid_{[\varrho_B'^-,\infty)}$ is a shift of $\mathfrak{L}_{\Pi,E,E^-}\mid_{[\varrho_B^-,\infty)}$ and $|\varrho_B'|=|\varrho_B|=\mathsf{depth}(\Pi)$, by Lemma~\ref{lemma:rederivedepthrep} we can show that $I^{\omega_1}_{\Pi,\mathfrak{L}_{\Pi,E,E^-}}(\mathfrak{I}_0)\mid_{\varrho_A}$ is a shift of $I^{\omega_1}_{\Pi,\mathfrak{L}_{\Pi,E,E^-}}(\mathfrak{I}_0)\mid_{\varrho_A'}$, and because $\mathfrak{I}\mid_{\varrho_A'}$ is a shift of $\mathfrak{I}\mid_{\varrho_A}$, and $\mathfrak{I}\mid_{\varrho_A'}\subseteq I^{\omega_1}_{\Pi,\mathfrak{L}_{\Pi,E,E^-}}(\mathfrak{I}_0)\mid_{\varrho_A'}$, we have $\mathfrak{I}\mid_{\varrho_A}\subseteq I^{\omega_1}_{\Pi,\mathfrak{L}_{\Pi,E,E^-}}(\mathfrak{I}_0)\mid_{\varrho_A}$. 
    The case for $\mathfrak{I}\mid_{(-\infty,\varrho_\mathsf{right}^+)}$ is symmetric.
    Therefore, $\mathfrak{I}\subseteq\mathfrak{M}^-_{\Pi,E,E^-}$.
\end{proof}

Given a DatalogMTL program $\Pi$, interpretations $\mathfrak{I}_b$, we define the \textbf{semi-na\"ive immediate insertion-consequence operator}, $\dot{I}_{\Pi,\mathfrak{I}_b}$, mapping two interpretations $\mathfrak{I}_1,\mathfrak{I}_2$ to the least interpretation $\dot{I}_{\Pi,\mathfrak{I}_{b}}(\mathfrak{I}_1,\mathfrak{I_2})$ containing $\mathfrak{I}_2$ satisfying for each ground rule instance $r'$ of a rule $r\in\Pi$, a fact $f$, a dataset $\mathcal{D}$ such that if $r'$ minimally derives $f$ from $\mathcal{D}$, $\mathfrak{I}_b \cup \mathfrak{I}_2\models \mathcal{D}$ and $(\mathfrak{I}_2-\mathfrak{I}_1)\cap\mathfrak{I}_{\mathcal{D}}\neq \mathfrak{I}_{\emptyset}$.

Given $\mathfrak{L}_{\Pi,E,E^-}$, $\varrho_\mathsf{left}$, and $\varrho_\mathsf{right}$, for an ordinal $\alpha<\omega_1$, let $\mathfrak{I}^p_{\alpha}=(\mathfrak{M}_{\Pi,E,E^-}^- \cap \mathfrak{L}_{\Pi,E,E^-}\mid_{[\varrho_\mathsf{left}^--\alpha\cdot|\varrho_\mathsf{left}|,\varrho_\mathsf{right}^++\alpha\cdot|\varrho_\mathsf{right}|]})$. 
We define a transfinite sequence of interpretations $\mathfrak{R}_0,\mathfrak{R}_1,\dots$ as follows
\begin{enumerate*}[label=(\roman*)]
    \item $\dot{\mathfrak{R}}_0 = \mathfrak{M}_{\Pi,E,E^-}^- \cap(\mathfrak{I}_{E\setminus E^-}\cup \mathfrak{I}^p_0) $,
    \item $\dot{\mathfrak{R}}_{1} = \dot{I}_{\Pi,\mathfrak{L}_{\Pi,E,E^-}}(\mathfrak{I}_{\emptyset},\dot{\mathfrak{R}}_{0}\cup \mathfrak{I}^p_1)\cup \mathfrak{I}^p_1$
    \item  
     $\dot{\mathfrak{R}}_{\alpha+1} = \dot{I}_{\Pi,\mathfrak{L}_{\Pi,E,E^-}}(\dot{\mathfrak{R}}_{\alpha-1},\dot{\mathfrak{R}}_{\alpha}\cup \mathfrak{I}^p_{\alpha+1})\cup \mathfrak{I}^p_{\alpha+1}$, for $\alpha$ a successor ordinal, and
    \item $\dot{\mathfrak{R}}_\beta = \bigcup_{\alpha<\beta}\dot{\mathfrak{R}}_\alpha$, for $\beta$ a limit ordinal
\end{enumerate*}

\begin{proposition}\label{proposition:rederivationseminaive=naive}
    For each ordinal $\alpha$, $$\dot{\mathfrak{R}}_\alpha = \mathfrak{R}_\alpha$$
\end{proposition}
\begin{proof}
    We prove the statement by induction on $\alpha$.

    When $\alpha = 0$, the equation is trivial.

    When $\alpha=1$, because $\dot{I}_{\Pi,\mathfrak{L}_{\Pi,E,E^-}}(\mathfrak{I}_\emptyset)$ is the same operator as $I_{\Pi,\mathfrak{L}_{\Pi,E,E^-}}(*)$, we have $\dot{\mathfrak{R}}_1 = \mathfrak{R}_1$.

    For the inductive step, we first show that $\dot{\mathfrak{R}}_{\alpha+1}\subseteq\mathfrak{R}_{\alpha+1}$. 
    By definition $\dot{\mathfrak{R}}_{\alpha+1} = \dot{I}_{\Pi,\mathfrak{L}_{\Pi,E,E^-}}(\dot{\mathfrak{R}}_{\alpha-1},\dot{\mathfrak{R}}_{\alpha}\cup \mathfrak{I}^p_{\alpha+1})\cup \mathfrak{I}^p_{\alpha+1}$.
    For each fact $f$ such that $\dot{\mathfrak{R}}_{\alpha+1}\models f$, if $\mathfrak{I}^p_{\alpha+1}\models f$ or $\dot{\mathfrak{R}}_\alpha\models f$, then the it is trivial that $\mathfrak{R}_{\alpha+1}\models f$.
    If $\mathfrak{I}^p_{\alpha+1}\not\models f$ or $\dot{\mathfrak{R}}_\alpha\not\models f$, then there must be a ground rule instance $r'$ of a rule $r\in\Pi$, a fact $f$, a dataset $\mathcal{D}$ such that if $r'$ minimally derives $f$ from $\mathcal{D}$, $\mathfrak{L}_{\Pi,E,E^-} \cup \dot{\mathfrak{R}}_\alpha\cup\mathfrak{I}^p_{\alpha+1}\models \mathcal{D}$ and $(\dot{\mathfrak{R}}_\alpha\cup\mathfrak{I}^p_{\alpha+1}-\dot{\mathfrak{R}}_{\alpha-1})\cap\mathfrak{I}_{\mathcal{D}}\neq \mathfrak{I}_{\emptyset}$.
    Then by the induction hypothesis $\mathfrak{R}_\alpha = \dot{\mathfrak{R}_\alpha}$, $\mathfrak{L}_{\Pi,E,E^-} \cup \mathfrak{R}_\alpha\cup\mathfrak{I}^p_{\alpha+1}\models \mathcal{D}$, and $(\mathfrak{R}_\alpha\cup\mathfrak{I}^p_{\alpha+1})\cap\mathfrak{I}_{\mathcal{D}}\neq \mathfrak{I}_{\emptyset}$.
    Consequently, by definition $I_{\Pi,\mathfrak{L}_{\Pi,E,E^-}}(\mathfrak{R}_\alpha\cup\mathfrak{I}^p_{\alpha+1})\models f$, and $\mathfrak{R}_{\alpha+1}\models f$.

    Next, we show that $\mathfrak{R}_{\alpha+1}\subseteq\dot{\mathfrak{R}}_{\alpha+1}$. 
    Again, we only need to consider each fact $f=M@t$ such that there is a ground rule instance $r'$ of a rule $r\in\Pi$, a dataset $\mathcal{D}$ such that $r'$ minimally derives $f$ from $\mathcal{D}$, $\mathfrak{L}_{\Pi,E,E^-}\cup\mathfrak{R}_\alpha\cup\mathfrak{I}_{\alpha+1}^p\models \mathcal{D}$, and $(\mathfrak{R}_\alpha\cup\mathfrak{I}_{\alpha+1}^p)\cap\mathfrak{I}_\mathcal{D}\neq \mathfrak{I}_\emptyset$.
    Let $\mathcal{D}_1'$ represent $(\mathfrak{R}_\alpha\cup\mathfrak{I}_{\alpha+1}^p)\cap\mathfrak{I}_\mathcal{D}$, $\mathcal{D}'_2$ represent $\mathfrak{I}_\mathcal{D}-(\mathfrak{R}_\alpha\cup\mathfrak{I}_{\alpha+1}^p)$, and $\mathcal{D}'=  \mathcal{D}_1'\cup\mathcal{D}_2'$, then by the induction hypothesis $\mathfrak{L}_{\Pi,E,E^-}\cup \mathfrak{R}_\alpha\cup\mathfrak{I}_{\alpha+1}^p\models \mathcal{D}'$, $\mathfrak{L}_{\Pi,E,E^-}\models \mathcal{D}'_2$. 
    We show that $(\dot{\mathfrak{R}}_{\alpha} - \dot{\mathfrak{R}}_{\alpha-1})\cap\mathfrak{I}_{\mathcal{D}'_1}\neq\mathfrak{I}_{\emptyset}$.
    We do this by showing that $\exists f'\in\mathcal{D}_1'$, $\mathfrak{R}_{\alpha}\models f'$ and $\mathfrak{R}_{\alpha-1}\not\models f$ or $\mathfrak{I}_{\alpha+1}^p\models f$. 
    If $\forall f'\in\mathcal{D}_1'$, $\mathfrak{I}_{\alpha+1}^p\not\models f$, then $\forall f'\in\mathcal{D}_1', \mathfrak{R}_{\alpha}\models f'$. If $\forall f'\in\mathcal{D}_1',\mathfrak{R}_{\alpha-1}\models f'$, then $\mathfrak{R}_{\alpha-1}\models \mathcal{D}_1'$, and because $\mathfrak{L}_{\Pi,E,E^-}\models \mathcal{D}_2'$, we have $\mathfrak{L}_{\Pi,E,E^-}\cup \mathfrak{R}_{\alpha-1}\cup\mathfrak{I}_{\alpha}^p\models \mathcal{D}'$, $\mathfrak{R}_{\alpha-1}\cap \mathfrak{I}_{\mathcal{D}'} \neq\mathfrak{I}_\emptyset$, and $\mathfrak{R}_\alpha\models f$, which contradicts $\mathfrak{R}_\alpha\not\models f$.
    Therefore, $(\dot{\mathfrak{R}}_{\alpha} - \dot{\mathfrak{R}}_{\alpha-1})\cap\mathfrak{I}_{\mathcal{D}'_1}\neq\mathfrak{I}_{\emptyset}$, and because $\mathfrak{L}_{\Pi,E,E^-}\cup \mathfrak{R}_\alpha\cup\mathfrak{I}_{\alpha+1}^p\models \mathcal{D}'$, $\dot{\mathfrak{R}}_{\alpha+1}\models f$.

    When $\alpha=\beta$ a limit ordinal, the equivalence holds trivially because $\dot{\mathfrak{R}}_\beta = \bigcup_{\alpha<\beta}\dot{\mathfrak{R}}_\alpha$ and 
    $\mathfrak{R}_\beta = \bigcup_{\alpha<\beta}\mathfrak{R}_\alpha$.
\end{proof}

With proof analogous to that of Proposition~\ref{proposition:deletion-saturated-align}, Propositions~\ref{proposition:rederivation-saturated-align} holds as well.
\begin{proposition}\label{proposition:rederivation-saturated-align}
Given a periodic materialisation $\mathds{I}$ such that $\mathsf{unfold}(\mathds{I}) =\mathfrak{L}_{\Pi,E,E^-}$, there exists $k \leq k_{\mathsf{max}}$ such that $\mathfrak{R}_k$ 
is rederivation-saturated, where the bound $k_{max}$ is defined as follows.
Let $A$ be the number of ground relational atoms in the
grounding of $\Pi$ with constants from $\Pi$ and $E$, let $B$ be the
number of $(\Pi, E)$-intervals within $[t_E^-, t_E^-
 + 2depth(\Pi)]$, let $C$ be the max number of $(\Pi,E)$-endpoints contained by $\mathds{I}.\varrho_{\mathsf{L}}$ or $\mathds{I}.\varrho_{\mathsf{R}}$,
and let $\varrho$ = $[t_{-w} - 2C\mathsf{depth}(\Pi), t_w + 2C\mathsf{depth}(\Pi)]$, where $\cdots < t_{-2}<t_{-1}<t_{E}^-$ and $t_E^+<t_1<t_2<\cdots$ are sequences of consecutive time points on the $(\Pi, E)$-ruler, while
$w$ = $C +B ·C· (2^A)^B$. 
Then $k_\mathsf{max}$ is the product of $A$ and the
number of $(\Pi, E)$-intervals contained in $\varrho$.
\end{proposition}

\subsection{Properties about Insertion}
Next, we deal with the insertion phase.
\begin{proposition}\label{proposition:insertionruler}
For each $\alpha\leq\omega_1,I_{\Pi,\mathfrak{C}_{\Pi,E\setminus E^-}}^\alpha(\mathfrak{I}_{E^+})$ is a $(\Pi,E\cup E^+)$-interpretation.
\end{proposition}
\begin{proof}
    The proof is analogous to that for Proposition~\ref{proposition:overdeletionruler}.
\end{proof}
\begin{proposition}\label{proposition:insertwithinnewcanon}
$I^{\omega_1}_{\Pi,\mathfrak{C}_{\Pi,E\setminus E^-}}(\mathfrak{I}_{E^+})\subseteq \mathfrak{C}_{\Pi,E\setminus E^-\cup E^+}$
\end{proposition}
\begin{proof}
    First, it is obvious that we have $\mathfrak{I}_{E^+}\subseteq \mathfrak{C}_{\Pi,E\setminus E^-\cup E^+}$.
    Next, we prove inductively that for each ordinal $\alpha\leq \omega_1$, $$I^{\alpha}_{\Pi,\mathfrak{C}_{\Pi,E\setminus E^-}}(\mathfrak{I}_{E^+})\subseteq \mathfrak{C}_{\Pi,E\setminus E^-\cup E^+}    $$

    The base case, where $\alpha=0$, is already proven.

    For the inductive step, consider a fact $f$ such that $I^{\alpha+1}_{\Pi,\mathfrak{C}_{\Pi,E\setminus E^-}}(\mathfrak{I}_{E^+})\models f$ and $I^{\alpha}_{\Pi,\mathfrak{C}_{\Pi,E\setminus E^-}}(\mathfrak{I}_{E^+})\not\models f$.
    Then there must a ground rule instance $r'$ of a rule $r\in\Pi$, a dataset $\mathcal{D}$ and a time point $t$ such that $r'$ minimally derives $f$ from $\mathcal{D}$, $\mathfrak{C}_{\Pi,E\setminus E^-}\cup I^{\alpha}_{\Pi,\mathfrak{C}_{\Pi,E\setminus E^-}}(\mathfrak{I}_{E^+})\models \mathcal{D}$ and $I^{\alpha}_{\Pi,\mathfrak{C}_{\Pi,E\setminus E^-}}(\mathfrak{I}_{E^+})\cap\mathfrak{I}_{\mathcal{D}}\neq \mathfrak{I}_{\emptyset}$.
    By the induction hypothesis $\mathfrak{C}_{\Pi,E\setminus E^-}\cup I^{\alpha}_{\Pi,\mathfrak{C}_{\Pi,E\setminus E^-}}(\mathfrak{I}_{E^+})\subseteq\mathfrak{C}_{\Pi,E\setminus E^-\cup E^+}$, and therefore, $\mathfrak{C}_{\Pi,E\setminus E^-\cup E^+},t\models body(r')$.
    Because $\mathfrak{C}_{\Pi,E\setminus E^-\cup E^+}$ is a model of $\Pi$, $\mathfrak{C}_{\Pi,E\setminus E^-\cup E^+},t \models head(r')$, and  $\mathfrak{C}_{\Pi,E\setminus E^-\cup E^+}\models f$.
    Consequently, $I^{\alpha+1}_{\Pi,\mathfrak{C}_{\Pi,E\setminus E^-}}(\mathfrak{I}_{E^+})\subseteq\mathfrak{C}_{\Pi,E\setminus E^-\cup E^+}$.

    The case when $\alpha=\omega_1$ holds trivially because
    $I^{\omega_1}_{\Pi,\mathfrak{C}_{\Pi,E\setminus E^-}}(\mathfrak{I}_{E^+}) = \bigcup_{\alpha<\omega_1}I^{\alpha}_{\Pi,\mathfrak{C}_{\Pi,E\setminus E^-}}(\mathfrak{I}_{E^+})$.
\end{proof}

\begin{proposition}\label{proposition:insertioncomplete}
    $\mathfrak{C}_{\Pi,E\setminus E^-\cup E^+} - \mathfrak{C}_{\Pi,E\setminus E^-}\subseteq I^{\omega_1}_{\Pi,\mathfrak{C}_{\Pi,E\setminus E^-}}(\mathfrak{I}_{E^+})$
\end{proposition}
\begin{proof}
    For each punctual fact $f$ such that $\mathfrak{C}_{\Pi,E\setminus E^-\cup E^+} - \mathfrak{C}_{\Pi,E\setminus E^-}\models f$, if $f$ is not partially rooted in $E^+$ regarding $\Pi,E\setminus E^-\cup E^+$, then for each $T\in\mathfrak{D}_{\Pi,E\setminus E^-\cup E^+}(f)$, $T=(V,\mathcal{E})$, and for each leaf node $v$ of $T$, it must be that $\mathfrak{I}_{\{l(v)\}}\cap\mathfrak{I}_{E^+}=\mathfrak{I}_\emptyset$, which means $\mathfrak{I}_{E\setminus E^-}\models l(v)$, and by Proposition~\ref{proposition:leafrange}, we have $\mathfrak{C}_{\Pi,E\setminus E^-}\models f$, and because $\mathfrak{C}_{\Pi,E\setminus E^-\cup E^+} - \mathfrak{C}_{\Pi,E\setminus E^-}\models f$, we have arrived at a contradiction.

    Therefore, $f$ must be partially rooted in $E^+$ regarding $\Pi,E\setminus E^-\cup E^+$.
    Let the set of all such derivation trees for $f$ be $\mathcal{T}(f)$.
    For a derivation tree $T$, let the length of the longest path from any of its leaves to its root be $h(T)$, and let $h(f) = \max(\{h(T)\mid T\in \mathcal{T(f)}\})$.

    We show inductively that for each natural number $n$ that for each $f$ such that $h(f)=n$, there is some $\alpha<\omega_1$ such that $I^{\alpha}_{\Pi,\mathfrak{C}_{\Pi,E\setminus E^-}}(\mathfrak{I}_{E^+})\models f$.

    When $n=0$, then $\mathfrak{I}_{E^+}\models f$, and the statement is trivially true.

    For the inductive step, let $T$ be such that $T\in\mathcal{T}(f)$, $h(T) = h(f)$. 
    Let the root of $T$ be $v$, $T=(V,\mathcal{E})$, 
    by the definition of $T$, there must be $v_{f'}\in V$, $l(v_{f'}) = f'$, $\mathfrak{I}_{\{f'\}}\cap \mathfrak{I}_{E^+}\neq\mathfrak{I}_{\emptyset}$. 
    Let the nodes on the path from $v_{f'}$ to $v$ be $v_{f'},v_k,\dots v_1,v$ for some $k$.
    Then by Proposition~\ref{proposition:partiallyrootedsub},  $\forall i\in\{1,2\dots,k\}$, $f_i = l(v_i)$ is partially rooted in $E^+$.
    Let the in-neighbours of $v$ be $v_1,v_2',\dots v_m'$ for some $m$. 
    For each $i\in\{2,3,\dots,m\}$, if $l(v_i')$ is partially rooted in $E^+$, then $h(l(v_i'))<h(f)$, otherwise, there must be a tree $T_i'\in\mathcal{T}(l(v_i'))$ with $h(T_i')\geq h(f)$, we replace the subtree of $T$ rooted at $v_i'$ with $T_i'$, then we derive a derivation tree $T'$ of $f$ in $\mathcal{T}(f)$ but $h(T')>h(f)$, which contradicts the definition of $h(f)$. 
    Since $h(l(v_i'))<h(f)$, by the induction hypothesis, $\exists \alpha_i<\omega_1,I^{\alpha_i}_{\Pi,\mathfrak{C}_{\Pi,E\setminus E^-}}(\mathfrak{I}_{E^+})\models l(v_i')$.
    Note that since $v_1$ is partially rooted in $E^+$ and an in-neighbour of $v$, the above reasoning also applies to $v_1$, and there must be $\alpha_1<\omega_1$ such that $I^{\alpha_1}_{\Pi,\mathfrak{L}_{\Pi,E,E^-}}(\mathfrak{I}_A)\models l(v_1)$.
    If $l(v_i')$ is not partially rooted in $E^+$, then we have already shown that $\mathfrak{C}_{\Pi, E\setminus E^-}\models l(v_i')$ when discussing whether $f$ is partially rooted in $E^+$, and $\forall \beta<\omega_1,I^{\beta}_{\Pi,\mathfrak{C}_{\Pi,E\setminus E^-}}(\mathfrak{I}_{E^+})\cup \mathfrak{C}_{\Pi,E\setminus E^-}\models l(v_i')$. In this case, we let $\alpha_i = 0$.
    Let the set of facts represented by the in-neighbours of $v$ be $\mathcal{D}$, then by definition of $\mathfrak{D}_{\Pi,E\setminus E^-\cup E^+}$, we know that $l((v_1,v))$ minimally derives $f$ from $\mathcal{D}$ at some $t$. 

    Let $\alpha = \max(\{\alpha_i\mid i\in\{1,2\dots,m\}\})$,
    and we have $I^{\alpha}_{\Pi,\mathfrak{C}_{\Pi,E\setminus E^-}}(\mathfrak{I}_A)\cup \mathfrak{C}_{\Pi,E\setminus E^-}\models \mathcal{D}$, $\mathfrak{I}_{\{l(v_1)\}}\subseteq I^{\alpha}_{\Pi,\mathfrak{C}_{\Pi,E\setminus E^-}}(\mathfrak{I}_{E^+})$ and therefore $I^{\alpha}_{\Pi,\mathfrak{C}_{\Pi,E\setminus E^-}}(\mathfrak{I}_{E^+})\cap\mathfrak{I}_{\mathcal{D}}\neq\mathfrak{I}_{\emptyset}$, and by the definition of $I_{\Pi,\mathfrak{C}_{\Pi,E\setminus E^-}}$, $I^{\alpha+1}_{\Pi,\mathfrak{C}_{\Pi,E\setminus E^-}}(\mathfrak{I}_{E^+})\models f$

    Because $I^{\omega_1}_{\Pi,\mathfrak{C}_{\Pi,E\setminus E^-}}(\mathfrak{I}_{E^+})= \bigcup_{\alpha<\beta}I^{\alpha}_{\Pi,\mathfrak{C}_{\Pi,E\setminus E^-}}(\mathfrak{I}_{E^+})$, for each punctual fact $f$ partially rooted in $E^+$, $I^{\omega_1}_{\Pi,\mathfrak{C}_{\Pi,E\setminus E^-}}(\mathfrak{I}_{E^+})\models f$, and for each punctual fact $f$ such that $\mathfrak{C}_{\Pi,E\setminus E^-\cup E^+} - \mathfrak{C}_{\Pi,E\setminus E^-}\models f$, $I^{\omega_1}_{\Pi,\mathfrak{C}_{\Pi,E\setminus E^-}}(\mathfrak{I}_{E^+})\models f$, which means $\mathfrak{C}_{\Pi,E\setminus E^-\cup E^+} - \mathfrak{C}_{\Pi,E\setminus E^-}\subseteq I^{\omega_1}_{\Pi,\mathfrak{C}_{\Pi,E\setminus E^-}}(\mathfrak{I}_{E^+})$.
\end{proof}

We can now show that insertion is correct
\begin{proposition}
\label{proposition:insertioncorrect}
    $$ I^{\omega_1}_{\Pi,\mathfrak{C}_{\Pi,E\setminus E^-}}(\mathfrak{I}_A)\cup \mathfrak{C}_{\Pi,E\setminus E^-}=\mathfrak{C}_{\Pi,E\setminus E^-\cup E^+}$$
\end{proposition}
\begin{proof}
    By
    \begin{align*}
        &I^{\omega_1}_{\Pi,\mathfrak{C}_{\Pi,E\setminus E^-}}(\mathfrak{I}_{E^+}) \cup \mathfrak{C}_{\Pi,E\setminus E^-}\\
    \subseteq & \mathfrak{C}_{\Pi,E\setminus E^-\cup E^+}\cup\mathfrak{C}_{\Pi,E\setminus E^-}\quad \text{(By Proposition~\ref{proposition:insertwithinnewcanon})}\\
    \subseteq& \mathfrak{C}_{\Pi,E\setminus E^-\cup E^+} \cup \mathfrak{C}_{\Pi,E\setminus E^-\cup E^+}\\
    =&\mathfrak{C}_{\Pi,E\setminus E^-\cup E^+}
    \end{align*}
    and 
    \begin{align*}
        &\mathfrak{C}_{\Pi,E\setminus E^-\cup E^+}\\
        =& (\mathfrak{C}_{\Pi,E\setminus E^-\cup E^+} - \mathfrak{C}_{\Pi,E\setminus E^-})\cup (\mathfrak{C}_{\Pi,E\setminus E^-\cup E^+}\cap\mathfrak{C}_{\Pi,E\setminus E^-})\\
        \subseteq& I^{\omega_1}_{\Pi,\mathfrak{L}_{\Pi,E,E^-}(\mathfrak{I}_A)}\cup\mathfrak{C}_{\Pi,E\setminus E^-}\quad\text{(By Proposition~\ref{proposition:insertioncomplete})}\\
    \end{align*}
    we have
    $$I^{\omega_1}_{\Pi,\mathfrak{C}_{\Pi,E\setminus E^-}}(\mathfrak{I}_{E^+})\cup \mathfrak{C}_{\Pi,E\setminus E^-}=\mathfrak{C}_{\Pi,E\setminus E^-\cup E^+}$$
\end{proof}

Next, we show that $I^{\omega_1}_{\Pi,\mathfrak{C}_{\Pi,E\setminus E^-}}(\mathfrak{I}_{E^+})$ has a finite representation that can be found without applying the operator infinitely many times.
\begin{lemma}\label{lemma:insertdepthrep}
    Let $\varrho$ be a closed interval of length $\mathsf{depth}(\Pi)$ and $E'$ a dataset representing $I^{\omega_1}_{\Pi,\mathfrak{C}_{\Pi,E\setminus E^-}}(\mathfrak{I}_{E^+})\mid_\varrho $. 
    Then both of the following hold:
    \begin{itemize}
        \item if $\varrho^+> t_{E\cup E^+}^+$, then $I^{\omega_1}_{\Pi,\mathfrak{C}_{\Pi,E\setminus E^-}}(\mathfrak{I}_{E^+})\mid_{(\varrho^+,\infty)} = I^{\omega_1}_{\Pi,\mathfrak{C}_{\Pi,E\setminus E^-}}(\mathfrak{I}_{E'})\mid_{(\varrho^+,\infty)} $,
        \item if $\varrho^-< t_E^-$, then $I^{\omega_1}_{\Pi,\mathfrak{C}_{\Pi,E\setminus E^-}}(\mathfrak{I}_{E^+})\mid_{(-\infty,\varrho^-)} = I^{\omega_1}_{\Pi,\mathfrak{C}_{\Pi,E\setminus E^-}}(\mathfrak{I}_{E'})\mid_{(-\infty,\varrho^-)} $.
    \end{itemize}

\end{lemma}
\begin{proof}
    We consider $\varrho^+ > t^+_{E\cup E^+}$. 
    Since $\mathfrak{I}_{E'}\subseteq I^{\omega_1}_{\Pi,\mathfrak{C}_{\Pi,E\setminus E^-}}(\mathfrak{I}_{E^+})$, we have $I^{\omega_1}_{\Pi,\mathfrak{C}_{\Pi,E\setminus E^-}}(\mathfrak{I}_{E'})\subseteq I^{\omega_1}_{\Pi,\mathfrak{C}_{\Pi,E\setminus E^-}}(\mathfrak{I}_{E^+})$.
    To show $I^{\omega_1}_{\Pi,\mathfrak{C}_{\Pi,E\setminus E^-}}(\mathfrak{I}_{E^+})\mid_{(\varrho^+,\infty)}\subseteq I^{\omega_1}_{\Pi,\mathfrak{C}_{\Pi,E\setminus E^-}}(\mathfrak{I}_{E'})\mid_{(\varrho^+,\infty)}$, we prove inductively that for each $\alpha$, each ground relational atom $M$, each time point $t>\varrho^+$, if $I^{\alpha}_{\Pi,\mathfrak{C}_{\Pi,E\setminus E^-}}(\mathfrak{I}_{E^+}),t\models M$, then $I^{\alpha}_{\Pi,\mathfrak{C}_{\Pi,E\setminus E^-}}(\mathfrak{I}_{E'}),t\models M$

    In the base case, the statement is trivially true, because such that $t$ and $M$ do not exist.

    In the induction step, consider $\alpha+1$ for an ordinal $\alpha$.
    For each time point $t>\varrho^+$ and ground relational atom $M$ such that $I^{\alpha+1}_{\Pi,\mathfrak{C}_{\Pi,E\setminus E^-}}(\mathfrak{I}_{E^+}),t\models M$, if $I^{\alpha}_{\Pi,\mathfrak{C}_{\Pi,E\setminus E^-}}(\mathfrak{I}_{E^+}),t\models M$, then by the induction hypothesis $I^{\alpha}_{\Pi,\mathfrak{L}_{\Pi,E,E^-}}(\mathfrak{I}_{E'}),t\models M$.
    If $I^{\alpha}_{\Pi,\mathfrak{C}_{\Pi,E\setminus E^-}}(\mathfrak{I}_{E^+}),t\not\models M$, then there must be a ground rule instance $r'$ of a rule $r\in\Pi$, a dataset $\mathcal{D}$, such that $r'$ minimally derives $f$ from $\mathcal{D}$ at some $t'$, $I^{\alpha}_{\Pi,\mathfrak{C}_{\Pi,E\setminus E^-}}(\mathfrak{I}_{E^+})\cup\mathfrak{C}_{\Pi,E\setminus E^-}\models \mathcal{D}$, and $\mathfrak{I}_{\mathcal{D}}\cap I^{\alpha}_{\Pi,\mathfrak{C}_{\Pi,E\setminus E^-}}(\mathfrak{I}_{E^+})\neq\mathfrak{I}_{\emptyset}$.
    Let the sum of every right endpoint of every interval mentioned in $head(r')$ be $m_1$, and the sum of every right endpoint of every interval mentioned in $body(r')$ be $m_2$, then by the definition of $\mathsf{depth}(\Pi)$, $m_1+m_2\leq \mathsf{depth}(\Pi)$.

    Because $r'$ derives $f$ at $t'$, we know that $t'\in[t-m_1,t+m_1]$ and $\forall M'@\varrho'\in\mathcal{D}$, $\varrho\subseteq [t'-m_2,t'+m_2]$. 
    Since $m_1+m_2\leq \mathsf{depth}(\Pi)$, $\forall M'@\varrho'\in\mathcal{D}$, $\varrho\subseteq [t-\mathsf{depth}(\Pi),t+\mathsf{depth}(\Pi)]$. 
    Because $t>\varrho^+$ and $|\varrho|=\mathsf{depth}(\Pi)$, $[t-\mathsf{depth}(\Pi),t+\mathsf{depth}(\Pi)]\subseteq [\varrho^-,t+\mathsf{depth}(\Pi)]$, and by the induction hypothesis, $I^{\alpha}_{\Pi,\mathfrak{C}_{\Pi,E\setminus E^-}}(\mathfrak{I}_{E'})\cup\mathfrak{L}_{\Pi,E,E^-}\models \mathcal{D}$, $\mathfrak{I}_{\mathcal{D}}\cap I^{\alpha}_{\Pi,\mathfrak{C}_{\Pi,E\setminus E^-}}(\mathfrak{I}_{E^+})\subseteq I^{\alpha}_{\Pi,\mathfrak{C}_{\Pi,E\setminus E^-}}(\mathfrak{I}_{E'})$, and $I^{\alpha}_{\Pi,\mathfrak{C}_{\Pi,E\setminus E^-}}(\mathfrak{I}_{E'})\cap \mathfrak{I}_{\mathcal{D}}\neq\mathfrak{I}_{\emptyset}$.
    Consequently, by the definition of $I_{\Pi,\mathfrak{C}_{\Pi,E\setminus E^-}}$, $I^{\alpha+1}_{\Pi,\mathfrak{C}_{\Pi,E\setminus E^-}}(\mathfrak{I}_{E'}),t\models M$.

    For a limit ordinal $\alpha$, the implication holds trivially because $I^{\alpha}_{\Pi,\mathfrak{C}_{\Pi,E\setminus E^-}}(\mathfrak{I}_{E'}) = \bigcup_{\beta<\alpha}I^{\beta}_{\Pi,\mathfrak{C}_{\Pi,E\setminus E^-}}(\mathfrak{I}_{E'})$
\end{proof}

Recall that the $(\varrho_\mathsf{left},\varrho_\mathsf{right})$-unfolding of $\mathfrak{L}$ coincides with $\mathfrak{L}_{\Pi,E,E^-}$.
We show that if we find repeating patterns of length $\mathsf{depth}(\Pi)$ in the appropriate region that are a multiple of $|\varrho_\mathsf{left}|$ or $|\varrho_\mathsf{right}|$ apart, then the repetition continues to infinity.

We say an interpretation $\mathfrak{I}$ \textbf{satisfies with insertion} of $E^+$ a ground rule $r$ regarding $\mathfrak{I}'$ (during $\varrho$), if for each time point $t(\in\varrho)$, whenever there is a dataset $\mathcal{D}$ such that $\mathfrak{I}\cup\mathfrak{I}',t\models body(r)$, there exists $f'\in\mathcal{D}$ partially rooted in $E^-$ such that $\mathfrak{I}_{\mathcal{D-\{f'\}}}\not\models body(r)$, then $\mathfrak{I},t\models head(r)$.

Analogous to the rederivation case, we define insertion-saturated interpretations and show that their unfolding coincides with the least insertion model.
\begin{definition}
    Interpretation $I^{k}_{\Pi,\mathfrak{C}_{\Pi,E\setminus E^-}}(\mathfrak{I}_{E^+})$ is \textbf{insertion-saturated} regarding $\Pi,E,E^-,E^+$, if there exist $\varrho_1,\varrho_2,\varrho_3$ and $\varrho_4$ of length $2\mathsf{depth}(\Pi)$,
whose endpoints are located on the $(\Pi, E\cup E^+)$-ruler and satisfy
$\varrho_1^+<\varrho_2^+<t_E^-$
and $t_E^+<\varrho^-_3<\varrho_4^-$, and for which the
following properties hold:
\begin{enumerate}
    \item $\mathfrak{C}_{\Pi,E\setminus E^-}\mid_{[\varrho_1^-,\varrho_4^+]}$ is a saturated interpretation with periods $[\varrho_1^-,\varrho_2^-),(\varrho_3^+,\varrho_4^+]$.
    \item $I^{k}_{\Pi,\mathfrak{C}_{\Pi,E\setminus E^-}}(\mathfrak{I}_{E^+})$ is an insertion model of $\Pi,\mathfrak{C}_{\Pi,E\setminus E^-}$ during $[\varrho_1^- +\mathsf{depth}(\Pi),\varrho_4^+-\mathsf{depth}(\Pi)]$.
    \item $I^{k}_{\Pi,\mathfrak{C}_{\Pi,E\setminus E^-}}(\mathfrak{I}_{E^+})\mid_{\varrho_2}$ is a shift of $I^{k}_{\Pi,\mathfrak{C}_{\Pi,E\setminus E^-}}(\mathfrak{I}_{E^+})\mid_{\varrho_1}$, 
    \item $I^{k}_{\Pi,\mathfrak{C}_{\Pi,E\setminus E^-}}(\mathfrak{I}_{E^+})\mid_{\varrho_4}$ is a shift of $I^{k}_{\Pi,\mathfrak{C}_{\Pi,E\setminus E^-}}(\mathfrak{I}_{E^+})\mid_{\varrho_3}$.
\end{enumerate}
For each such that $\varrho_1,\varrho_2,\varrho_3$ and $\varrho_4$, we say that $[\varrho_1^-,\varrho_2^-),(\varrho_3^+,\varrho_4^+]$ are the \textbf{insertion periods} of $I^{k}_{\Pi,\mathfrak{C}_{\Pi,E\setminus E^-}}(\mathfrak{I}_{E^+})$.

\end{definition}
\begin{proposition}
        Interpretation $I^{k}_{\Pi,\mathfrak{C}_{\Pi,E\setminus E^-}}(\mathfrak{I}_{E^+})$ is also \textbf{insertion-saturated} regarding $\Pi,\mathfrak{C}_{\Pi,E\setminus E^-}$, if there exist $\varrho_1,\varrho_2,\varrho_3$ and $\varrho_4$ of length $2\mathsf{depth}(\Pi)$,
whose endpoints are located on the $(\Pi, E\cup E^+)$-ruler and satisfy
$\varrho_1^+<\varrho_2^+<t_E^-$
and $t_E^+<\varrho^-_3<\varrho_4^-$, and for which the
following properties hold:
\begin{enumerate}
    \item $\mathfrak{C}_{\Pi,E\setminus E^-}\mid_{[\varrho_1^-,\varrho_4^+]}$ is a saturated interpretation with periods $[\varrho_1^-,\varrho_2^-),(\varrho_3^+,\varrho_4^+]$.
    \item $I^{k}_{\Pi,\mathfrak{C}_{\Pi,E\setminus E^-}}(\mathfrak{I}_{E^+})\mid_{[\varrho_1^-,\varrho_4^+]}=I^{k+1}_{\Pi,\mathfrak{C}_{\Pi,E\setminus E^-}}(\mathfrak{I}_{E^+})\mid_{[\varrho_1^-,\varrho_4^+]}$ 
    \item $I^{k}_{\Pi,\mathfrak{C}_{\Pi,E\setminus E^-}}(\mathfrak{I}_{E^+})\mid_{\varrho_2}$ is a shift of $I^{k}_{\Pi,\mathfrak{C}_{\Pi,E\setminus E^-}}(\mathfrak{I}_{E^+})\mid_{\varrho_1}$, 
    \item $I^{k}_{\Pi,\mathfrak{C}_{\Pi,E\setminus E^-}}(\mathfrak{I}_{E^+})\mid_{\varrho_4}$ is a shift of $I^{k}_{\Pi,\mathfrak{C}_{\Pi,E\setminus E^-}}(\mathfrak{I}_{E^+})\mid_{\varrho_3}$.
\end{enumerate}

\end{proposition}
\begin{proof}
The proof is analogous to that of Proposition~\ref{proposition:deletion-saturated}.
\end{proof}

With a proof analogous to that of Theorem~\ref{theo:deletionunfold}, we have
\begin{theorem}\label{theo:insertionunfold}
    For a insertion-saturated interpretation $$I^{k}_{\Pi,\mathfrak{C}_{\Pi,E\setminus E^-}}(\mathfrak{I}_{E^+})$$ and a pair of its insertion  periods $\varrho_{\mathit{left}},\varrho_{\mathit{right}}$, the $(\varrho_{\mathit{left}},\varrho_{\mathit{right}})$-unfolding $\mathfrak{I}$ of $I^{k}_{\Pi,\mathfrak{C}_{\Pi,E\setminus E^-}}(\mathfrak{I}_{E^+})$ coincides with $I^{\omega_1}_{\Pi,\mathfrak{C}_{\Pi,E\setminus E^-}}(\mathfrak{I}_{E^+})$.
\end{theorem}

Given two interpretations $\mathfrak{I}_b$ and $\mathfrak{I}_0$,
we define a transfinite sequence of interpretations $\mathfrak{A}_0,\mathfrak{A}_1,\dots$ such that 
\begin{enumerate*}[label=(\roman*)]
    \item $\mathfrak{A}_0 = \mathfrak{I}_{E^+}$,
    \item $\mathfrak{A}_1 = \dot{I}_{\Pi,\mathfrak{I}_b}(\mathfrak{I}_\emptyset,\mathfrak{A}_{0})$
    \item $\mathfrak{A}_{\alpha+1} = \dot{I}_{\Pi,\mathfrak{I}_b}(\mathfrak{A}_{\alpha-1},\mathfrak{A}_{\alpha})$, for $\alpha>1$ a successor ordinal, and
    \item $\mathfrak{A}_{\beta} = \bigcup_{\alpha<\beta}\mathfrak{A}_{\alpha}$, for $\beta$ a limit ordinal.
\end{enumerate*}
By a proof analogous to that for Proposition~\ref{proposition:rederivationseminaive=naive},
 we have
\begin{proposition}\label{proposition:insertionseminaive=naive}
     For each ordinal $\alpha$, $$\mathfrak{A}_\alpha = I_{\Pi,\mathfrak{C}_{\Pi,E\setminus E^-}}^\alpha(\mathfrak{I}_{E^+})$$
\end{proposition}

With proof analogous to that of Proposition~\ref{proposition:deletion-saturated-align}, Proposition \ref{proposition:insertion-saturated-align} holds as well.

\begin{proposition}\label{proposition:insertion-saturated-align}
Given a periodic materialisation $\mathds{I}$ such that $\mathsf{unfold}(\mathds{I}) =\mathfrak{C}_{\Pi,E\setminus E^-}$, there exists $k \leq k_{\mathsf{max}}$ such that $I^k_{\Pi,\mathfrak{C}_{\Pi,E\setminus E^-}}(\mathfrak{I}_{E^+})$ 
is insertion-saturated, where the bound $k_{max}$ is defined as follows.
Let $A$ be the number of ground relational atoms in the
grounding of $\Pi$ with constants from $\Pi$ and $E$, let $B$ be the
number of $(\Pi, E)$-intervals within $[t_E^-, t_E^-
 + 2depth(\Pi)]$, let $C$ be the max number of $(\Pi,E)$-endpoints contained by $\mathds{I}.\varrho_{\mathsf{L}}$ or $\mathds{I}.\varrho_{\mathsf{R}}$,
and let $\varrho$ = $[t_{-w} - 2C\mathsf{depth}(\Pi), t_w + 2C\mathsf{depth}(\Pi)]$, where $\cdots < t_{-2}<t_{-1}<t_{E}^-$ and $t_E^+<t_1<t_2<\cdots$ are sequences of consecutive time points on the $(\Pi, E)$-ruler, while
$w$ = $C +B ·C· (2^A)^B$. 
Then $k_\mathsf{max}$ is the product of $A$ and the
number of $(\Pi, E)$-intervals contained in $\varrho$.
\end{proposition}

\subsection{Correctness of Algorithm~\ref{algo:DREDMTL}}
\theoremalgorithmcorrectness*
\begin{proof}
    For the overdeletion phase, let $\mathfrak{I}_{\Delta_D}$ be the least model of $\Delta_D$ immediately after the first execution of line 8, then $\mathfrak{I}_{\Delta_D} = \mathfrak{I}_{E^-}$. 
    By definition of the semi-na\"ive immediate deletion consequence opertor and $\Pi\langle\rangle$, line 11 computes $\dot{D}_{\Pi,E}(\mathfrak{I}_\emptyset,\mathfrak{I}_{E^-})$, and after the first execution of line 12, $\mathfrak{I}_{D} = \mathfrak{I}_{E^-}$.

    Let the least model of $D$ after the $i$-th execution of line 12 be $\mathfrak{I}_{D_i}$. Then $\mathfrak{I}_{D_0} = \mathfrak{I}_{\emptyset}$ and $\mathfrak{I}_{D_1} = \mathfrak{I}_{E^-}$, and the least of model of $\Delta_D$ after the $i$-th execution of 8 is $\mathfrak{I}_{D_i} - \mathfrak{I}_{D_{i-1}}$, we know that $\mathfrak{I}_{D_{i+1}} = \dot{D}_{\Pi,E}(\mathfrak{I}_{D_{i-i}},\mathfrak{I}_{D_{i}})$, and we have $\mathfrak{I}_{D_{\omega_1}}$ is $\dot{D}_{\Pi,E}^{\omega_1}(\mathfrak{I}_{E^-})$.
    Then by Proposition~\ref{proposition:deletionseminaive=naive}, $\mathfrak{I}_{D_{i}}=D_{\Pi,E}^i(\mathfrak{I}_{E^-})$. 
    Proposition~\ref{proposition:deletion-saturated-align} ensures that line 9 will always find the deletion-saturated interpretation with its deletion periods $\varrho_\mathsf{L}^D,\varrho_\mathsf{R}^D$, in finitely many rounds, and therefore, \textsc{Overdelete} terminates. 
    Theorem~\ref{theo:deletionunfold} ensures that the $(\varrho_\mathsf{L}^D,\varrho_\mathsf{R}^D)$-unfolding of the deletion-saturated interpretation found by line 9, $\mathsf{unfold}(\mathds{D})$,  is $\mathfrak{M}^-_{\Pi,E,E^-}$.

    For the rederivation phase, we can by a similar analysis to the overdeletion phase know that the least model of $R$, $\mathfrak{I}_{R_i}$ after the $i$-th execution of line 23 is $\dot{\mathfrak{R}}_i$. 
    Then by Proposition~\ref{proposition:rederivationseminaive=naive} $\mathfrak{I}_{R_i} = \mathfrak{R}_i$. 
    Proposition~\ref{proposition:rederivation-saturated-align} ensures that line 21 will always find the rederivation-saturated interpretation with its periods $\varrho_\mathsf{L}^R,\varrho_\mathsf{R}^R$ in finitely many rounds.
    Theorem~\ref{theo:rederivitionunfold} ensures that the $\varrho_\mathsf{L}^R,\varrho_\mathsf{R}^R$-unfolding of the rederivation-saturated interpretation coincides with $\mathfrak{R}_{\omega_1}$.
    Proposition~\ref{proposition:rederivationcorrect} then ensures that after line 3, $$\mathsf{unfold}(\mathds{I}) =\mathfrak{C}_{\Pi,E\setminus E^-}.$$

    For the insertion phase, we can again by a similar analysis know that the least model of $A$, $\mathfrak{I}_{A_i}$ after the $i$-th execution of line 33 is $\mathfrak{A}_i$.
    Then by Proposition~\ref{proposition:insertionseminaive=naive} $\mathfrak{I}_{R_i} = I_{\Pi,\mathfrak{C}_{\Pi,E\setminus E^-}}^i(\mathfrak{I}_{E^+})$. 
    Proposition~\ref{proposition:insertion-saturated-align} ensures that line 30 will always find the insertion-saturated interpretation with its periods $\varrho_\mathsf{L}^A,\varrho_\mathsf{R}^A$ in finitely many rounds. 
    Theorem~\ref{theo:insertionunfold} ensures that the $(\varrho_\mathsf{L}^A,\varrho_\mathsf{R}^A)$-unfolding of the insertion-saturated interpretation coincides with $I_{\Pi,\mathfrak{C}_{\Pi,E\setminus E^-}}^{\omega_1}(\mathfrak{I}_{E^+})$.
    Proposition~\ref{proposition:insertioncorrect} then ensures that after line 4, $$\mathsf{unfold}(\mathds{I}) =\mathfrak{C}_{\Pi,E\setminus E^-\cup E^+} = \mathfrak{C}_{\Pi,E'}.$$
\end{proof}
\end{appendices}

\end{document}